\documentclass[11pt]{article}

\usepackage{amssymb, amsmath, amsthm, latexsym}
\usepackage{fullpage}
\usepackage{url}
\usepackage{algorithm}
\usepackage{algorithmic}
\usepackage{tabularx}
\usepackage{paralist}
\usepackage{mathtools}

\usepackage{bbm} 
\usepackage{wrapfig}
\usepackage{makecell}
\usepackage{multirow}
\usepackage{booktabs}

\usepackage{nicefrac}       

\usepackage[flushleft]{threeparttable} 

\usepackage{caption}
\usepackage{multirow}
\usepackage{colortbl}
\definecolor{bgcolor}{rgb}{0.8,1,1}
\definecolor{bgcolor2}{rgb}{0.8,1,0.8}
\definecolor{niceblue}{rgb}{0.0,0.19,0.56}

\usepackage{hyperref}
\hypersetup{colorlinks,linkcolor={blue},citecolor={niceblue},urlcolor={blue}}

\usepackage{natbib}
\usepackage{pifont}
\definecolor{PineGreen}{RGB}{0,110,51}
\definecolor{BrickRed}{RGB}{143,20,2}

\usepackage{tikz-cd} 

\usepackage{subfigure}

\newcommand{\R}{\mathbb{R}}
\newcommand{\eqdef}{:=}

\def\<#1,#2>{\left\langle #1,#2\right\rangle}

\newcolumntype{Y}{>{\centering\arraybackslash}X}

\usepackage{xspace}

\definecolor{mycolor}{RGB}{0,150,70}
\newcommand{\algname}[1]{{\color{mycolor}\sf  #1}\xspace}

\newcommand{\circledOne}{\text{\ding{172}}}
\newcommand{\circledTwo}{\text{\ding{173}}}
\newcommand{\circledThree}{\text{\ding{174}}}
\newcommand{\circledFour}{\text{\ding{175}}}
\newcommand{\circledFive}{\text{\ding{176}}}
\newcommand{\circledSix}{\text{\ding{177}}}
\newcommand{\circledSeven}{\text{\ding{178}}}

\usepackage[colorinlistoftodos,bordercolor=orange,backgroundcolor=orange!20,linecolor=orange,textsize=scriptsize]{todonotes}

\newcommand\norm[1]{\left\lVert#1\right\rVert}

\theoremstyle{plain}
\newtheorem{theorem}{Theorem}[section]

\newtheorem{lemma}[theorem]{Lemma}
\newtheorem{corollary}[theorem]{Corollary}
\theoremstyle{definition}
\newtheorem{definition}[theorem]{Definition}
\newtheorem{assumption}[theorem]{Assumption}
\theoremstyle{remark}

\newcommand{\cB}{{\cal B}}

\newcommand{\cD}{{\cal D}}

\newcommand{\cF}{{\cal F}}

\newcommand{\cN}{{\cal N}}
\newcommand{\cO}{{\cal O}}
\newcommand{\cP}{{\cal P}}

\newcommand{\cT}{{\cal T}}

\newcommand{\cX}{{\cal X}}

\newcommand{\mI}{{\bf I}}

\newcommand{\EE}{\mathbb{E}}

\newcommand{\PP}{\mathbb{P}}

\newcommand{\tg}{\widetilde{g}}

\def\clip{\texttt{clip}}

\def\clip{\texttt{clip}}

\usepackage{hyperref}
\graphicspath{{plots/}}

\usepackage{makecell}

\usepackage{accents}
\newlength{\dhatheight}

\usepackage{pgfplotstable} 
\usetikzlibrary{automata, positioning, arrows, shapes, fit, calc, intersections}
\usepgfplotslibrary{statistics}
\include{figure}

\def\la{\langle}
\def\ra{\rangle}

\makeatletter
\newcommand{\mytag}[1]{%
  \refstepcounter{equation}%
  \edef\@currentlabel{\theequation}%
  {({\@currentlabel})}%
  \@bsphack
  \begingroup
    \@onelevel@sanitize\@currentlabelname
    \edef\@currentlabelname{%
      \expandafter\strip@period\@currentlabelname\relax.\relax\@@@%
    }%
    \protected@write\@auxout{}{%
      \string\newlabel{#1}{%
        {\@currentlabel}%
        {\thepage}%
        {\@currentlabelname}%
        {\@currentHref}{}%
      }%
    }%
  \endgroup
  \@esphack
}
\makeatother

\title{\bf Differentially Private Clipped-SGD: High-Probability Convergence with Arbitrary Clipping Level\footnote{Preprint under review.}}
\author{\begin{tabular}{c}
     {\bf Saleh Vatan Khah}\\
     {\small IUST}
\end{tabular} \quad 
\begin{tabular}{c}
     {\bf Savelii Chezhegov}\\
     {\small Independent Researcher}
\end{tabular}\quad \begin{tabular}{c}
     {\bf Shahrokh Farahmand}\\
     {\small IUST}
\end{tabular}  \\  \begin{tabular}{c}
     {\bf Samuel Horv\'ath} \\
     {\small MBZUAI} 
\end{tabular} \quad 
\begin{tabular}{c}
     {\bf Eduard Gorbunov}\thanks{Corresponding author: \texttt{eduard.gorbunov@mbzuai.ac.ae}.}\\
     {\small MBZUAI}
\end{tabular}}

\begin{document}

\maketitle

\begin{abstract}
    Gradient clipping is a fundamental tool in Deep Learning, improving the high-probability convergence of stochastic first-order methods like \algname{SGD}, \algname{AdaGrad}, and \algname{Adam} under heavy-tailed noise, which is common in training large language models. It is also a crucial component of Differential Privacy (DP) mechanisms. However, existing high-probability convergence analyses typically require the clipping threshold to increase with the number of optimization steps, which is incompatible with standard DP mechanisms like the Gaussian mechanism. In this work, we close this gap by providing the first high-probability convergence analysis for \algname{DP-Clipped-SGD} with a fixed clipping level, applicable to both convex and non-convex smooth optimization under heavy-tailed noise, characterized by a bounded central $\alpha$-th moment assumption, $\alpha \in (1,2]$. Our results show that, with a fixed clipping level, the method converges to \emph{a neighborhood} of the optimal solution with a \emph{faster rate} than the existing ones. The neighborhood can be balanced against the noise introduced by DP, providing a refined trade-off between convergence speed and privacy guarantees.
\end{abstract}

\tableofcontents

\section{Introduction}\label{introduction}

Stochastic first-order optimization methods, such as Stochastic Gradient Descent (\algname{SGD}) \citep{robbins1951stochastic}, \algname{AdaGrad} \citep{streeter2010less, duchi2011adaptive}, and \algname{Adam} \citep{kingma2014adam}, are fundamental for training modern Machine Learning (ML) and Deep Learning (DL) models. However, these methods are often enhanced with additional algorithmic techniques that play a critical role in their convergence and practical performance. Among these, gradient clipping \citep{pascanu2013difficulty} is one of the most widely used and well-studied approaches. In recent years, substantial efforts have been made to theoretically understand the advantages of gradient clipping and its impact on the convergence of stochastic optimization algorithms.

In particular, gradient clipping is a key component in managing heavy-tailed noise, which commonly arises in the training of language models on textual data \citep{zhang2020adaptive}, in the training of GANs \citep{goodfellow2014generative, gorbunov2022clipped}, and even in simpler tasks such as image classification \citep{csimcsekli2019heavy}. This approach is primarily analyzed through the lens of high-probability convergence, as such guarantees provide a more accurate reflection of the actual behavior of optimization methods compared to their more conventional in-expectation counterparts \citep{gorbunov2020stochastic}. Moreover, as demonstrated by \citet{sadiev2023high} for \algname{SGD} and by \citet{chezhegov2024gradient} for \algname{AdaGrad} and \algname{Adam}, methods without clipping may fail to exhibit high-probability convergence with logarithmic dependence on the failure probability. In contrast, several recent works \citep{gorbunov2020stochastic, cutkosky2021high, sadiev2023high, nguyen2023improved, gorbunov2024high, chezhegov2024gradient, parletta2024high} have established that various stochastic first-order methods attain significantly better high-probability convergence under heavy-tailed noise assumptions across different settings.

On the other hand, clipping is a cornerstone of Differentially Private (DP) machine learning. The widely used Gaussian mechanism \citep{dwork2014algorithmic} achieves privacy by adding Gaussian noise to the gradients, thereby introducing uncertainty about their true values. However, the DP guarantees provided by this mechanism rely on the assumption that the gradients have bounded norms, a condition typically enforced through gradient clipping \citep{abadi2016deep}.

It is therefore tempting to claim that gradient clipping can provably address two distinct challenges simultaneously: mitigating heavy-tailed noise and ensuring differential privacy (DP). However, this is not entirely accurate, as the clipping policies required for these two objectives differ substantially. In the context of heavy-tailed noise, existing convergence guarantees are typically derived assuming that the clipping level increases with the total number of training steps. In contrast, DP mechanisms require a fixed and bounded clipping threshold to ensure robust privacy guarantees. This fundamental mismatch raises a critical question:
\begin{gather*}
    \textit{How does differentially private version of \algname{Clipped-SGD} converge with high probability}\\
    \textit{under the heavy-tailed noise?}
\end{gather*}

\paragraph{Our contribution.} In this paper, we address the above question by providing the first high-probability convergence bounds for the differentially private version of \algname{Clipped-SGD} (\algname{DP-Clipped-SGD}) with an \emph{arbitrary fixed clipping level} applied to convex smooth optimization problems under heavy-tailed noise. Specifically, we assume that the stochastic gradient has a bounded central $\alpha$-th moment for some $\alpha \in (1,2]$ and establish that \algname{DP-Clipped-SGD} achieves a high-probability convergence rate of $\widetilde\cO(K^{-\nicefrac{1}{2}})$ to a certain \emph{neighborhood} of the optimal solution. This rate is significantly better than the previously known bound of $\widetilde\cO(K^{-\nicefrac{(\alpha-1)}{\alpha}})$ in this setting.

However, this improvement is achieved by relaxing the requirement for exact convergence and instead demonstrating convergence to a neighborhood whose size depends non-trivially on the clipping level, noise scale, and other problem-dependent parameters. Importantly, the size of this neighborhood, introduced due to the inherent bias in clipped stochastic gradients, can be carefully balanced with the neighborhood induced by the DP noise, allowing for more flexible control over the trade-off between convergence accuracy and privacy. Additionally, we extend our results to the non-convex case, illustrating the broader applicability of our analysis.

\section{Technical Preliminaries}\label{preliminaries}

The optimization problem considered in this work has the following form
\begin{align}
    \min_{x \in \mathbb{R}^d} \left\{f(x):=\mathbb{E}_{\xi \sim \mathcal{D}}[f_{\xi}(x)]\right\}. \label{min_problem}
\end{align}
Here, $x$ denotes the model parameters, $f:\R^d \to \R$ is the expected loss function, and $f_{\xi}:\R^d \to \R$ represents the loss computed for a random sample $\xi$ drawn from an (often unknown) distribution $\cD$. Such problems are fundamental in machine learning \citep{shalev2014understanding}. 

We assume that at each iteration, we have access to an oracle that provides a stochastic gradient $\nabla f_{\xi}(x)$, as well as a $d$-dimensional random vector $\omega$ sampled from a Gaussian distribution $\cN(0, \sigma_\omega^2 \mI_d)$, where $\mI_d$ is the $d \times d$ identity matrix. More precisely, the random variables $\xi$ and $\omega$ are defined on the probability space $\left(\Omega_d \times \R^d,\cB(\Omega_d) \otimes \cB(\mathbb{R}^d),\cF^t,\mathbb{P}\right)$, where $\Omega_d$ represents the data sample space, and $\cB(\cX)$ denotes the Borel $\sigma$-algebra generated by the set $\cX$. This probability space is also equipped with the natural filtration $\cF^t=\sigma \left(\left[\nabla f_{\xi^0}(x^0), \omega_0\right]^T, \dots \left[\nabla f_{\xi^t}(x^t),\omega_t\right]^T\right)$, which captures the history of the stochastic process up to time $t$. The probability measure $\mathbb{P}$ is defined as the product measure on this space, given by
\begin{align}
    \mathbb{P} \{B_d \times B_{\omega}\} = (\mu \times \nu)(B_d \times B_{\omega}) = \mu(B_d) ~\nu(B_{\omega}), \quad \forall B_d \in \cB(\Omega_d), \forall B_{\omega} \in \cB(\mathbb{R}^d),
\end{align}
where $\mu$ is a probability measure on $\Omega_d$, and $\nu$ is the Gaussian measure on $\mathbb{R}^d$ with mean zero and covariance matrix $\sigma_\omega^2 \mI_d$. 


\paragraph{Types of convergence bounds.} Several types of convergence bounds are commonly used to analyze the behavior of stochastic optimization methods, ranging from in-expectation bounds to almost sure convergence guarantees. High-probability convergence bounds provide guarantees of the form $\PP \left\{ \cP(x^K) \leq \epsilon \right\} \geq 1-\beta$, where $\cP(x)$ is a performance metric that measures the quality of the solution\footnote{Examples of such performance metric for problem \eqref{min_problem}: $\cP(x) = f(x)-f(x^*)$, $\cP(x) = \norm{\nabla f(x)}^2$, $\cP(x) = \norm{x-x^*}^2$, where $x^* \in \arg \min_{x\in \mathbb{R}^d} f(x)$.}. Here, $\PP\{\cdot\}$ denotes the probability measure defined by the problem setup, $x^K$ is the algorithm's output after $K$ iterations, $\beta$ is the confidence level (or failure probability), and $\epsilon$ is the optimization error. 

This type of convergence is generally considered superior to in-expectation guarantees (e.g., $\EE[\cP(x^K)] \leq \epsilon$), as it captures not only the average behavior of the underlying random variables but also their tail behavior, which is particularly important for distributions with heavy tails. However, it is worth noting that the number of iterations $K$ required to achieve such high-probability guarantees can depend inversely on the failure probability $\beta$, as seen in analyses for methods like \algname{SGD} \citep{sadiev2023high}, \algname{AdaGrad}, and \algname{Adam} \citep{chezhegov2024gradient}. Such inverse-power dependencies on $\beta$ are generally undesirable, as $\beta$ is typically chosen to be very small. Consequently, a major objective in the high-probability convergence literature is to establish bounds with polylogarithmic dependence on $\nicefrac{1}{\beta}$, which are significantly tighter and more practical.


\paragraph{Assumptions.} In the following, we list the assumptions on the structure of the problem at hand. These assumptions are very mild and cover a wide range of problems.
\begin{assumption}\label{ass:bounded_below}
    We assume the function $f$ is uniformly lower-bounded on some subset $Q \subseteq \mathbb{R}^d$, i.e., $f^\ast \eqdef \inf_{x\in Q} f(x) > -\infty$.
\end{assumption}
The above assumption is necessary for problem \eqref{min_problem} to be feasible. Next, we make a standard assumption about the smoothness of the objective function.

\begin{assumption}\label{ass:smoothness}
    We assume that there exists a constant $L>0$ such that for all $x,y \in Q \subseteq \mathbb{R}^d$ the function $f$ satisfies the following.
    \begin{align}
        \norm{\nabla f (x)-\nabla f (y)}\leq L\norm{x-y}.
    \end{align}
\end{assumption}
In this work, we consider both classes of convex and non-convex functions. The following assumption holds only for convex functions.
\begin{assumption}\label{ass:convexity}
    We assume there exists a subset $Q$ of $\mathbb{R}^d$ such that for all $x,y \in Q$
    \begin{align}
        f(y)\geq f(x)+\langle \nabla f(x),y-x \rangle.
    \end{align}
\end{assumption}
The following assumption is with respect to the stochastic oracle that our algorithm receives at each iteration.  We assume that the stochastic gradients have a bounded central $\alpha$ moment for some $\alpha\in  (1,2]$. This assumption is stated explicitly below.
\begin{assumption}\label{ass:oracle}
    We assume there exist some subset $Q \subseteq \mathbb{R}^d$, and some constants $\sigma >0$, $\alpha\in (1,2]$ such that for all $x\in Q$
    \begin{align}
        \mathbb{E}_{\xi \sim D}\left[\nabla f_{\xi}(x)\mid x\right]=\nabla f (x), \\
         \mathbb{E}_{\xi \sim D}\left[\norm{\nabla f_{\xi}(x)-\nabla f(x)}^\alpha \mid x\right]\leq \sigma^\alpha .
    \end{align}
\end{assumption}
As it can be seen, in the case $\alpha=2$, the aforementioned conditions recover the standard uniformly bounded variance assumption widely used for obtaining convergence guarantees for optimization algorithms in the literature. Since the $L^p$ norms of random variable are non-decreasing in $p$, this assumption allows the stochastic gradients to have infinite variance.

Next, we use the classical  definition of $(\varepsilon,\delta)$-differential privacy. Intuitively, it  provides probabilistic guarantees that an intruder cannot infer the existence of a particular data in the data set that the algorithm used to train the model.
\begin{definition}
    ( ($\epsilon,\delta$)-Differential Privacy \citep{dwork2014algorithmic}). A randomized method \(\mathcal{M}: \mathcal{D} \to \mathcal{R}\) satisfies \((\varepsilon,\delta)\)-Differential Privacy, if for any adjacent \(D, D' \in \mathcal{D}\) and for any \(\mathcal{S} \subseteq \mathcal{R}\)
\begin{equation}
    \PP\left(\mathcal{M}(\mathcal{D})\in \mathcal{S}\right) \leq e^{\varepsilon}\PP\left(\mathcal{M}(\mathcal{D'})\in \mathcal{S}\right) + \delta,
\end{equation}
\end{definition}
Smaller ($\varepsilon, \delta$) provides stronger privacy guarantee. This also can be viewed from the perspective of Bayesian hypothesis testing where the null and alternative hypothesis are about the existence of an individual's data in the dataset \citep{pmlr-v37-kairouz15, su2024statistical}.

\section{Related Work}

\paragraph{Clipping in Differential Private learning.} There are several approaches to ensuring DP guarantees in \algname{SGD}, but the most common method relies on a combination of gradient clipping and noise injection. In the finite-sum setting, \citet{abadi2016deep} demonstrated that it is sufficient to add Gaussian noise (the Gaussian mechanism) with standard deviation $\sigma_\omega = \Theta \left(\frac{q\lambda}{\varepsilon}\sqrt{K \ln{\frac{1}{\delta}}}\right)$ to the clipped gradients, where $q$ is the sampling probability for each individual summand. This approach reduces the variance of the required Gaussian noise by a factor of $\sqrt{\ln{K}}$ compared to the advanced composition theorem \citep{dwork2014algorithmic}, significantly improving the utility of DP training.

This combination of gradient clipping and the Gaussian mechanism has become a standard approach in many DP training algorithms. However, these methods often rely on restrictive assumptions, such as requiring the clipping level to always be larger than the norm of the transmitted vector \citep{zhang2022understanding, noble2022differentially, allouah2023privacy, allouah2024privacy, li2025convergence}\footnote{\citet{li2025convergence} also provide an in-expectation convergence result without the bounded gradient assumption, but with a worse dependence on the variance bound of the stochastic gradients.}, assuming symmetry of the noise distribution \citep{liu2022communication}, or requiring that the full gradients be computed \citep{wei2020federated}. These conditions can be quite restrictive, particularly in practical large-scale settings.

To the best of our knowledge, the only works that avoid these restrictive assumptions are \citet{koloskova2023revisiting, islamov2025double}. Specifically, \citet{koloskova2023revisiting} analyzed the in-expectation convergence of \algname{DP-Clipped-SGD} with mini-batching under the bounded variance assumption, for an arbitrary clipping level in the non-convex $(L_0, L_1)$-smooth regime \citep{zhang2019gradient}. However, they leave open the question of high-probability convergence under heavy-tailed noise. \citet{islamov2025double} proposed a distributed optimization method that incorporates clipping, error feedback \citep{seide20141, richtarik2021ef21}, and heavy-ball momentum \citep{polyak1964some}. Yet, their high-probability convergence analysis crucially relies on the assumption that the noise in the stochastic gradients has sub-Gaussian tails. In contrast, under the more realistic Assumption~\ref{ass:oracle} with $\alpha \geq 2$ (which is still more restrictive than the heavy-tailed case with $\alpha < 2$), \citet{zhao2025differential} derive in-expectation convergence bounds for a variant of projected \algname{SGD} that employs DP mean estimation using a sufficiently large number of samples. However, this approach can be prohibitively expensive in practice, especially for training large language models.

\paragraph{High-probability convergence bounds.} If the noise in the stochastic gradient has light tails, then classical stochastic first-order methods like \algname{SGD} and its adaptive and momentum-based variants can achieve desirable high-probability convergence rates, characterized by polylogarithmic dependence on the failure probability $\beta$. For instance, under the sub-Gaussian noise assumption, such results exist for \algname{SGD} \citep{nemirovski2009robust, harvey2019simple}, its accelerated variants \citep{ghadimi2012optimal, dvurechensky2016stochastic}, and its momentum and \algname{AdaGrad} versions \citep{li2020high, liu2023high}. Additionally, \citet{madden2024high} demonstrate that polylogarithmic high-probability bounds can also be achieved for \algname{SGD} under the weaker sub-Weibull noise assumption. However, as highlighted by \citet{sadiev2023high} and \citet{chezhegov2024gradient}, methods like \algname{SGD}, \algname{AdaGrad}, and \algname{Adam} can fail to achieve these desired high-probability rates under heavier-tailed noise distributions.

To address the limitations of high-probability convergence for stochastic methods under heavy-tailed noise, several algorithmic modifications have been proposed and rigorously analyzed in recent years. \citet{nazin2019algorithms} introduced a variant of Stochastic Mirror Descent \citep{nemirovskij1983problem} with \emph{truncation} of the stochastic gradient, establishing high-probability complexity bounds for convex and strongly convex smooth optimization over compact sets under the bounded variance assumption (Assumption~\ref{ass:oracle} with $\alpha = 2$). Interestingly, the truncation operator used in this work, while not identical, is closely related to the standard \emph{gradient clipping} technique that has since become the foundation of many subsequent studies.

In particular, \citet{gorbunov2020stochastic} derived the first high-probability complexity bounds for \algname{Clipped-SGD} and also proposed an accelerated version based on the Stochastic Similar Triangles Method (\algname{SSTM}) \citep{gasnikov2016universal}. These results were later extended to non-smooth problems by \citet{gorbunov2024high_non_smooth, parletta2024high}, to unconstrained variational inequalities by \citet{gorbunov2022clipped}, and to settings with noise having a bounded $\alpha$-th moment by \citet{cutkosky2021high} (with an additional bounded gradient assumption in the non-convex case). Building on these foundations, \citet{sadiev2023high} extended the results from \citet{gorbunov2020stochastic} and \citet{gorbunov2022clipped} to the more challenging setting defined by Assumption~\ref{ass:oracle} with $\alpha < 2$, removing the bounded gradient assumption for non-convex objectives. This work also introduced new high-probability bounds for \algname{Clipped-SGD} in the non-convex regime. These non-convex results were further refined by \citet{nguyen2023improved}, who also obtained tighter logarithmic factors in the convergence rates for both convex and strongly convex settings. 

In the context of distributed optimization, \citet{gorbunov2024high} extended the results of \citet{sadiev2023high} to distributed composite minimization and variational inequalities using the clipping of gradient differences, thereby broadening the applicability to decentralized and federated learning scenarios.

Adaptive methods have also been analyzed through the lens of high-probability convergence. \citet{li2023high} derived new high-probability bounds for \algname{Clipped-AdaGrad} with scalar step-sizes, while \citet{chezhegov2024gradient} obtained analogous bounds for various versions of \algname{Clipped-AdaGrad} and \algname{Clipped-Adam} with both scalar and coordinate-wise step-sizes. Additionally, \citet{kornilov2023accelerated} proposed a zeroth-order variant of \algname{Clipped-SSTM} and analyzed it under Assumption~\ref{ass:oracle}, extending the clipping framework to derivative-free settings.

However, a critical limitation shared by all of these methods is that the clipping level $\lambda$ is typically chosen as an increasing function of the total number of steps $K$\footnote{In some cases, such as the analysis of \algname{Clipped-SSTM} \citep{gorbunov2020stochastic} or \algname{Clipped-SGD} under strong convexity \citep{sadiev2023high}, the clipping level decreases as a function of the current iteration counter $k$ but still increases overall as a function of $K$.}. This choice, while theoretically convenient, leads to prohibitively large DP noise variance when aiming to guarantee $(\varepsilon,\delta)$-DP, resulting in utility bounds that grow with $K$ and significantly degrade the practical effectiveness of these methods in privacy-preserving applications. 

There exist other alternatives to gradient clipping that also ensure high-probability convergence with polylogarithmic dependency on the failure probability. They include robust distance estimation coupled with inexact proximal point steps \citep{davis2021low}, gradient normalization \citep{cutkosky2021high, hubler2024gradient}, and sign-based methods \citep{kornilov2025sign}. Notably, the approaches from \citet{hubler2024gradient, kornilov2025sign} enjoy provable (yet sub-optimal) high-probability convergence even when $\alpha$ is unknown. In the special case of symmetric distributions, \citet{armacki2023high, armacki2024large} provide new high-probability convergence bounds for a large class of \algname{SGD}-type methods with non-linear transformations such as standard clipping, coordinate-wise clipping, normalization, and sign-operator, and \citet{puchkin2024breaking} derive high-probability convergence of \algname{SGD} with median-based clipping and also extend this result to problems with structured non-symmetry for \algname{SGD} with smoothed median of means coupled with gradient clipping.

\section{Main Results}

The well-known \algname{Clipped-SGD} algorithm with the Gaussian DP mechanism (\algname{DP-Clipped-SGD}) is described in Algorithm~\ref{alg:Clipped-SGD}. If differential privacy (DP) is not required, one can simply set $\sigma_\omega^2 = 0$. As shown by \citet{sadiev2023high}, achieving exact convergence to the optimal solution of problem \eqref{min_problem} using \algname{Clipped-SGD} requires the clipping level to be chosen as $\lambda=\cO \left(\sigma \left(\nicefrac{K}{(\ln{\frac{K}{\beta}})}\right)^{\nicefrac{1}{\alpha}}\right)$. However, this choice of clipping level, which scales with the total number of iterations $K$, is problematic from a DP perspective. Specifically, larger clipping levels necessitate larger DP noise to maintain privacy, significantly increasing the variance in gradient estimates and leading to a larger convergence neighborhood.

To address this limitation, in this work, we focus on the more general case of arbitrary fixed clipping levels that do not scale with the total number of iterations. This approach is more compatible with practical DP requirements, where clipping levels are typically kept constant. However, our theoretical results can also accommodate clipping levels that scale with $K$ up to the order $\lambda=\cO \left(\sigma \left(\nicefrac{K}{(\ln{\frac{K}{\beta}})}\right)^{\nicefrac{1}{\alpha}}\right)$, as we discuss in detail in the appendix. This broader analysis introduces a few additional step-size conditions, which we also explore thoroughly in the supplementary material.

\begin{algorithm}[t]
\caption{\algname{DP-Clipped-SGD}}
\label{alg:Clipped-SGD}   
\begin{algorithmic}[1]
\REQUIRE starting point $x^0$, number of iterations $K$, step-size $\gamma > 0$, clipping level $\lambda$.
\FOR{$k=0,\ldots, K$}
\STATE Compute $\hat g_k = \clip\left(\nabla f_{\xi^k}(x^{k}), \lambda\right)$ using a fresh sample $\xi^k \sim \cD$
\STATE $\omega_k \sim \mathcal{N}(0,\sigma_\omega^2 I_d )$
\STATE $\widetilde g_k = \hat g_k + \omega_k$
\STATE $x^{k+1} = x^k - \gamma \tg_k$
\ENDFOR
\end{algorithmic}
\end{algorithm}

The following two theorems present our newly derived step-size bounds and the corresponding performance guarantees for both convex and non-convex settings. Following each theorem, we provide a table that further simplify the performance bounds under the assumption that the clipping level falls within specific intervals. In these tables, we assume that no DP noise is present, focusing purely on the impact of the clipping bias. The final corollary extend these results to the case where DP noise is included in the convex case, while the result for DP case in the non-convex setup is deffered to the supplementary materials due to space limitation. 


\paragraph{Convex problems.} We start with the convex case.

\begin{theorem}[Convergence of \algname{DP-Clipped-SGD} for the convex objectives]\label{convex_convergence}
    Let the integer $K\geq 0 $ and $\beta \in (0,1]$ be given. Furthermore, let Assumptions   \ref{ass:bounded_below}, \ref{ass:smoothness}, \ref{ass:convexity}, \ref{ass:oracle}, hold for $Q=B_{2R}(x^\star), \;R\geq \norm{x^0-x^\star}$. Set $\zeta_\lambda \eqdef \max \left \{0,2LR-\frac{\lambda}{2}\right \}$, and further assume that the step-size $\gamma$ is selected to satisfy
\begin{align}
    \gamma &\leq \cO \left (\min \left \{\frac{1}{L},\frac{R}{\lambda^{1-\alpha/2} \sqrt{K\ln \left (\frac{K}{\beta}\right )(\sigma^\alpha+\zeta_\lambda^\alpha)}},\right.\right. \nonumber\\ & \left.\left.\hspace{2cm}\frac{R\lambda^{\alpha-1}}{K(\sigma^\alpha+\zeta_\lambda^\alpha)\left(\frac{LR}{\lambda}+\frac{\lambda^{\alpha-1}\zeta_\lambda}{\sigma^\alpha+\zeta_\lambda^\alpha}+\left(\sigma^\alpha+\zeta_\lambda^\alpha\right)^{\frac{-1}{\alpha}}\right)},\frac{R}{\sigma_\omega \sqrt{dK\ln\left ( \frac{K}{\beta}\right)}} \right \}\right). \label{Stepszie_condition_for_convex_objective}
\end{align}
Then, after $K$ iterations of \algname{DP-Clipped-SGD}, the iterates with probability at least $1-\beta$ satisfy
     \begin{eqnarray}
                \min_{t\in [0,K]} f(x^t)-f(x^\star) \leq \frac{4R^2} {\gamma(K+1)}+ \frac{64LR^4}{\lambda^2\gamma^2 (K+1)^2}. \label{my_thm1}
     \end{eqnarray}
\end{theorem}

The convergence rate and the neighborhood to which the algorithm converges depend on the magnitude of $\lambda$ in a non-trivial way. Table~\ref{tab:convex} summarizes these relationships for different values of $\lambda$ in the absence of DP noise. In the special case where $\lambda = \cO \left(\sigma \left(\nicefrac{K}{\ln{\frac{K}{\beta}}}\right)^{\nicefrac{1}{\alpha}}\right)$, our theorem provides a convergence rate of $\cO \left( \left(\nicefrac{(\ln{\frac{K}{\beta}})}{K}\right)^{\nicefrac{(\alpha-1)}{\alpha}} + \nicefrac{(\ln{\frac{K}{\beta}})}{K}\right)$ to the exact solution in the asymptotic regime. This matches the rate previously derived by \citet{sadiev2023high}.

In contrast, if $\lambda$ is chosen as a constant, independent of $K$, the leading term in the convergence rate simplifies to $\cO(\sqrt{\nicefrac{(\ln \frac{K}{\beta})}{K}})$, which is faster than the more conservative bound $\cO \left( \left(\nicefrac{(\ln{\frac{K}{\beta}})}{K}\right)^{\nicefrac{(\alpha-1)}{\alpha}}\right)$. However, this faster rate comes at the cost of only guaranteeing convergence to a neighborhood around the optimal solution, determined by the third term in the step-size condition \eqref{Stepszie_condition_for_convex_objective}.

To ensure $(\varepsilon,\delta)$-DP for \algname{DP-Clipped-SGD} in our setting (i.e., expectation minimization), one can set the noise scale as $\sigma_\omega = \Theta\left(\frac{\lambda}{\varepsilon}\sqrt{K\ln\left(\frac{K}{\delta}\right)\ln\left(\frac{1}{\delta}\right)}\right)$ and apply the advanced composition theorem \citep[Theorem 3.22]{dwork2014algorithmic}. Given the fourth term in \eqref{Stepszie_condition_for_convex_objective}, this choice implies that the step-size decreases as $\nicefrac{1}{K}$, resulting in convergence to a certain neighborhood. This observation is formalized in the next corollary.


\begin{corollary}[Convergence of \algname{Clipped-SGD} for the convex objective] Let the assumptions of Theorem~\ref{convex_convergence} hold, $\sigma_\omega = \Theta\left(\frac{\lambda}{\varepsilon}\sqrt{K\ln\left(\frac{K}{\delta}\right)\ln\left(\frac{1}{\delta}\right)}\right)$, and  $\gamma$ is chosen as the minimum of \eqref{Stepszie_condition_for_convex_objective}. Then, with probability at least $1-\beta$
\begin{align}
                \min_{t\in [0,K]} f(x^t)-f(x^\star) \leq \cO & \left(\max \left \{\eqref{convex_first_decreasing_term} ,\eqref{convex_second_decreasing_term} ,\eqref{convex_bias_term},\eqref{DP_bias_term}\right\}\right). 
     \end{align}
     where
     \begin{eqnarray} 
        &\frac{LR^2}{K} + \frac{L^3R^4}{\lambda^2 K^2} \label{convex_first_decreasing_term}& \\
        &R\lambda^{1-\alpha/2}\sqrt{\frac{(\sigma^{\alpha} + \zeta_\lambda^\alpha)\ln{(\nicefrac{K}{\beta})}}{K}}+\frac{LR^2\lambda^\alpha (\sigma^\alpha + \zeta_\lambda^\alpha) \ln{(\nicefrac{K}{\beta})}}{K} \label{convex_second_decreasing_term}& \\
        &\frac{R(\sigma^\alpha+\zeta_\lambda^\alpha)\left(\frac{LR}{\lambda}+\frac{\lambda^{\alpha-1}\zeta_\lambda}{\sigma^\alpha+\zeta_\lambda^\alpha}+\left(\sigma^\alpha+\zeta_\lambda^\alpha\right)^{\frac{-1}{\alpha}}\right)} {\lambda^{\alpha-1}} +  \frac{R^2L(\sigma^\alpha+\zeta_\lambda^\alpha)^2\left(\frac{LR}{\lambda}+\frac{\lambda^{\alpha-1}\zeta_\lambda}{\sigma^\alpha+\zeta_\lambda^\alpha}+\left(\sigma^\alpha+\zeta_\lambda^\alpha\right)^{\frac{-1}{\alpha}}\right)^2} {\lambda^{2\alpha}} \label{convex_bias_term}&\\
        &\frac{R\lambda}{\varepsilon} \sqrt{d\ln{\left(\frac{K}{\beta}\right)}\ln{\left(\frac{K}{\delta}\right)}\ln{\left(\frac{1}{\delta}\right)}} + \frac{LR^2d\ln{\left(\frac{K}{\beta}\right)}\ln{\left(\frac{K}{\delta}\right)}\ln{\left(\frac{1}{\delta}\right)}}{\varepsilon^2}. \label{DP_bias_term}&
     \end{eqnarray}
\end{corollary}


One may notice that there is a non-trivial trade-off between the convergence rate, clipping level, and the size of the neighborhood. Therefore, we consider two special cases and provide the result with optimally selected $\lambda$ in the following corollary.

\begin{corollary}
    [Convergence of \algname{DP-Clipped-SGD} for the convex objective] Let the assumptions of Theorem~\ref{convex_convergence} hold, $K$ is sufficiently large, $\gamma$ is chosen as the minimum of \eqref{Stepszie_condition_for_convex_objective}, $\sigma_\omega = \Theta\left(\frac{\lambda}{\varepsilon}\sqrt{K\ln\left(\frac{K}{\delta}\right)\ln\left(\frac{1}{\delta}\right)}\right)$, and $\lambda > 4LR$. Then the optimal value for $\lambda$ is
    \begin{align}
  \lambda=\max \left \{ 4LR, \left(\frac{\varepsilon \sigma^\alpha}{d\ln\left(\frac{K}{\delta}\right)\ln\left(\frac{1}{\delta}\right)\ln{\frac{K}{\beta}}} \right)^{\frac{1}{\alpha}}\right \}. \notag  
\end{align}
    With this value, the iterates produced by the algorithm with probability of at least $1-\beta$ satisfy
    \begin{eqnarray}
\min_{k\in [0,K]} f(x^t)-f(x^\star)= \mathcal{ O} \left(\max \left \{\eqref{convex_first_term_large_lambda_DP_optimal_main}, \eqref{convex_second_term_large_lambda_DP_optimal_main}, \eqref{convex_third_term_large_lambda_DP_optimal_main},\eqref{convex_fourth_term_large_lambda_DP_optimal_main}\right \}\right), \notag
\end{eqnarray}
where
\begin{gather}
   \max \left \{\sqrt{\frac{R^{4-\alpha}L^{2-\alpha}\sigma^\alpha \ln \left(\frac{K}{\beta}\right)}{K}}  , R \left (\frac{\varepsilon \sigma^{\alpha}}{\sqrt{d\ln\left(\frac{K}{\delta}\right)\ln\left(\frac{1}{\delta}\right)}}\right)^{\frac{1}{\alpha}}\sqrt{\frac{ \ln^{\frac{3\alpha-2}{2\alpha}} \left(\frac{K}{\beta}\right)}{K}}\right \}\label{convex_first_term_large_lambda_DP_optimal_main} \\
   \min \left \{ \frac{R^{2-\alpha}\sigma^{\alpha}}{L^{\alpha-1}},R\sigma \left( \frac{\sqrt{d \ln\left(\frac{K}{\delta}\right)\ln\left(\frac{1}{\delta}\right)}}{\varepsilon}\right)^{\frac{\alpha-1}{\alpha}} \right \}
\label{convex_second_term_large_lambda_DP_optimal_main} \\
   \min \left \{\frac{LR^2}{K^2} , \frac{L^3 R^4 \left(d \ln\left(\frac{K}{\delta}\right)\ln\left(\frac{1}{\delta}\right)  \ln \left(\frac{K}{\beta}\right)\right)^{\frac{1}{\alpha}}}{(\varepsilon)^{\frac{1}{\alpha}}\sigma K^2} \right \}+ \frac{LR^2}{K} \label{convex_third_term_large_lambda_DP_optimal_main} \\
   \max \left \{\frac{LR^2}{\varepsilon}\sqrt{d\ln\left(\frac{K}{\delta}\right)\ln\left(\frac{1}{\delta}\right) \ln \left(\frac{K}{\beta}\right)}, \frac{R\sigma \left (d \ln\left(\frac{K}{\delta}\right)\ln\left(\frac{1}{\delta}\right)  \ln \left(\frac{K}{\beta}\right)\right)^{\frac{\alpha+2}{2\alpha}}}{\varepsilon^{\frac{\alpha-1}{\alpha}}} \right \} \notag\\ +\frac{LR^2d}{\varepsilon^2}\ln\left(\frac{K}{\delta}\right)\ln\left(\frac{1}{\delta}\right) \ln \left(\frac{K}{\beta}\right).\label{convex_fourth_term_large_lambda_DP_optimal_main}
\end{gather}
Also, for small $\lambda$ regime $\left(\lambda \leq \frac{4}{3}LR\right)$, the optimal value for $\lambda$ is
\begin{align}
    \lambda= \min \left \{\frac{4}{3}LR, \frac{2\varepsilon LR}{\left(d \ln\left(\frac{K}{\delta}\right)\ln\left(\frac{1}{\delta}\right)\ln{\frac{K}{\beta}}\right)^{\frac{1}{2\alpha+2}} +1}\right \}.
\end{align}
 With this value, the iterates produced by the algorithm with probability of at least $1-\beta$ satisfy
 \begin{align}
    \min_{t\in [0,K]} f(x^t)-f(x^\star)= \mathcal{O} \left(\max \left\{ \eqref{convex_first_term_small_lambda_DP_optimal_main}, \eqref{convex_second_term_small_lambda_DP_optimal_main},\eqref{convex_third_term_small_lambda_DP_optimal_main}, \eqref{convex_fourth_term_small_lambda_DP_optimal_main}\right\}\right) \notag,
\end{align}
where
\begin{gather}
    \min \left \{\sqrt{ \frac{R^{4-\alpha}L^{2-\alpha}\sigma^\alpha\ln \left(\frac{K}{\beta}\right)}{K}}, \sqrt{\frac{R^{4-\alpha}(\varepsilon L)^{2-\alpha}\ln^{\frac{3\alpha}{4\alpha+4}} \left(\frac{K}{\beta}\right)}{(d \ln\left(\frac{K}{\delta}\right)\ln\left(\frac{1}{\delta}\right))^{\frac{2-\alpha}{4\alpha+4}}K}} \right \}\label{convex_first_term_small_lambda_DP_optimal_main} \\
    \max \left \{\frac{R^{2-\alpha}\sigma^{\alpha}}{L^{\alpha-1}} , \frac{R^{2-\alpha}\sigma^\alpha}{\varepsilon}\left(d \ln\left(\frac{K}{\delta}\right)\ln\left(\frac{1}{\delta}\right)\ln \left(\frac{K}{\beta}\right)\right)^{\frac{\alpha-1}{2\alpha+2}}\right\}\label{convex_second_term_small_lambda_DP_optimal_main} \\
    \max \left \{ \frac{LR^2}{K^2}, \frac{LR^2}{\varepsilon^2 K^2}\left(d \ln\left(\frac{K}{\delta}\right)\ln\left(\frac{1}{\delta}\right)\ln \left(\frac{K}{\beta}\right)\right)^{\frac{2}{2\alpha+2}} \right\}+\frac{LR^2}{K}
\label{convex_third_term_small_lambda_DP_optimal_main}\\
   \min \left \{\frac{LR^2}{\varepsilon}\sqrt{d\ln\left(\frac{K}{\delta}\right)\ln\left(\frac{1}{\delta}\right)\ln \left(\frac{K}{\beta}\right)}, \frac{ LR^2}{\left(d \ln\left(\frac{K}{\delta}\right)\ln\left(\frac{1}{\delta}\right)\ln \left(\frac{K}{\beta}\right)\right)^{\frac{1}{2\alpha+2}}} \right \} \notag\\ + \frac{LR^2d}{\varepsilon^2}\ln\left(\frac{K}{\delta}\right)\ln\left(\frac{1}{\delta}\right)\ln\left(\frac{K}{\beta}\right).
   \label{convex_fourth_term_small_lambda_DP_optimal_main}
\end{gather}
\end{corollary}

In the finite-sum case, i.e., when $f(x) = \frac{1}{n}\sum_{i=1}^nf_i(x)$ for some finite $n$, \citet{abadi2016deep} show that it is sufficient to choose $\sigma_\omega = \Theta \left(\frac{q\lambda}{\varepsilon}\sqrt{K \ln{\frac{1}{\delta}}}\right)$, where $q = \nicefrac{b}{n}$, $b$ is the mini-batch size, clipping is applied to each stochastic gradient, and $\varepsilon = \cO(q^2 K)$, allowing to have smaller $\varepsilon$ and $\delta$ for given $\sigma_\omega$ and $\lambda$. We note that our analysis holds for the finite-sum case without changes as long as the assumptions of the theorem are satisfied and the mini-batch size equals $1$.


\begin{table}[t!]
    \centering
    \caption{Rate, neighborhood and optimal $\lambda$ in different regimes for the convex objective function. Here, $\lambda$ denotes the clipping level, $L$ denotes the smoothness parameter, $R\geq \|x^0-x^\ast\|$ represents the initial error, $\alpha\in (1,2]$ denotes the moment that is bounded and $\sigma^{\alpha}$ is that upper bound value. Furthermore, $\beta$ is the confidence level, $\zeta_{\lambda}:=\max\{0,2LR-\frac{\lambda}{2}\}$, and $\eta$ is a small positive constant. By optimal $\lambda$ and optimal neighborhood, we refer to the $\lambda$ that minimizes the right hand side (RHS) of \eqref{my_thm1} and the minimized RHS value itself, respectively. }
    \vspace{0.1cm}
    \resizebox{\textwidth}{!}{
    \begin{tabular}{|c|c c c c|}
    \hline
    Regime & Neighborhood & Optimal $\lambda$ & Convergence rate &  Optimal Neighborhood \\
    \hline\hline
     {\makecell{$\lambda >4LR$\\ ($\zeta_\lambda = 0$)}} & $\cO \left(R\frac{\sigma^{\alpha}}{\lambda^{\alpha-1}}+ LR^2 \frac{\sigma^{2\alpha}}{\lambda^{2\alpha}}\right)$  &$\cO \left(\sigma \left(\frac{K}{\ln{\frac{K}{\beta}}}\right)^\frac{1}{\alpha}\right)$ & $ \cO \left( \left(\frac{\ln{\frac{K}{\beta}}}{K}\right)^{\frac{\alpha-1}{\alpha}} + \frac{\ln^2{\frac{K}{\beta}}}{K^2}\right)$ & - \\
     \hline
     {\makecell{$\frac{4}{3}LR <\lambda \leq 4LR$\\ $\zeta_\lambda < \lambda < \sigma$}} & $\cO \left(R\frac{\sigma^{\alpha}}{\lambda^{\alpha-1}}+ LR^2 \frac{\sigma^{2\alpha}}{\lambda^{2\alpha}}\right)$ &$4LR$ & $ \cO \left( \sqrt{\frac{\ln{\frac{K}{\beta}}}{K}} + \frac{\ln{\frac{K}{\beta}}}{K}\right)$ &  $\cO \left(\frac{R^{2-\alpha} \sigma^\alpha}{L^{\alpha-1}}+ \frac{\sigma^{2\alpha}}{L^{2\alpha-1}R^{2\alpha-2}}\right)$ \\
     \hline
     \multirow{4.5}{*}{\makecell{$\frac{4}{3}LR <\lambda \leq 4LR$\\ $\zeta_\lambda < \sigma < \lambda$}} & $\cO \left(R\frac{\sigma^{\alpha}}{\lambda^{\alpha-1}}+ LR^2 \frac{\sigma^{2\alpha}}{\lambda^{2\alpha}}\right)$ &$4LR$ & $ \cO \left( \sqrt{\frac{\ln{\frac{K}{\beta}}}{K}} + \frac{\ln{\frac{K}{\beta}}}{K}\right)$ &  $\cO \left(\frac{R^{2-\alpha} \sigma^\alpha}{L^{\alpha-1}}+ \frac{\sigma^{2\alpha}}{L^{2\alpha-1}R^{2\alpha-2}}\right)$ \\ &$\cO \left(R\zeta_\lambda + \frac{LR^2\zeta_\lambda^2}{\lambda^2}\right)$ & $4LR-\eta$ & $\cO \left( \sqrt{\frac{\ln{\frac{K}{\beta}}}{K}} + \frac{\ln{\frac{K}{\beta}}}{K}\right)$ & $\cO \left({R\eta}+ \frac{LR^2\eta^2}{(LR-\eta)^2}\right) $\\
     \hline
     {\makecell{$\frac{4}{3}LR<\lambda \leq 4LR$\\ ($\sigma <\zeta_\lambda < \lambda$)}} & $\cO \left(R\zeta_\lambda + \frac{LR^2\zeta_\lambda^2}{\lambda^2}\right)$ & $4LR-2\sigma$ & $\cO \left( \sqrt{\frac{\ln{\frac{K}{\beta}}}{K}} + \frac{\ln{\frac{K}{\beta}}}{K}\right)$ & $\cO \left(R\sigma + \frac{LR^2\sigma^2}{(LR-\sigma)^2}\right)$\\
     \hline
     {\makecell{$\lambda \leq \frac{4}{3}LR$\\ ($\lambda <\zeta_\lambda < \sigma$)}} & $\cO \left(R\frac{\sigma^\alpha \zeta_\lambda}{\lambda^\alpha}+ \frac{LR^2\sigma^{2\alpha}\zeta_\lambda^2}{\lambda^{2\alpha+2}}\right)$ & $\frac{4}{3}LR$ & $\cO \left( \sqrt{\frac{\ln{\frac{K}{\beta}}}{K}} + \frac{\ln{\frac{K}{\beta}}}{K}\right)$ & $\cO \left(\frac{R^{2-\alpha}\sigma^\alpha}{L^{\alpha-1}}+ \frac{\sigma^{2\alpha}}{L^{2\alpha-1}R^{2\alpha-2}}\right)$ \\
     \hline
      {\makecell{$\lambda \leq \frac{4}{3}LR$\\ ($\lambda <\sigma < \zeta_\lambda$)}} & $\cO \left(R\frac{\zeta_\lambda^{\alpha+1}}{\lambda^\alpha}+ \frac{LR^2\zeta_\lambda^{2\alpha}}{\lambda^{2\alpha+2}}\right)$ & $\frac{4}{3}LR-\eta$ & $\cO \left( \sqrt{\frac{\ln{\frac{K}{\beta}}}{K}} + \frac{\ln{\frac{K}{\beta}}}{K}\right)$ & $\cO \left(\frac{R(LR+\eta)^{\alpha+1}}{(LR - \eta)^\alpha} + \frac{LR^2(LR+\eta)^{2\alpha}}{(LR-\eta)^{2\alpha+2}}\right)$\\
      \hline
      \multirow{4.5}{*}{\makecell{$\lambda \leq \frac{4}{3}LR$\\ ($\sigma < \lambda < \zeta_\lambda$)}} & $\cO \left(R\frac{\zeta_\lambda^{\alpha+1}}{\lambda^\alpha}+ \frac{LR^2\zeta_\lambda^{2\alpha}}{\lambda^{2\alpha+2}}\right)$ & $\frac{4}{3}LR-\eta$ & $\cO \left( \sqrt{\frac{\ln{\frac{K}{\beta}}}{K}} + \frac{\ln{\frac{K}{\beta}}}{K}\right)$ & $\cO \left(\frac{R(LR+\eta)^{\alpha+1}}{(LR - \eta)^\alpha} + \frac{LR^2(LR+\eta)^{2\alpha}}{(LR-\eta)^{2\alpha+2}}\right)$\\ &$\cO \left(R\frac{\sigma\zeta_\lambda^{\alpha-1}}{\lambda^{\alpha-1}}+\frac{LR^2\sigma^2\zeta_\lambda^{2\alpha-2}}{\lambda^{2\alpha}}\right)$ & $\frac{4}{3}LR$ & $\cO \left( \sqrt{\frac{\ln{\frac{K}{\beta}}}{K}} + \frac{\ln{\frac{K}{\beta}}}{K}\right)$ & $\cO \left( R\sigma +\frac{\sigma^2}{L}\right)$ \\
      \hline 
    \end{tabular}
    }
    \label{tab:convex}
\end{table}


\paragraph{Non-convex problems.} In the non-convex case, we derive the following result.

\begin{theorem}[Convergence of \algname{DP-Clipped-SGD} for the non-convex objective]\label{non-convex-convergence}
    Let the integer $K\geq 0 $ and $\beta \in (0,1]$ be given. Let the assumptions   \ref{ass:bounded_below}, \ref{ass:smoothness}, \ref{ass:oracle}, hold for the set $Q$ defined as $Q  =\left\{x\in \mathbb{R^d}~|~\exists  ~y \in \mathbb{R}^d: f(y)\leq f^\ast + 2\Delta \text{  and  } \|x-y\|\leq \nicefrac{\sqrt{\Delta}}{20\sqrt{L}}\right\}$, where $\Delta\geq f(x^0)-f^\ast$, $\zeta_\lambda \eqdef \max \left \{0,2\sqrt{L\Delta}-\frac{\lambda}{2}\right \}$, and $\gamma$ is selected according to
\begin{align}
    \gamma &\leq \cO \left (\min \left \{\frac{1}{L}, \frac{\sqrt{\frac{\Delta}{L}}}{\lambda^{1-\alpha/2 }\sqrt{K\ln\left ( \frac{K}{\beta}\right)(\sigma^\alpha+\zeta_\lambda^\alpha)}}\right.\right.,\nonumber \\ & \left.\left.\hspace{3.5cm}\frac{\sqrt{\frac{\Delta}{L}}\lambda^{\alpha-1}}{K(\sigma^\alpha+\zeta_\lambda^\alpha)\left(\frac{\sqrt{L\Delta}}{\lambda}+\frac{\lambda^{\alpha-1}\zeta_\lambda}{\sigma^\alpha+\zeta_\lambda^\alpha}+\left(\sigma^\alpha+\zeta_\lambda^\alpha\right)^{\frac{-1}{\alpha}}\right)},\frac{\sqrt{\frac{\Delta}{L}}}{\sigma_\omega \sqrt{dK\ln\left ( \frac{K}{\beta}\right)}} \right \}\right). \label{Stepszie_condition_for_non_convex_objective}
\end{align}
Then, after $K$ iterations of \algname{DP-Clipped-SGD} and with probability at least $1-\beta$, we have
     \begin{eqnarray}
                \min_{t\in [0,K]} \norm{\nabla f(x^t)}^2 \leq \frac{8\Delta}{\gamma(K+1)}+ \frac{128\Delta^2}{\lambda^2\gamma^2 (K+1)^2}
                \label{my_thm2}
     \end{eqnarray}
\end{theorem}

Similarly to the convex case, the above result establishes the convergence to a certain neighborhood with a faster $\cO(\nicefrac{1}{\sqrt{K}})$ rate. We defer the corollaries for the non-convex case to the appendix and describe different special cases for the no-DP regime in Table~\ref{tab:non_convex}.

\begin{corollary}[Convergence of \algname{DP-Clipped-SGD} for the non-convex objective]
Let the assumption of Theorem~\ref{non-convex-convergence} hold, and  $\gamma$ is chosen as the minimum of \eqref{Stepszie_condition_for_non_convex_objective}. Then, with probability at least $1-\beta$
\begin{align}
      \min_{t\in [0,K]} \norm{\nabla f(x^t)}^2 \leq \cO\left(\max\left\{ \eqref{non_convex_first_decreasing_term} ,\eqref{non_convex_second_decreasing_term} ,\eqref{non_convex_bias_term},\eqref{non_convex_DP_bias_term}\right\}\right),
\end{align}
where
\begin{eqnarray}
    & \frac{L\Delta}{K} + \frac{L^2\Delta^2}{\lambda^2 K^2}\label{non_convex_first_decreasing_term}\\
    & \sqrt{L\Delta}\lambda^{1 - \alpha/2}\sqrt{\frac{(\sigma^\alpha + \zeta_\lambda^\alpha)\ln\nicefrac{K}{\beta}}{K}} + \frac{L\Delta(\sigma^\alpha + \zeta_\lambda^\alpha)\ln(\nicefrac{K}{\beta})}{\lambda^\alpha K}\label{non_convex_second_decreasing_term}\\
    &\frac{\sqrt{\Delta L}(\sigma^\alpha+\zeta_\lambda^\alpha)\left(\frac{\sqrt{L\Delta}}{\lambda}+\frac{\lambda^{\alpha-1}\zeta_\lambda}{\sigma^\alpha+\zeta_\lambda^\alpha}+\left(\sigma^\alpha+\zeta_\lambda^\alpha\right)^{\frac{-1}{\alpha}}\right)}{\lambda^{\alpha-1}}+\frac{\Delta L(\sigma^\alpha+\zeta_\lambda^\alpha)^2\left(\frac{\sqrt{L\Delta}}{\lambda}+\frac{\lambda^{\alpha-1}\zeta_\lambda}{\sigma^\alpha+\zeta_\lambda^\alpha}+\left(\sigma^\alpha+\zeta_\lambda^\alpha\right)^{\frac{-1}{\alpha}}\right)^2}{\lambda^{2\alpha}}\label{non_convex_bias_term}\\
    & \frac{\sigma_\omega\sqrt{dL\Delta \ln(\nicefrac{K}{\beta})}}{\sqrt{K}} + \frac{\sigma_\omega^2 dL\Delta \ln(\nicefrac{K}{\beta})}{\lambda^2 K}.\label{non_convex_DP_bias_term}
\end{eqnarray}
\end{corollary}

\begin{table}[t!]
    \centering
    \caption{Rate, neighborhood and optimal $\lambda$ in different regimes for the non-convex objective function. Here, $\lambda$ denotes the clipping level, $L$ denotes the smoothness parameter, $\Delta \geq f(x^0)-f(x^\ast)$ represents the initial error, $\alpha\in (1,2]$ denotes the moment that is bounded and $\sigma^{\alpha}$ is that upper bound value. Furthermore, $\beta$ is the confidence level,  $\zeta_{\lambda}:=\max\{0,2\sqrt{L\Delta}-\frac{\lambda}{2}\}$,and $\eta$ is a small positive constant. By optimal $\lambda$ and optimal neighborhood, we refer to the $\lambda$ that minimizes the right hand side (RHS) of \eqref{my_thm2} and the minimized RHS value itself, respectively.} 
    \vspace{0.1cm}
    \resizebox{\textwidth}{!}{
    \begin{tabular}{|c|c c c c|}
    \hline
    Regime & Neighborhood & Optimal $\lambda$ & Convergence rate &  Optimal Neighborhood \\
    \hline\hline
     {\makecell{$\lambda > 4\sqrt{L\Delta}$\\ ($\zeta_\lambda = 0$)}} & $\cO \left(\sqrt{L\Delta}\frac{\sigma^{\alpha}}{\lambda^{\alpha-1}}+ L\Delta  \frac{\sigma^{2\alpha}}{\lambda^{2\alpha}}\right)$  &$\cO \left(\sigma \left(\frac{K}{\ln{\frac{K}{\beta}}}\right)^\frac{1}{\alpha}\right)$ & $ \cO \left( \left(\frac{\ln{\frac{K}{\beta}}}{K}\right)^{\frac{\alpha-1}{\alpha}} + \frac{\ln^2{\frac{K}{\beta}}}{K^2}\right)$ & - \\
     \hline
     {\makecell{$\frac{4}{3} \sqrt{L\Delta} <\lambda \leq 4\sqrt{L\Delta}$\\ $\zeta_\lambda < \lambda < \sigma$}} & $\cO \left(\sqrt{L\Delta}\frac{\sigma^{\alpha}}{\lambda^{\alpha-1}}+ L\Delta  \frac{\sigma^{2\alpha}}{\lambda^{2\alpha}}\right)$ &$4 \sqrt{L\Delta}$ & $ \cO \left( \sqrt{\frac{\ln{\frac{K}{\beta}}}{K}} + \frac{\ln{\frac{K}{\beta}}}{K}\right)$ &  $\cO \left({\frac{\sigma^\alpha}{(\sqrt{L\Delta})^{\alpha-2}}}+ \frac{\sigma^{2\alpha}}{{(\sqrt{L\Delta})^{2\alpha-2}}}\right)$ \\
     \hline
     \multirow{4.5}{*}{\makecell{$\frac{4}{3} \sqrt{L\Delta} <\lambda \leq 4\sqrt{L\Delta}$\\ $\zeta_\lambda < \lambda < \sigma$}} & $\cO \left(\sqrt{L\Delta}\frac{\sigma^{\alpha}}{\lambda^{\alpha-1}}+ L\Delta  \frac{\sigma^{2\alpha}}{\lambda^{2\alpha}}\right)$ &$4 \sqrt{L\Delta}$ & $ \cO \left( \sqrt{\frac{\ln{\frac{K}{\beta}}}{K}} + \frac{\ln{\frac{K}{\beta}}}{K}\right)$ &  $\cO \left({\frac{\sigma^\alpha}{(\sqrt{L\Delta})^{\alpha-2}}}+ \frac{\sigma^{2\alpha}}{{(\sqrt{L\Delta})^{2\alpha-2}}}\right)$ \\ &$\cO \left(\sqrt{L\Delta}\zeta_\lambda + \frac{{L\Delta} \zeta_\lambda^2}{\lambda^2}\right)$ & $ 4\sqrt{L\Delta}-\eta$ & $\cO \left( \sqrt{\frac{\ln{\frac{K}{\beta}}}{K}} + \frac{\ln{\frac{K}{\beta}}}{K}\right)$ & $\cO \left({\sqrt{L\Delta}\eta}+ \frac{{L\Delta} \eta^2}{(\sqrt{L\Delta}-\eta)^2}\right) $\\
     \hline
     {\makecell{$\frac{4}{3}  \sqrt{L\Delta} <\lambda \leq 4 \sqrt{L\Delta}$\\ ($\sigma <\zeta_\lambda < \lambda$)}} & $\cO \left(\sqrt{L\Delta}\zeta_\lambda + \frac{{L\Delta} \zeta_\lambda^2}{\lambda^2}\right)$ & $4\sqrt{L\Delta}-2\sigma$ & $\cO \left( \sqrt{\frac{\ln{\frac{K}{\beta}}}{K}} + \frac{\ln{\frac{K}{\beta}}}{K}\right)$ & $\cO \left(\sqrt{L\Delta}\sigma + \frac{{L\Delta} \sigma^2}{(\sqrt{L\Delta}-\sigma)^2}\right)$\\
     \hline
     {\makecell{$\lambda \leq \frac{4}{3} \sqrt{L\Delta}$\\ ($\lambda <\zeta_\lambda < \sigma$)}} & $\cO \left(\sqrt{L\Delta}\frac{\sigma^\alpha \zeta_\lambda}{\lambda^\alpha}+ \frac{{L\Delta} \sigma^{2\alpha}\zeta_\lambda^2}{\lambda^{2\alpha+2}}\right)$ & $\frac{4}{3} \sqrt{L\Delta}$ & $\cO \left( \sqrt{\frac{\ln{\frac{K}{\beta}}}{K}} + \frac{\ln{\frac{K}{\beta}}}{K}\right)$ & $\cO \left({\frac{\sigma^\alpha}{(\sqrt{L\Delta})^{\alpha-2}}}+ \frac{\sigma^{2\alpha}}{{(\sqrt{L\Delta})^{2\alpha-2}}}\right)$ \\
     \hline
      {\makecell{$\lambda \leq \frac{4}{3} \sqrt{L\Delta}$\\ ($\lambda <\sigma < \zeta_\lambda$)}} & $\cO \left(\sqrt{L\Delta}\frac{\zeta_\lambda^{\alpha+1}}{\lambda^\alpha}+ \frac{{L\Delta} \zeta_\lambda^{2\alpha}}{\lambda^{2\alpha+2}}\right)$ & $\frac{4}{3} \sqrt{L\Delta}-\eta$ & $\cO \left( \sqrt{\frac{\ln{\frac{K}{\beta}}}{K}} + \frac{\ln{\frac{K}{\beta}}}{K}\right)$ & $\cO \left(\frac{\sqrt{L\Delta}(\sqrt{L\Delta}+ \eta)^{\alpha+1}}{(\sqrt{L\Delta}-\eta)^{\alpha}} + \frac{L\Delta(\sqrt{L\Delta}+ \eta)^{2\alpha}}{(\sqrt{L\Delta}-\eta)^{2\alpha+2}}\right)$\\
      \hline
      \multirow{4.5}{*}{\makecell{$\lambda \leq \frac{4}{3} \cdot 4\sqrt{L\Delta}$\\ ($\sigma < \lambda < \zeta_\lambda$)}} & $\cO \left(\sqrt{L\Delta}\frac{\zeta_\lambda^{\alpha+1}}{\lambda^\alpha}+ \frac{{L\Delta} \zeta_\lambda^{2\alpha+2}}{\lambda^{2\alpha+2}}\right)$ & $\frac{4}{3} \sqrt{L\Delta}-\eta$ & $\cO \left( \sqrt{\frac{\ln{\frac{K}{\beta}}}{K}} + \frac{\ln{\frac{K}{\beta}}}{K}\right)$ & $\cO \left(\frac{\sqrt{L\Delta}(\sqrt{L\Delta}+ \eta)^{\alpha+1}}{(\sqrt{L\Delta}-\eta)^{\alpha}} + \frac{L\Delta(\sqrt{L\Delta}+ \eta)^{2\alpha}}{(\sqrt{L\Delta}-\eta)^{2\alpha+2}}\right)$\\ &$\cO \left(\sqrt{L\Delta}\frac{\sigma\zeta_\lambda^{\alpha-1}}{\lambda^{\alpha-1}}+L\Delta\frac{ \sigma^2\zeta_\lambda^{2\alpha-2}}{\lambda^{2\alpha}}\right)$ & $\frac{4}{3} \sqrt{L\Delta}$ & $\cO \left( \sqrt{\frac{\ln{\frac{K}{\beta}}}{K}} + \frac{\ln{\frac{K}{\beta}}}{K}\right)$ & $\cO \left( \sqrt{L\Delta}\sigma +{\sigma^2}{}\right)$ \\
      \hline 
    \end{tabular}
    }
     \label{tab:non_convex}
\end{table}

\paragraph{Comparison with the results by \citet{koloskova2023revisiting}.} \citet{koloskova2023revisiting} derive their \emph{in-expectation} convergence result under the $(L_0,L_1)$-smoothness assumption \citep{zhang2019gradient} and the $\sigma^2$-uniformly bounded variance assumption (i.e., Assumption~\ref{ass:oracle} with $\alpha = 2$), for \algname{DP-Clipped-SGD} with mini-batching. For ease of comparison, we consider the special case $L_1 = 0$ and $L_0 = L$, which corresponds to standard $L$-smoothness. Moreover, for simplicity, we assume a mini-batch size of $1$. In this setting, the result from \citet[Appendix C.4.2]{koloskova2023revisiting} for \algname{DP-Clipped-SGD} can be written as follows: if $\gamma \leq \nicefrac{1}{9L}$, then
\begin{equation*}
    \min\limits_{t\in [0,K]}\left(\EE\left[\|\nabla f(x^t)\|\right]\right)^2 \leq \cO\left(\frac{\Delta}{\gamma K} + \frac{\Delta^2}{\lambda^2\gamma^2K^2} + \gamma L \sigma^2 +  \min\left\{\sigma^2, \frac{\sigma^4}{\lambda^2}\right\} + \gamma L d\sigma_\omega^2 + \frac{\gamma^2 L^2 d^2 \sigma_\omega^4}{\lambda^2} \right).
\end{equation*}

The structure of our bound is quite similar. Specifically, the terms from \eqref{non_convex_first_decreasing_term} correspond to the convergence of \algname{DP-Clipped-SGD} in the noiseless regime ($\sigma = \sigma_\omega = 0$) and match the $\cO\left(\frac{\Delta}{\gamma K} + \frac{\Delta^2}{\lambda^2\gamma^2K^2}\right)$ part when $\gamma = \Theta(\nicefrac{1}{L})$. Next, the terms in \eqref{non_convex_second_decreasing_term} serve as analogs of the $\cO(\gamma L \sigma^2)$ term. The leading term in \eqref{non_convex_second_decreasing_term} matches the $K$-dependence of $\cO(\gamma L \sigma^2)$ for $\gamma = \Theta(\nicefrac{1}{\sqrt{K}})$. However, these terms also depend on the clipping level $\lambda$, which arises from our high-probability convergence analysis and the presence of heavy-tailed noise.

The key difference lies in the terms stemming from the inherent bias of \algname{Clipped-SGD} \citep[Theorems 3.1–3.2]{koloskova2023revisiting} and the DP noise. In our result, these bias terms appear in \eqref{non_convex_bias_term}, while the corresponding term in \citet{koloskova2023revisiting} is $\cO\left( \min\left\{\sigma^2, \frac{\sigma^4}{\lambda^2}\right\} \right)$. As shown in Table~\ref{tab:non_convex}, in the special case $\lambda > 4\sqrt{L\Delta}$, the bias terms (i.e., the convergence neighborhood when $\sigma_\omega = 0$) in \eqref{non_convex_bias_term} reduce to $\cO\left(\sqrt{L\Delta}\frac{\sigma^\alpha}{\lambda^{\alpha-1}} + L\Delta \frac{\sigma^{2\alpha}}{\lambda^{2\alpha}}\right)$. Assuming $\lambda > \sigma$ for simplicity, the term from \citet{koloskova2023revisiting} becomes $\cO\left(\frac{\sigma^4}{\lambda^2}\right)$, which is strictly larger than the second term and strictly smaller than the first term in our bound when $\alpha = 2$. Furthermore, in this regime, both terms in our bound decrease with increasing $\alpha$, suggesting that the convergence neighborhood grows with the heaviness of the noise. Whether the bound in \eqref{non_convex_bias_term} is tight and whether improvements are possible in other regimes remain open questions.

Finally, ignoring logarithmic factors (introduced by the high-probability analysis), the DP-noise-related terms in our bound \eqref{non_convex_DP_bias_term} are $\tilde{\cO}\left(\frac{\sigma_\omega\sqrt{dL\Delta}}{\sqrt{K}} + \frac{\sigma_\omega^2 dL\Delta}{\lambda^2 K}\right)$, while the corresponding terms in \citet{koloskova2023revisiting} are $\cO\left(\gamma L d\sigma_\omega^2 + \frac{\gamma^2L^2d^2\sigma_\omega^4}{\lambda^2}\right)$. Setting $\gamma = \sqrt{\nicefrac{\Delta}{LdK}}$ yields the latter bound as $\cO\left(\frac{\sigma_\omega \sqrt{dL\Delta}}{\sqrt{K}} + \frac{\sigma_\omega^4 d L\Delta}{\lambda^2 K}\right)$, which matches  \eqref{non_convex_DP_bias_term} up to logarithmic factors.

\begin{proof}[Proof sketch of our main results]
    The proof of Theorems \ref{convex_convergence} and \ref{non-convex-convergence} is heavily inspired by \citep{sadiev2023high}. Yet, there is a crucial difference in defining the clipping level parameter. In contrast to \citep{sadiev2023high}, we treat $\lambda$ as given rather than calculating it based on other problem parameters. By doing so, the fundamental assumption regarding the magnitude of $\lambda$ in comparison to the norm of the gradient in bias-variance of the clipped vector (Lemma 5.1) of \citep{sadiev2023high} becomes invalid. Thus, we develop a general bias-variance lemma (Lemma~\ref{lem:Bias-Variance}) to study the statistical properties of the clipped vector.
\end{proof}

\section{Conclusion}

In this paper, we present the first high-probability convergence analysis of \algname{DP-Clipped-SGD} for both convex and non-convex smooth optimization problems under heavy-tailed noise. Our results demonstrate that \algname{DP-Clipped-SGD} converges to a certain neighborhood of the optimal solution at a rate of $\cO(\nicefrac{1}{\sqrt{K}})$. In future work, it would be valuable to extend these results to the Federated Learning setting and to investigate the tightness and optimality of the derived bounds.

\bibliography{refs}

\clearpage

\appendix

\section{Notation Table and Auxiliary Facts}

To facilitate the readability of the proofs, we provide a notation table below\footnote{We fixed minor typos in Table~\ref{tab:non_convex} from the main part of the paper. Changes are highlighted using \textcolor{red}{red color}.}.

\begin{center}
    \begin{table}[h]
    \caption{Our notation.}
        \centering
        \begin{tabular}{c  c }
        \hline
        Notation & Explanation  \\
        \hline
             $g_t $& Stochastic gradient  \\
             $\hat g_t$ & Clipped stochastic gradient \\ 
             $\tilde g_t$& Clipped stochastic gradient after DP noise injection  \\
             $c_t$ & $\min \left \{1, \frac{\lambda}{2\norm{\nabla f(x^t)}} \right\}$\\
             $\omega_t$ & Injected DP noise at iteration $t$ \\
             $\beta$ & Confidence level/failure probability \\
             \multirow{2.5}{*}{$\zeta_\lambda$} & Convex case: $\max \left\{0,2LR-\frac{\lambda}{2}\right\}$ \\ &Non-convex case: $\max \left\{0,2\sqrt{L\Delta}-\frac{\lambda}{2}\right\}$ \\
             $\cF^t$& Filtration up to the time $t$ \\
             $\sigma$ & Gradient noise parameter \\
             $\sigma_\omega$ & DP noise parameter\\
             $R$ & Upper bound on $\norm{x^0-x^\ast}$ for convex functions \\
             $\Delta$ &Upper bound on $f(x^0)-f^\ast$ for non-convex functions \vspace{0.1 cm}\\
             \hline \vspace{0.1cm}
        \end{tabular}
         \label{tab:my_label}
    \end{table}
\end{center}

\paragraph{Auxiliary facts.} Let $(\Omega, \mathcal{F}, \mathbb{P})$ be a probability space. A sequence $\{\mathcal{F}_i\}_{i\geq 1}$ of nested sigma algebras in $\mathcal{F}$ (i.e., $\mathcal{F}_i \subset \mathcal{F}_{i+1} \subset \mathcal{F}$) is called a filtration, in which case $(\Omega, \mathcal{F}, \{\mathcal{F}_i\}_{i\geq 1}, \mathbb{P})$ is called a filtered probability space. A sequence of random variables $\{X_i\}_{i\geq 1}$ is said to be adapted to $\{\mathcal{F}_i\}_{i\geq 1}$ if each $X_i$ is $\mathcal{F}_i$-measurable. Furthermore, if $\mathbb{E}[X_i \mid \mathcal{F}_{i-1}] = X_{i-1} \ \forall i$, then $\{X_i\}_{i\geq 1}$ is called a martingale. On the other hand, if $\mathbb{E}[X_i \mid \mathcal{F}_{i-1}] = 0 \ \forall i$, then $\{X_i\}_{i\geq 1}$ is called a martingale difference sequence.

One of the very useful tools in establishing high probability convergence guarantees in this work is the following lemma, which is known as the Bernstein inequality for martingale difference sequences \citep{freedman1975tail}, \citep{dzhaparidze2001bernstein}.

\begin{lemma}\label{lem: Bernstein_inequality}
         Let the sequence of random variables \(\left\{X_i\right\}_{i\geq 1}\) form a martingale difference sequence on the filtered probability space $(\Omega, \mathcal{F}, \{\mathcal{F}_i\}_{i\geq 1}, \mathbb{P})$.  Assume that conditional variances \(\sigma_i^2 := \mathbb{E}\left[X_i^2 |\mathcal{F}_{i-1}\right]\) exist and
are bounded. Furthermore, there exists a deterministic constant \(c\geq 0\) such that \(\lvert X_i \rvert \leq c\) almost surely for all \(i\geq 0\). Then for all \(b > 0 \), \(G >0\) and \(n \geq 1\)
\begin{align}
    \mathbb{P}\left\{\left |\sum_{i=1}^n X_i\right| >b  \;\; and \;\; \sum_{i=1}^n \sigma^2_i \leq G\right\} \leq 2\exp\left(-\frac{b^2}{2G +\nicefrac{2bc}{3}}\right).
\end{align}
\end{lemma}
\vspace{0.2cm}
\begin{lemma}\label{lem: subGaussian_norm_concentration}
     (Corollary of Theorem 2.1, item (ii) from \citep{juditsky2008large}) Let \(\left\{\xi_k\right\}_{k=1}^N\) be a sequence of random
vectors in \(\mathbb{R}^n\) such that 
\[\mathbb{E}\left[\xi_k | \mathcal{F}_{k-1}\right]=0 \;\;\text{almost surely}, \;\;\;k=1,...,N~.\]
Define \(S_N := \sum_{k=1}^N \xi_k\). Assume that the sequence \(\left\{\xi_k\right\}_{k=1}^N\) satisfies the following light-tail condition
\begin{align}
\mathbb{E}\left[\exp\left(\frac{\norm{\xi_k}^2}{\sigma_k^2}\right) \mid \mathcal{F}_{k-1}\right] \leq \exp(1)\;\; \text{almost surely}, \;\;\;k=1,...,N
\end{align}
where \(\sigma_1, ..., \sigma_N\) are some positive numbers. Then for all \(\phi \geq 0 \), we have
\begin{align}
\mathbb{P}\left\{\norm{S_N}_2 \geq (\sqrt{2}+\sqrt{2}\phi)\sqrt{\sum_{k=1}^N \sigma_k^2}\right\} \leq \exp\left(-\frac{\phi^2}{3}\right).
\end{align}
\end{lemma}
\begin{lemma}[Lemma 1 from \citep{laurent2000adaptive}] \label{lem: Chi-square_concentration}
Let $\left\{Y_i\right\}_{i=1}^n$ be i.i.d. Gaussian variables, with mean 0 and variance 1. Let $\left\{a_i\right\}_{i=1}^n$ be nonnegative constants. Define
\[
\| a\|_{\infty}= \sup_{i=1, \dots n} \left | a_i \right |, \quad \|a\|_2^2 = \sum_{i=1}^n a_i^2.
\]
Let 
\[
X=\sum_{i=1}^n a_i\left(Y_i^2-1\right).
\]
Then the following inequalities hold for any positive t:
\begin{align}
    \mathbb{P}\left\{X \geq 2\|a\|_2\sqrt{t}+2\|a\|_{\infty}t\right\} \leq \exp(-t), \\
    \mathbb{P}\left\{X \leq -2\|a\|_2 \sqrt{t}\right\} \leq \exp(-t).
\end{align}
\end{lemma}
\begin{lemma}[Remark 2.8 from \citep{zhivotovskiy2024dimension}; see also example 4.3 from \citep{polyanskiy2025information}] \label{Lem. Norm_sub}
    Let X be a zero-mean sub-Gaussian random vector in $\mathbb{R}^d$ with covariance matrix $\Sigma$. Then the norm of this vector can be bounded in probability as below
    \begin{align}
        \mathbb{P} \left\{\norm{X}_2 > \sqrt{{\rm tr}(\Sigma)} + \sqrt{2\|\Sigma\|_{2}\ln{\frac{1}{\delta}}}\right\} \leq \delta.
    \end{align}
\end{lemma}

\clearpage

\section{Bound for the Bias and Variance of Clipped Estimator}
\label{lemmas}
\begin{lemma}\label{lem:Bias-Variance}
     Let $X$ be a random vector from $\mathbb{R}^d$. We define the random vector $\hat{X} := {\rm clip} \left(X, \lambda\right)$ for an arbitrary clipping level $\lambda > 0$. Let us assume
    \begin{align*}
        \mathbb{E}[X] = x, \qquad \mathbb{E}[\norm{X - x}^\alpha] \leq \sigma^\alpha,
    \end{align*}
    where $\sigma > 0$ is bounded, $\alpha \in (1, 2]$, and we also define $\hat{x} := {\rm clip}(x, \nicefrac{\lambda}{2})$. Then, the following inequalities hold:
    \begin{align}
        \norm{\mathbb{E}[\hat{X}] - \hat{x}} &\leq \frac{2^{2\alpha - 1}\sigma\left(\sigma^\alpha + (\max\{0, \norm{x} - \nicefrac{\lambda}{2}\})^\alpha\right)^{\frac{\alpha - 1}{\alpha}}}{\lambda^{\alpha - 1}} \notag \\&~~+ \max\{\norm{x}, \nicefrac{\lambda}{2}\}\frac{2^{2\alpha - 1}\left(\sigma^\alpha + (\max\left\{0, \norm{x} - \nicefrac{\lambda}{2}\right\})^\alpha\right)}{\lambda^{\alpha}} \notag \\&~~+ \max\{0, \norm{x} - \nicefrac{\lambda}{2}\}, \label{Lem. Bias part} \\ \nonumber \\
        \mathbb{E}\norm{\hat{X} - \mathbb{E}\hat{X}}^2 &\leq \frac{9(2^{2\alpha-1} + 1)\lambda^{2 - \alpha}\sigma^{\alpha}}{4} + \frac{9(2^{2\alpha-1}+1)\lambda^{2 - \alpha}(\max\{0, \norm{x} - \nicefrac{\lambda}{2}\})^{\alpha}}{4} \label{Lem. Variance part}.
    \end{align}
\end{lemma}
\begin{proof} 

The proof technique is similar to the proof of Lemma 5.1 from \citep{sadiev2023high}. Define random variables $\chi$ and $\eta$ as
    \begin{align*}
        \chi = \mathbb{I}_{\{\norm{X} > \lambda\}}, \qquad
        \eta = \mathbb{I}_{\{\norm{X - \hat{x}} > \nicefrac{\lambda}{2}\}}.
    \end{align*}
     Since $\norm{X} \leq \norm{\hat{x}} + \norm{X - \hat{x}} 
    \leq \frac{\lambda}{2} + \norm{X - \hat{x}}$, we get $\chi \leq \eta$. Moreover, note that
    \begin{align*}
        \hat{X} = \min\left\{1, \frac{\lambda}{\norm{X}}\right\}X = \chi\frac{\lambda}{\norm{X}}X + (1 - \chi)X. 
    \end{align*}
    \textbf{Proof of \eqref{Lem. Bias part}.} For the bias term, we obtain
    \begin{align*}
        \norm{\mathbb{E}\hat{X} - \hat{x}} &= \norm{\mathbb{E}\left(X + \chi\left(\frac{\lambda}{\norm{X}} - 1\right)X - \min\left\{1, \frac{\lambda}{2\norm{x}}\right\}x\right)} \\&\leq \norm{\mathbb{E}[\chi\left(\frac{\lambda}{\norm{X}} - 1\right)X]} + \left(1 - \min\left\{1, \frac{\lambda}{2\norm{x}}\right\}\right)\norm{x} \\&= \norm{\mathbb{E}[\chi\left(\frac{\lambda}{\norm{X}} - 1\right)X]} + \max\left\{0, \norm{x} - \frac{\lambda}{2}\right\} \\ 
        &\leq \mathbb{E}\left[\left|\chi\left(\frac{\lambda}{\norm{X}} - 1\right)\right|\norm{X}\right] + \max\left\{0, \norm{x} - \frac{\lambda}{2}\right\} \\ &\overset{(i)}{\leq} \mathbb{E}\left[\chi\norm{X}\right] + \max\left\{0, \norm{x} - \frac{\lambda}{2}\right\},
    \end{align*}
    where in $(i)$, we used the fact that $\chi \in \{0,1\}$ and when $\chi = 1$ we have $\left| \frac{\lambda}{\norm{X}} - 1 \right| = 1 - \frac{\lambda}{\norm{X}} \leq 1$. Then, we continue the derivation as follows:
    \begin{align}
        \norm{\mathbb{E}\hat{X} - \hat{x}} &\leq \mathbb{E}\left[\chi\norm{X}\right] + \max\left\{0, \norm{x} - \frac{\lambda}{2}\right\} \notag \\ 
        &\overset{\chi \leq \eta}{\leq} \mathbb{E}\left[\eta\norm{X}\right] + \max\left\{0, \norm{x} - \frac{\lambda}{2}\right\} \notag \\
        &\leq \mathbb{E}\left[\eta\norm{X - x}\right] + \mathbb{E}\left[\eta\norm{x}\right]+ \max\left\{0, \norm{x} - \frac{\lambda}{2}\right\} \notag\\
        &\overset{(i)}{\leq}\left(\mathbb{E}\norm{X - x}^\alpha\right)^{\nicefrac{1}{\alpha}} \left(\mathbb{E}[\eta^{\nicefrac{\alpha}{\alpha - 1}}]\right)^{\nicefrac{(\alpha - 1)}{\alpha}} + \mathbb{E}\eta \norm{x} + \max\{0, \norm{x} - \nicefrac{\lambda}{2}\}, \label{eq:hjvdhvjdbhvbdj}
    \end{align}
    where in $(i)$, we used H\"older inequality. Moreover, due to Markov's inequality, we also have
    \begin{align}
    \label{eq: mark}
        \mathbb{E}[\eta^{\nicefrac{\alpha}{\alpha - 1}}] = \mathbb{E}\eta = \mathbb{P}\{\norm{X - \hat{x}} > \nicefrac{\lambda}{2}\} =  \mathbb{P}\left\{\norm{X - \hat{x}}^\alpha > \left(\nicefrac{\lambda}{2}\right)^\alpha \right\} \leq \frac{2^\alpha\mathbb{E}{\norm{X - \hat{x}}^\alpha}}{\lambda^\alpha}.
    \end{align}
    Then, the expected value from the right-hand side (RHS) of \eqref{eq: mark} can be decomposed as follows 
    \begin{align}
    \label{eq: bias-mark}
        \mathbb{E}{\norm{X - \hat{x}}^\alpha} &= \mathbb{E}\norm{X - x + x - \hat{x}}^\alpha \leq 2^{\alpha-1}(\mathbb{E}{\norm{X - x}^\alpha} + \max\{0, \norm{x} - \nicefrac{\lambda}{2}\}^\alpha) \nonumber\\ &\leq 2^{\alpha-1}(\sigma^\alpha + \max\{0, \norm{x} - \nicefrac{\lambda}{2}\}^\alpha),
    \end{align}
    where we use the Jensen's inequality for the convex function $\norm{x}^\alpha$. After substitution of \eqref{eq: bias-mark} into \eqref{eq: mark}, we get
    \begin{align}
        \mathbb{E}[\eta^{\nicefrac{\alpha}{\alpha - 1}}] = \mathbb{E}\eta \leq \frac{2^{2\alpha - 1}(\sigma^\alpha + \max\{0, \norm{x} - \nicefrac{\lambda}{2}\}^\alpha)}{\lambda^\alpha}. \label{eq:bhdjjsdbsdjkskndf}
    \end{align}
    Plugging the above bound in \eqref{eq:hjvdhvjdbhvbdj}, we derive
    \begin{align*}
        \norm{\mathbb{E}\hat{X} - \hat{x}} &\leq \sigma\left(\frac{2^{2\alpha - 1}(\sigma^\alpha + \max\{0, \norm{x} - \nicefrac{\lambda}{2}\}^\alpha)}{\lambda^\alpha}\right)^{\frac{\alpha-1}{\alpha}} + \norm{x}\frac{2^{2\alpha - 1}(\sigma^\alpha + \max\{0, \norm{x} - \nicefrac{\lambda}{2}\}^\alpha)}{\lambda^\alpha} \\&+ \max\{0, \norm{x} - \nicefrac{\lambda}{2}\}.
    \end{align*}
    Using that $\frac{\alpha-1}{\alpha} \leq 1$ and $\norm{x} \leq \max\{\norm{x}, \nicefrac{\lambda}{2}\}$, we conclude the proof of the result for the bias term, i.e., bound \eqref{Lem. Bias part}.
    
    \textbf{Proof of \eqref{Lem. Variance part}.}
    First, we use the following standard inequality:
    \begin{align*}
        \mathbb{E}\norm{\hat{X} - \mathbb{E}\hat{X}}^2 \leq \mathbb{E}\norm{\hat{X} - \hat{x}}^2.
    \end{align*}
    Then, we bound the RHS as
    \begin{align*}
        \mathbb{E}\norm{\hat{X} - \hat{x}}^2 &= \mathbb{E}\left[\left(\norm{\hat{X} - \hat{x}}^{2 - \alpha}\right)\left(\norm{\hat{X} - \hat{x}}^\alpha\right)\right] \\&\leq \left(\frac{3\lambda}{2}\right)^{2-\alpha}\left(\mathbb{E}\norm{\hat{X} - \hat{x}}^\alpha\right)\\&= \left(\frac{3\lambda}{2}\right)^{2-\alpha}\left(\mathbb{E}\left[\chi\norm{\frac{\lambda}{\norm{X}}X - \hat{x}}^\alpha + (1 - \chi)\norm{X - \hat{x}}^\alpha\right]\right)\\&\leq \left(\frac{3\lambda}{2}\right)^{2}\EE\chi + \left(\frac{3\lambda}{2}\right)^{2-\alpha}\EE{\norm{X - \hat{x}}^\alpha}\\&\leq\left(\frac{3\lambda}{2}\right)^{2}\EE\eta + \left(\frac{3\lambda}{2}\right)^{2-\alpha}\EE{\norm{X - \hat{x}}^\alpha}.
    \end{align*}
    Applying upper bounds \eqref{eq: bias-mark} and \eqref{eq:bhdjjsdbsdjkskndf} from the previous part of the proof, we obtain
    \begin{align*}
        \mathbb{E}\norm{\hat{X} - \hat{x}}^2 &\leq \left(\frac{3\lambda}{2}\right)^{2} \frac{2^{2\alpha-1}(\sigma^\alpha + \max\{0, \norm{x} - \nicefrac{\lambda}{2}\}^\alpha)}{\lambda^\alpha} \\&+ \left(\frac{3\lambda}{2}\right)^{2-\alpha}2^{\alpha-1}(\sigma^\alpha + \max\{0, \norm{x} - \nicefrac{\lambda}{2}\}^\alpha) \\&=\frac{9\cdot (2^{2\alpha-1} + 1)\lambda^{2-\alpha}\sigma^\alpha}{4} + \frac{9\cdot (2^{2\alpha-1} + 1)\lambda^{2-\alpha}(\max\{0, \norm{x} - \nicefrac{\lambda}{2}\})^\alpha}{4},
    \end{align*}
    which concludes the proof.
\end{proof}

\clearpage

\section{Missing Proofs: Convex Case}


We start the analysis with the following lemma. This lemma follows the proof of deterministic \algname{GD} and separates the stochastic part from the deterministic part of \algname{Clipped-SGD}.
\begin{lemma}\label{lem: Convex_descent_lemma}
     Let Assumptions \ref{ass:bounded_below}, \ref{ass:smoothness}, and  \ref{ass:convexity}, and   hold for $Q = B_{2R}(x^\star)$, where $R \geq \|x^0 - x^\star\|$ and $0 < \gamma \leq \nicefrac{1}{8L}$. If $x^k \in Q$  for all $k = 0,1,\ldots, K$ for some $K\geq 0$, then for any $0\leq T\leq K$ the iterates produced by \algname{DP-Clipped-SGD} satisfy
    \begin{align*}
        \frac{\gamma}{T+1} \sum_{t=0}^Tc_t (f(x^t) -f^\star) &\leq \frac{\norm{x^0-x^\star}^2-\norm{x^{T+1}-x^\star}^2}{T+1} - \frac{2\gamma}{T+1} \sum_{t=0}^T\langle x^t-x^\star , \theta_t \rangle \notag\\ 
        & \quad -\frac{2\gamma}{T+1} \sum_{t=0}^T\langle x^t-x^\star , \omega_t \rangle + \frac{2\gamma^2}{T+1} \sum_{t=0}^T\norm{\theta_t}^2  \notag \\
        & \quad + \frac{4\gamma^2}{T+1}\sum_{t=0}^T\norm{\omega_t}^2,
    \end{align*}
    where we have defined
    \begin{align}
     c_t &\eqdef \min \left\{1,\frac{\lambda}{2\norm{\nabla f(x^t)}}\right\} ,\label{c_t_def}& \\
        \theta_t &\eqdef \hat g_t -c_t \nabla f(x^t).& \label{theta_def}
    \end{align}
\end{lemma}
\begin{proof}
    Since $x^{t+1}=x^t-\gamma \tilde g_t$, the following set of inequalities hold for all $t=0,1, \dots, K$:
    \begin{align*}
    \norm{x^{t+1}-x^\star}^2 &= \norm{x^t-x^\star}^2 -2\gamma \langle x^t-x^\star , \tilde g_t\rangle + \gamma^2 \norm{\tilde g_t}^2 \\
    &= \norm{x^t-x^\star}^2 -2\gamma \langle x^t-x^\star , \hat g_t+\omega_t\rangle + \gamma^2 \norm{\hat g_t + \omega_t}^2  \\
    &=  \norm{x^t-x^\star}^2 -2\gamma \langle x^t-x^\star , \hat g_t+\omega_t + c_t \nabla f(x^t) -c_t \nabla f(x^t)\rangle\\
    & \;\; +\gamma^2 \norm{\hat g_t + \omega_t+ c_t \nabla f(x^t) -c_t \nabla f(x^t)}^2 \\
    &\leq  \norm{x^t-x^\star}^2 -2\gamma \langle x^t-x^\star , \theta_t+\omega_t \rangle -2\gamma c_t \langle x^t-x^\star , \nabla f(x^t) \rangle + 2\gamma^2\norm{\theta_t}^2 \\
    & \;\; +4\gamma^2\norm{\omega_t}^2 +4\gamma^2c_t^2\norm{\nabla f(x^t)}^2  \\
    &\leq  \norm{x^t-x^\star}^2 -2\gamma \langle x^t-x^\star , \theta_t+\omega_t \rangle -2\gamma c_t (f(x^t)-f^\star) + 2\gamma^2\norm{\theta_t}^2 \\
    & \;\;\; +4\gamma^2\norm{\omega_t}^2 +8\gamma^2c_t^2L(f(x^t)-f^\star) \\
    &=  \norm{x^t-x^\star}^2 -2\gamma \langle x^t-x^\star , \theta_t+\omega_t \rangle -(2\gamma-8\gamma^2L)c_t (f(x^t)-f^\star) + 2\gamma^2\norm{\theta_t}^2  +4\gamma^2\norm{\omega_t}^2 . 
\end{align*}
First, we rearrange the terms, and utilize the inequalities $\gamma \leq \nicefrac{1}{8L}$ and $c_t^2 \leq c_t$. Upon summing over $t=0,1,\ldots,T$, we obtain the following inequality
\begin{eqnarray}
         \frac{\gamma}{T+1} \sum_{t=0}^Tc_t (f(x^t) -f^\star) &\leq&\frac{\norm{x^0-x^\star}^2-\norm{x^{T+1}-x^\star}^2}{T+1} - \frac{2\gamma}{T+1} \sum_{t=0}^T\langle x^t-x^\star , \theta_t \rangle \notag\\ 
        && - \frac{2\gamma}{T+1} \sum_{t=0}^T\langle x^t-x^\star , \omega_t \rangle + \frac{2\gamma^2}{T+1} \sum_{t=0}^T\norm{\theta_t}^2 + \frac{4\gamma^2}{T+1}\sum_{t=0}^T\norm{\omega_t}^2, \notag
\end{eqnarray}
which concludes the proof.
\end{proof}

Using this lemma, we prove the main convergence result for \algname{DP-Clipped-SGD} in the convex case.
\begin{theorem} \label{main_thm_convex}
Let Assumptions \ref{ass:bounded_below}, \ref{ass:smoothness},   \ref{ass:convexity}, and \ref{ass:oracle}  hold for $Q = B_{2R}(x^\star)$, where $R$ is such that $R \geq \|x^0 - x^\star\|$. Let $\zeta_\lambda \eqdef \max\{0,2LR-\frac{\lambda}{2}\}$, and $\gamma \leq \min\{\nicefrac{1}{8L},\gamma_1, \gamma_2, \gamma_3, \gamma_4, \gamma_5, \gamma_6\}$, where
     \begin{eqnarray}
          \gamma_1 &\eqdef& \frac{R}{42(2^{2\alpha-1}+1)^{1/2}\sigma^{\alpha/2}\lambda^{1-\alpha/2}\sqrt{6(K+1)\ln\frac{8(K+1)}{\beta}{\left(1+ {\color{black}{\frac{\zeta_\lambda^\alpha}{\sigma^\alpha}}}\right)}}}, \label{convex_first_step_size_condition}\\ 
        \gamma_2 &\eqdef& \frac{R\lambda^{\alpha-1}}{28(K+1)2^{2\alpha-1}\sigma^\alpha\left(1+\frac{\zeta_\lambda^\alpha}{\sigma^\alpha} \right)\left(\frac{\zeta_\lambda}{\lambda}+\frac{1}{2}+\frac{\lambda^{\alpha-1}\zeta_\lambda}{2^{2\alpha-1}\left(\sigma^\alpha+\zeta_\lambda^\alpha\right)}+\left(1+\frac{\zeta_\lambda^\alpha}{\sigma^\alpha}\right)^{-1/\alpha}\right)}, \label{convex_second_step_size_condition}\\
        \gamma_3 &\eqdef& \frac{R}{56\sigma_\omega\sqrt{d(K+1)}(\sqrt{2}+\sqrt{2}\phi)}, \label{convex_third_step_size_condition}\\
        \gamma_4&\eqdef& \frac{(2-\sqrt{2})R} {\lambda + \sigma_\omega  \left(\sqrt{d}+ \sqrt{2\ln\left(\frac{K+1}{\beta}\right)}\right) }, \label{convex_fourth_step_size_condition}\\
        \gamma_5&\eqdef&\frac{R}{56\lambda\ln\frac{8(K+1)}{\beta}},\label{convex_fifth_step_size_condition}\\
        \gamma_6&\eqdef&\frac{R}{2\sigma_w\sqrt{7\left[(K+1)d + 2\sqrt{(K+1)d\ln\frac{4(K+1)}{\beta}}+2\ln\frac{4(K+1)}{\beta}\right]}}.\label{convex_sixth_step_size_condition}
     \end{eqnarray}
     with $\phi \eqdef \sqrt{3\ln{\frac{4(K+1)}{\beta}}}$ for some $K > 0$ and $\beta \in  (0, 1]$. Then, after $K$ iterations of \algname{DP-Clipped-SGD}, the iterates with probability at least $1-\beta$ satisfy
     \begin{eqnarray}
                \min_{k\in [0,K]} f(x^k)-f(x^\star) \leq \frac{4R^2}{\gamma(K+1)}+ \frac{64LR^4}{\lambda^2\gamma^2 (K+1)^2} \quad and \quad \{x^k\}_{k=0}^K \subseteq B_{\sqrt{2}R}(x^\star). 
     \end{eqnarray}
\end{theorem}
\begin{proof}
    Let $R_k \eqdef \|x^k - x^\star\|$ for all $k\geq 0$. Next, our goal is to show by induction that $R_{k} \leq 2R$ for all $k = 0,1,\ldots, K$ with high probability, which allows us to apply the result of Lemma~\ref{lem: Convex_descent_lemma} and then use Bernstein's inequality to estimate the stochastic part of the upper-bound. More precisely, for each $k = 0,\ldots, K+1$ we consider probability event $E_k$ defined as follows: inequalities
    \begin{eqnarray}
        &- 2\gamma\sum\limits_{l=0}^{t-1} \langle x^l - x^\star , \theta_l \rangle  - 2\gamma\sum\limits_{l=0}^{t-1} \langle x^l - x^\star , \omega_l \rangle + 2\gamma^2\sum\limits_{l=0}^{t-1} \|\theta_l\|^2 + 4\gamma^2\sum\limits_{l=0}^{t-1} \|\omega_l\|^2\leq  R^2,& \label{eq:clipped_SGD_convex_induction_inequality_1}\\
        &R_t \leq \sqrt{2}R,& \label{eq:clipped_SGD_convex_induction_inequality_2} \\
        &\norm{\omega_t} \leq \sigma_\omega \left(\sqrt{d}+\sqrt{2\ln \left (\frac{K+1}{(t+1)\beta}\right)}\right),& \label{eq. clipped_SGD_Convex_thirs_induction_inequality}
    \end{eqnarray}
    hold for all $t = 0,1,\ldots, k$ simultaneously. We want to prove via induction that $\PP\{E_k\} \geq 1 - \nicefrac{(k+1)\beta}{(K+1)}$ for all $k = 0,1,\ldots, K$. For $k = 0$ the statements \eqref{eq:clipped_SGD_convex_induction_inequality_1} and \eqref{eq:clipped_SGD_convex_induction_inequality_2} trivially hold. Given Lemma \ref{Lem. Norm_sub}, statement \eqref{eq. clipped_SGD_Convex_thirs_induction_inequality} will also hold. Assume that the statement is true for some $k = T - 1 \leq K$: $\PP\{E_{T-1}\} \geq 1 - \nicefrac{T\beta}{(K+1)}$. One needs to prove that $\PP\{E_{T}\} \geq 1 - \nicefrac{(T+1)\beta}{(K+1)}$. First, we notice that probability event $E_{T-1}$ implies that $x_t \in B_{\sqrt{2}R}(x^\star)$ for all $t = 0,1,\ldots, T-1$. For $x^T$, we can obtain the following inequalities
    \begin{eqnarray}
        \|x^T - x^\star\| &=& \|x^{T-1} - x^\star - \gamma \tg_{T-1}\| \leq \|x^{T-1} - x^\star\| + \gamma\|\hat g_{T-1}\|+\gamma\|\omega_{T-1}\| \notag\\
        &\leq& \sqrt{2}R + \gamma \lambda + \gamma \sigma_\omega  \left(\sqrt{d}+ \sqrt{2\ln\left(\frac{K+1}{T\beta}\right)}\right) \overset{\eqref{convex_fourth_step_size_condition}}{\leq} 2R.\label{bounded_iterates}
    \end{eqnarray}
    This means that $x^0, x^1, \ldots, x^T \in B_{2R}(x^\star)$. Therefore, $E_{T-1}$ implies $\{x^k\}_{k=0}^{T} \subseteq Q$, meaning that the assumptions of Lemma~\ref{lem: Convex_descent_lemma} are satisfied. Subsequently, the following inequality holds 
    \begin{eqnarray}
        \frac{\gamma}{t} \sum_{l=0}^{t-1} c_l\left(f(x^l) -f(x^\star)\right) &\leq& \frac{\|x^0 - x^\star\|^2 - \|x^{t} - x^\star\|^2}{t}+\frac{4\gamma^2}{t}\sum_{l=0}^{t-1}\norm{\omega_l}^2\notag\\
       &&\quad - \frac{2\gamma}{t}\sum\limits_{l=0}^{t-1} \langle x^l - x^\star , \theta_l + \omega_l \rangle  + \frac{2\gamma^2}{t}\sum\limits_{l=0}^{t-1} \|\theta_l\|^2 ,\label{eq:clipped_SGD_convex_technical_1}
    \end{eqnarray}
    for all $t=1,\ldots, T$ simultaneously. For all $t = 1, \ldots, T-1$ this event also implies
    \begin{eqnarray}
       \gamma  \sum _{l=0}^{t-1} c_l (f(x^l) -f(x^\star)) &\leq&R^2 - 2\gamma\sum\limits_{l=0}^{t-1} \langle x^l - x^\star , \theta_l \rangle - 2\gamma\sum\limits_{l=0}^{t-1} \langle x^l - x^\star , \omega_l \rangle  + 2\gamma^2\sum\limits_{l=0}^{t-1} \|\theta_l\|^2 \notag \\ 
        &+&4\gamma^2\sum\limits_{l=0}^{t-1} \|\omega_l\|^2 \notag\\
        &\leq& {2R^2},
\label{eq:clipped_SGD_convex_technical_1_1}
    \end{eqnarray}
where we have used \eqref{eq:clipped_SGD_convex_induction_inequality_1} for $E_{T-1}$.     
    Taking into account that $\sum _{l=0}^{t-1} c_l(f(x^l) -f(x^\star)) \geq 0$,  \eqref{eq:clipped_SGD_convex_technical_1} implies
    \begin{eqnarray}
        R_T^2 \leq R^2 - 2\gamma\sum\limits_{t=0}^{T-1} \langle x^t - x^\star, \theta_t \rangle - 2\gamma\sum\limits_{t=0}^{T-1} \langle x^t - x^\star, \omega_t \rangle  + 2\gamma^2\sum\limits_{t=0}^{T-1} \|\theta_t\|^2 +4\gamma^2\sum\limits_{t=0}^{T-1} \|\omega_t\|^2.
        \label{eq:clipped_SGD_convex_technical_2}
    \end{eqnarray}
    Next, we define random vectors
    \begin{equation}
        \eta_t \eqdef \begin{cases} x^t - x^\star, & \text{if } \|x^t - x^\star\| \leq 2R,\\ 0,&\text{otherwise}, \end{cases} \notag
    \end{equation}
    for all $t = 0,1,\ldots, T-1$. By definition, these random vectors are bounded with probability $1$
    \begin{equation}
        \|\eta_t\| \leq 2R. \label{eq:convex_eta_bound}
    \end{equation}
    Next, we introduce the following vectors
    \begin{align}
        \theta_t^u \eqdef \hat g_t - \mathbb{E}\left[\hat g_t \mid \mathcal{F}^{t-1}\right], \quad \theta_t^b \eqdef \mathbb{E}\left[\hat g_t \mid \mathcal{F}^{t-1}\right] - c_t \nabla f(x^t) \label{Biased-Unbiased decomposition}
    \end{align}
    Using the above notation, we notice that $\theta_{t} = \theta_{t}^u + \theta_{t}^b$. Subsequently, $E_{T-1}$ implies
    \begin{eqnarray}
        R_T^2 
        &\leq& R^2 \underbrace{-2\gamma\sum\limits_{t=0}^{T-1}\langle \theta_t^u, \eta_t\rangle}_{\circledOne}  \underbrace{-2\gamma\sum\limits_{t=0}^{T-1}\langle \theta_t^b, \eta_t\rangle}_{\circledTwo}  
        \underbrace{-2\gamma\sum\limits_{t=0}^{T-1}\langle \omega_l, \eta_t\rangle}_{\circledThree} +\underbrace{4\gamma^2\sum\limits_{t=0}^{T-1}\EE\left[\left\|\theta_{t}^u\right\|^2 \mid \mathcal{F}^{t-1}\right]}_{\circledFour}
        \notag\\
        &&+ \underbrace{4\gamma^2\sum\limits_{t=0}^{T-1}\left(\left\|\theta_{t}^u\right\|^2 - \EE\left[\left\|\theta_{t}^u\right\|^2\mid \mathcal{F}^{t-1}\right]\right)}_{\circledFive} +   \underbrace{4\gamma^2\sum\limits_{t=0}^{T-1}\left\|\theta_{t}^b\right\|^2}_{\circledSix} +
        \underbrace{4\gamma^2\sum\limits_{t=0}^{T-1}\left\|\omega_t \right\|^2}_{\circledSeven}.\label{eq:clipped_SGD_convex_technical_7}
    \end{eqnarray}

   To finish our inductive proof we need to show that $\circledOne + \circledTwo + \circledThree + \circledFour + \circledFive + \circledSix+ \circledSeven \leq R^2$ with high probability. In the subsequent parts of the proof, we will utilize the bounds for the norm and norm squared moments of $\theta_{t}^u$ and $\theta_{t}^b$. First, by definition of clipping operator and Lemma \ref{lem:Bias-Variance} we have
     \begin{equation}
        \|\theta_{t}^u\| \leq 2\lambda, \label{eq:clipped_SGD_convex_norm_theta_u_bound}
    \end{equation}
    and
    \begin{eqnarray}
        \|\theta_{t}^b\| &&\leq \frac{2^{2\alpha - 1}\sigma\left(\sigma^\alpha + (\max\{0, \norm{\nabla f(x^t)} - \nicefrac{\lambda}{2}\})^\alpha\right)^{\frac{\alpha - 1}{\alpha}}}{\lambda^{\alpha - 1}} \notag \\ 
        &&+ \max\{\norm{\nabla f(x^t)}, \nicefrac{\lambda}{2}\}\frac{2^{2\alpha - 1}\left(\sigma^\alpha + (\max\{0, \norm{\nabla f(x^t)} - \nicefrac{\lambda}{2}\})^\alpha\right)}{\lambda^{\alpha}}\notag \\&&+ \max\{0, \norm{\nabla f(x^t)} - \nicefrac{\lambda}{2}\}, 
    \end{eqnarray}
    \begin{eqnarray}
        \EE\left[\norm{\theta_t^u}^2 \mid \mathcal{F}^{t-1}\right] \leq \frac{9(2^{2\alpha-1} + 1)\lambda^{2 - \alpha}\sigma^{\alpha}}{4} + \frac{9(2^{2\alpha-1}+1)\lambda^{2 - \alpha}(\max\{0, \norm{\nabla f(x^t)} - \nicefrac{\lambda}{2}\})^{\alpha}}{4}. 
    \end{eqnarray}
As can be seen, these bounds are iteration-dependent due to the presence of $\|\nabla f(x^t)\|$. As a remedy, we bound $\norm{\nabla f(x^t)}$ by $2LR$ inside event $E_{T-1}$. This bound can be obtained from a combination of Assumption \ref{ass:smoothness}, $E_{T-1}$, and \eqref{bounded_iterates}. Next, we introduce a new variable $\zeta_{\lambda}:=\max \{0,2LR-\frac{\lambda}{2}\}$. Thus, we get the following bounds for the bias and variance of $\theta_t$: $E_{T-1}$ implies
\begin{eqnarray}
        \|\theta_{t}^b\| &&\leq \frac{2^{2\alpha - 1}\sigma\left(\sigma^\alpha + \zeta_\lambda^\alpha\right)^{\frac{\alpha - 1}{\alpha}}}{\lambda^{\alpha - 1}} 
        + \left(\zeta_\lambda+ \frac{\lambda}{2}\right)\frac{2^{2\alpha - 1}\left(\sigma^\alpha + \zeta_\lambda^\alpha\right)}{\lambda^{\alpha}} + \zeta_\lambda ,\label{eq:convex_norm_theta_b_bound}
    \end{eqnarray}
    \begin{eqnarray}
        \EE\left[\norm{\theta_t^u}^2 \mid \mathcal{F}^{t-1}\right] \leq \frac{9(2^{2\alpha-1} + 1)\lambda^{2 - \alpha}\sigma^{\alpha}}{4} + \frac{9(2^{2\alpha-1}+1)\lambda^{2 - \alpha}\zeta_\lambda^{\alpha}}{4}\label{eq:convex_theta_u_norm_variance}
    \end{eqnarray}
    for $t = 0,1,\ldots, T-1$.

    \paragraph{Upper bound for $\circledOne$.} By definition of $\theta_{t}^u$, we have $\EE[\theta_{t}^u \mid \mathcal{F}^{t-1}] = 0$ and
    \begin{equation}
        \EE\left[-2\gamma\langle\theta_t^u, \eta_t\rangle \mid \mathcal{F}^{t-1}\right] = 0. \notag
    \end{equation}
    Furthermore, $\circledOne$ is bounded with probability $1$ as
    \begin{equation}
        |2\gamma\left\la \theta_{t}^u, \eta_t\right\ra| \leq 2\gamma \|\theta_{t}^u\| \cdot \|\eta_t\| \overset{\eqref{eq:clipped_SGD_convex_norm_theta_u_bound},\eqref{eq:convex_eta_bound}}{\leq} 8\gamma \lambda R \overset{\eqref{convex_fifth_step_size_condition}}{\leq} \frac{R^2}{7\ln\frac{8(K+1)}{\beta}} \eqdef c.\label{convex_first_martingale_bound}
    \end{equation}
    The summands also have bounded conditional variances $\sigma_t^2 \eqdef \EE[4\gamma^2\langle\theta_t^u, \eta_t\rangle^2 \mid \mathcal{F}^{t-1}]$ as
    \begin{equation}
        \sigma_t^2 \leq \EE\left[4\gamma^2\|\theta_{t}^u\|^2\cdot \|\eta_t\|^2 \mid \mathcal{F}^{t-1}\right] \overset{\eqref{eq:convex_eta_bound}}{\leq} 16\gamma^2 R^2 \EE\left[\|\theta_{t}^u\|^2 \mid \mathcal{F}^{t-1}\right]. \label{eq:clipped_SGD_convex_technical_9}
    \end{equation}
    In other words, we showed that $\{-2\gamma\left\la \theta_{t}^u, \eta_t\right\ra\}_{t=0}^{T-1}$ is a bounded martingale difference sequence with bounded conditional variances $\{\sigma_t^2\}_{t=0}^{T-1}$. Next, we apply Bernstein's inequality (Lemma~\ref{lem: Bernstein_inequality}) with $X_t = -2\gamma \left\la \theta_{t}^u, \eta_t\right\ra$, parameter $c$ as in \eqref{convex_first_martingale_bound}, $b = \frac{R^2}{7}$, $G = \frac{R^4}{294\ln\frac{8(K+1)}{\beta}}$ to obtain
    \begin{equation*}
        \PP\left\{|\circledOne| > \frac{R^2}{7}\quad \text{and}\quad \sum\limits_{t=0}^{T-1} \sigma_{t}^2 \leq \frac{R^4}{294\ln\frac{8(K+1)}{\beta}}\right\} \leq 2\exp\left(- \frac{b^2}{2G + \nicefrac{2cb}{3}}\right) = \frac{\beta}{4(K+1)}.
    \end{equation*}
    Equivalently, we have
    \begin{equation}
        \PP\left\{ E_{\circledOne} \right\} \geq 1 - \frac{\beta}{4(K+1)},\quad \text{for}\quad E_{\circledOne} = \left\{ \text{either} \quad  \sum\limits_{t=0}^{T-1} \sigma_{t}^2 > \frac{R^4}{294\ln\frac{8(K+1)}{\beta}} \quad \text{or}\quad |\circledOne| \leq \frac{R^2}{7}\right\}. \label{eq:clipped_SGD_convex_sum_1_upper_bound}
    \end{equation}
    In addition, $E_{T-1}$ implies
    \begin{eqnarray}
        \sum\limits_{t=0}^{T-1} \sigma_{t}^2 &{\leq}& 16\gamma^2 R^2 \sum\limits_{t=0}^{T-1}  \EE\left[\|\theta_{t}^u\|^2 \mid \mathcal{F}^{t-1}\right] \notag\\
        &\overset{\eqref{eq:convex_theta_u_norm_variance}}{\leq}& 4R^2\gamma^2 {T}\left(9(2^{2\alpha-1}+1)\lambda^{2-\alpha}\sigma^\alpha + 9(2^{2\alpha-1}+1)\lambda^{2-\alpha}\zeta_\lambda^\alpha
        )\right) \notag \\
        &\overset{\eqref{convex_first_step_size_condition}}{\leq}& \frac{R^4}{294\ln\frac{8(K+1)}{\beta}}.
    \end{eqnarray}
    \paragraph{Upper bound for $\circledTwo$.} From $E_{T-1}$ it follows that
\begin{eqnarray}
        \circledTwo &=& -2\gamma\sum\limits_{t=0}^{T-1}\langle \theta_t^b, \eta_t \rangle \leq 2\gamma\sum\limits_{t=0}^{T-1}\|\theta_{t}^b\|\cdot \|\eta_t\| \notag \\
        &\overset{\eqref{eq:convex_norm_theta_b_bound}, \eqref{eq:convex_eta_bound}}{\leq}& 4\gamma RT \left(\frac{2^{2\alpha - 1}\sigma\left(\sigma^\alpha + \zeta_\lambda^\alpha\right)^{\frac{\alpha - 1}{\alpha}}}{\lambda^{\alpha - 1}} \notag 
        + (\zeta_\lambda+\nicefrac{\lambda}{2})\frac{2^{2\alpha - 1}\left(\sigma^\alpha + \zeta_\lambda^\alpha\right)}{\lambda^{\alpha}}\notag + \zeta_\lambda\right) \notag \\
        &\overset{T< K+1}{\leq}& 4\gamma R (K+1) \frac{2^{2\alpha-1}}{\lambda^{\alpha-1}} \left(\sigma^{\alpha} +\zeta_\lambda^\alpha\right)\left(\left(1+\frac{\zeta_\lambda^\alpha}{\sigma^\alpha}\right)^{-1/\alpha} + \frac{\zeta_\lambda}{\lambda}+\frac{1}{2}+\frac{\lambda^{\alpha-1}\zeta_\lambda}{2^{2\alpha-1}\left(\sigma^\alpha+\zeta_\lambda^\alpha\right)}\right) \notag\\
        &\overset{\eqref{convex_second_step_size_condition}}{\leq}& \frac{R^2}{7}.
\end{eqnarray}

\paragraph{Upper bound for $\circledThree$.} We have
\begin{eqnarray}
    \left|\circledThree\right| = \left |-2\gamma\sum_{t=0}^{T-1}\langle\eta_t, \omega_t \rangle \right | = \left |\sum_{t=0}^{T-1}\sum_{i=1}^d 2\gamma\eta_{t,i}\omega_{t,i}\right | 
\end{eqnarray}
where $\eta_{t,i} \eqdef [\eta_t]_i$ and $\omega_{t,i} \eqdef [\omega_t]_i$ denote the $i$-th components of $\eta_t$ and $\omega_t$ respectively. 

Each summand is the product of a zero-mean Gaussian random variable and a bounded random variable, resulting in the product being a zero-mean sub-Gaussian random variable with parameter  $\sigma_{t,i}^2=64R^2\gamma^2\sigma_\omega^2$. To prove this, consider
\begin{eqnarray}
   \EE \left[\exp \left(\frac{4\gamma^2}{\sigma_{t,i}^2}\left | \eta_{t,i}^2 \omega_{t,i}^2\right |\right) \mid \cF^{t-1}\right] &\overset{\eqref{eq:convex_eta_bound}}{\leq}& \EE \left[\exp \left(\frac{16 R^2 \gamma^2}{64\gamma^2 R^2\sigma^2_\omega} \left | \omega_{t,i}\right |^2\right)\right] \notag \\ 
   &\leq&   \EE \left[\exp \left (\frac{ \left | \omega_{t,i}\right |^2}{4\sigma^2_{\omega}}\right) \right]\overset{(ii)}{\leq} \exp(1)
\end{eqnarray}
where $(ii)$ uses the fact that $\omega_{t,i}^2$ is light-tailed random variable with parameter $\sigma_\omega^2$.
Now that we have established the light-tailedness of summands, we can use the Lemma \ref{lem: subGaussian_norm_concentration} to obtain
\begin{eqnarray}
    \PP\left\{\left |\sum_{t=0}^{T-1}\sum_{i=1}^d 2\gamma\eta_{t,i}\omega_{t,i}\right| > \left(\sqrt{2}+\sqrt{2}\phi \right)\sqrt{\sum_{t=0}^{K}\sum_{i=1}^d 64\gamma^2R^2\sigma_\omega^2}\right\} &\leq& \exp\left(\frac{-\phi^2}{3}\right) \notag\\
    &=&\frac{\beta}{4(K+1)}.
\end{eqnarray}
The choice of $\gamma \leq \gamma_3$ for $\gamma_3$ defined \eqref{convex_third_step_size_condition} implies 
\begin{equation*}
    \left(\sqrt{2}+\sqrt{2}\phi \right)\sqrt{\sum_{t=0}^{T-1}\sum_{i=1}^d 64\gamma^2R^2\sigma_\omega^2} \leq \left(\sqrt{2}+\sqrt{2}\phi \right)\sqrt{64\gamma^2R^2(K+1)d\sigma_\omega^2} \overset{\eqref{convex_third_step_size_condition}}{\leq} \frac{R^2}{7},
\end{equation*}
and
\begin{align}
    \mathbb{P}\{E_{\circledThree}\} \geq 1-\frac{\beta}{4(K+1)} \;\;\; \text{for} \;\;\; E_{\circledThree}=\left\{\left | \circledThree \right| \leq \frac{R^2}{7}\right\}.
\end{align}



    \paragraph{Upper bound for $\circledFour$.} From $E_{T-1}$, and conditions on the step-size it follows that
\begin{eqnarray}
        \circledFour &=& 4\gamma^2\sum\limits_{t=0}^{T-1}\EE\left[\left\|\theta_{t}^u\right\|^2\mid \mathcal{F}_{\xi}^{t-1}\right] \notag \\ 
        &\overset{\eqref{eq:convex_theta_u_norm_variance}}{\leq}&  4\gamma ^2 T \left(\frac{9(2^{2\alpha-1} + 1)\lambda^{2 - \alpha}\sigma^{\alpha}}{4} + \frac{9(2^{2\alpha-1}+1)\lambda^{2 - \alpha}\zeta_\lambda^{\alpha}}{4}\right) \overset{\eqref{convex_first_step_size_condition}}{\leq} \frac{R^2}{7}.
\end{eqnarray}

\paragraph{Upper bound for $\circledFive$.} First, we have
    \begin{equation}
        \EE\left[4\gamma^2\left(\left\|\theta_{t}^u\right\|^2 - \EE\left[\left\|\theta_{t}^u\right\|^2\mid \mathcal{F}^{t-1}\right]\right)\mid \mathcal{F}^{t-1}\right] = 0. \notag
    \end{equation}
    Next, sum $\circledFive$ has bounded with probability $1$ terms:
    \begin{eqnarray}
        \left|4\gamma^2\left(\left\|\theta_{t}^u\right\|^2 - \EE\left[\left\|\theta_{t}^u\right\|^2\mid \mathcal{F}^{t-1}\right]\right)\right| &\leq& 4\gamma^2\left( \|\theta_{t}^u\|^2 +   \EE\left[\left\|\theta_{t}^u\right\|^2\mid \mathcal{F}^{t-1}\right]\right)\notag\\
        &{\leq}& 32\gamma^2\lambda^2 \overset{\eqref{convex_fifth_step_size_condition}}{\leq} \frac{R^2}{7 \ln\frac{8(K+1)}{\beta}}\eqdef c \label{convex_second_martingale_bound}. 
    \end{eqnarray}
   The summands also have bounded conditional variances
    \begin{align}
          \widetilde\sigma_t^2 \eqdef \EE\left[16\gamma^4\left(\left\|\theta_{t}^u\right\|^2 - \EE\left[\left\|\theta_{t}^u\right\|^2\mid \mathcal{F}^{t-1}\right]\right)^2\mid \mathcal{F}^{t-1}\right], \notag 
    \end{align}
    \begin{eqnarray}
        \widetilde\sigma_t^2 &\overset{\eqref{convex_second_martingale_bound}}{\leq}& \frac{R^2}{7 \ln\frac{8(K+1)}{\beta}} \EE\left[4\gamma^2\left|\left\|\theta_{t}^u\right\|^2 - \EE\left[\left\|\theta_{t}^u\right\|^2\mid \mathcal{F}^{t-1}\right]\right|\mid \mathcal{F}^{t-1}\right] \\ &\leq& \frac{8\gamma^2 R^2}{7\ln\frac{8(K+1)}{\beta}} \EE\left[\|\theta_{t}^u\|^2\mid \mathcal{F}^{t-1}\right]. \label{eq:clipped_SGD_convex_technical_11}
    \end{eqnarray}
    To summarize, we have shown that $\left\{4\gamma^2\left(\left\|\theta_{t}^u\right\|^2 - \EE\left[\left\|\theta_{t}^u\right\|^2\mid \mathcal{F}^{t-1}\right]\right)\right\}_{t=0}^{T-1}$ is a bounded martingale difference sequence with bounded conditional variances $\{\widetilde\sigma_t^2\}_{t=0}^{T-1}$. Next, we apply Bernstein's inequality (Lemma~\ref{lem: Bernstein_inequality}) with $X_t = 4\gamma^2\left(\left\|\theta_{t}^u\right\|^2 - \EE\left[\left\|\theta_{t}^u\right\|^2 \mid \mathcal{F}^{t-1}\right]\right)$, parameter $c$ as in \eqref{convex_second_martingale_bound}, $b = \frac{R^2}{7}$, $G = \frac{R^4}{294\ln\frac{8(K+1)}{\beta}}$:
    \begin{equation*}
        \PP\left\{|\circledFive| > \frac{R^2}{7}\quad \text{and}\quad \sum\limits_{t=0}^{T-1} \widetilde\sigma_{t}^2 \leq \frac{R^4}{294\ln\frac{8(K+1)}{\beta}}\right\} \leq 2\exp\left(- \frac{b^2}{2G + \nicefrac{2cb}{3}}\right) = \frac{\beta}{4(K+1)}.
    \end{equation*}
    Equivalently, we have
    \begin{equation}
        \PP\left\{ E_{\circledFive} \right\} \geq 1 - \frac{\beta}{4(K+1)},\quad \text{for}\quad E_{\circledFive} = \left\{ \text{either} \quad  \sum\limits_{t=0}^{T-1} \widetilde\sigma_{t}^2 > \frac{R^4}{294\ln\frac{8(K+1)}{\beta}} \quad \text{or}\quad |\circledFour| \leq \frac{R^2}{7}\right\}. \label{eq:clipped_SGD_convex_sum_3_upper_bound}
    \end{equation}
    In addition, $E_{T-1}$ implies that
    \begin{eqnarray}
        \sum\limits_{t=0}^{T-1} \widetilde\sigma_{t}^2 &\overset{\eqref{eq:clipped_SGD_convex_technical_11}}{\leq}& \frac{8\gamma^2R^2 (K+1)}{7\ln\frac{8(K+1)}{\beta}} \EE\left[\norm{\theta_t^u}^2\mid \mathcal{F}^{t-1}\right] \overset{\eqref{eq:convex_theta_u_norm_variance},\eqref{convex_first_step_size_condition}}{\leq} \frac{R^4}{294\ln\frac{8(K+1)}{\beta}}.
    \end{eqnarray}

\paragraph{Upper bound for $\circledSix$.} From $E_{T-1}$, and conditions on the step-size it follows that
\begin{eqnarray}
        \circledSix &=& 4\gamma^2\sum\limits_{t=0}^{T-1}\left\|\theta_{t}^b\right\|^2 \notag \\
        &\leq& 4\gamma^2 T \left(\frac{2^{2\alpha - 1}\sigma\left(\sigma^\alpha + \zeta_\lambda^\alpha\right)^{\frac{\alpha - 1}{\alpha}}}{\lambda^{\alpha - 1}}  
        + (\zeta_\lambda+\nicefrac{\lambda}{2})\frac{2^{2\alpha - 1}\left(\sigma^\alpha + \zeta_\lambda^\alpha\right)}{\lambda^{\alpha}}\notag + \zeta_\lambda\right)^2
        \notag \\
        &\overset{\eqref{convex_second_step_size_condition}}{\leq}& \frac{R^2}{7}.
        \label{eq:clipped_SGD_convex_sum_5_upper_bound}
\end{eqnarray}

\paragraph{Upper bound for $\circledSeven$.}  We have
\begin{eqnarray}
    4\gamma^2\sum_{t=0}^{T-1} \norm{\omega_t}^2 = 4\gamma^2\sigma_\omega^2 \sum_{t=0}^{T-1}\sum_{i=1}^d z_{t,i}^2, 
\end{eqnarray}
where $z_{t,i} \eqdef \nicefrac{\omega_{t,i}}{\sigma_\omega}$. Using Lemma \ref{lem: Chi-square_concentration}, we get
\begin{eqnarray}
    \PP\left\{\sum_{t=0}^{T-1}\sum_{i=1}^dz_{t,i}^2 >Td + 2\sqrt{Td\ln\frac{4(K+1)}{\beta}}+2\ln\frac{4(K+1)}{\beta}\right\}\leq \frac{\beta}{4(K+1)}.
\end{eqnarray}
Since $\gamma \leq \gamma_6$ for $\gamma_6$ defined in \eqref{convex_sixth_step_size_condition}, we obtain
\begin{eqnarray}
    \mathbb{P}\left\{\circledSeven > \frac{R^2}{7}\right\} \leq \frac{\beta}{4(K+1)},
\end{eqnarray}
which is equivalent to
\begin{align}
        \mathbb{P}\{E_{\circledSeven}\} \geq 1-\frac{\beta}{4(K+1)} \;\;\; \text{for} \;\;\; E_{\circledThree}=\left\{\left | \circledSeven \right| \leq \frac{R^2}{7}\right\}.
\end{align}
Now, we have the upper bounds for  $\circledOne, \circledTwo, \circledThree, \circledFour, \circledFive, \circledSix, \circledSeven$ . 
    Thus, probability event $E_{T-1} \cap E_{\circledOne} \cap  E_{\circledThree} \cap E_{\circledFive} \cap E_{\circledSeven}$ implies
\begin{eqnarray}
        R_T^2 &\leq& R^2 - 2\gamma\sum\limits_{l=0}^{t-1} \langle x^l - x^\star , \theta_l \rangle  - 2\gamma\sum\limits_{l=0}^{t-1} \langle x^l - x^\star , \omega_l \rangle + 2\gamma^2\sum\limits_{l=0}^{t-1} \|\theta_l\|^2 + 4\gamma^2\sum\limits_{l=0}^{t-1} \|\omega_l\|^2 \notag\\
        &\leq& R^2 +  \circledOne + \circledTwo + \circledThree + \circledFour + \circledFive + \circledSix + \circledSeven\notag\\
        &\leq& R^2 + \frac{R^2}{7} + \frac{R^2}{7} + \frac{R^2}{7} + \frac{R^2}{7} + \frac{R^2}{7} + \frac{R^2}{7} + \frac{R^2}{7} = 2R^2, \notag
\end{eqnarray}
    which is equivalent to \eqref{eq:clipped_SGD_convex_induction_inequality_1} and \eqref{eq:clipped_SGD_convex_induction_inequality_2} for $t = T$, and 
\begin{eqnarray}
        \PP\{E_T\}&\geq& \PP\left\{E_{T-1}  \cap E_{\circledOne} \cap E_{\circledThree} \cap E_{\circledFive} \cap E_{\circledSeven}\right\} \notag \\
        &=& 1 - \PP\left\{\overline{E}_{T-1}  \cup\overline{E}_{\circledOne} \cup \overline{E}  _{\circledThree}\cup \overline{E}_{\circledFive}\cup \overline{E}_{\circledSeven}\right\} \notag \\
        &\geq& 1 - \PP\{\overline{E}_{T-1}\} - \PP\{\overline{E}_{\circledOne}\} - \PP\{\overline{E}_{\circledThree}\}-\PP\{ \overline{E}_{\circledFive}\}- \PP\{\overline{E}_{\circledSeven}\} \notag \\ &\geq& 1 - \frac{(T+1)\beta}{K+1}.
\end{eqnarray}
    This finishes the inductive part of our proof, i.e., for all $k = 0,1,\ldots, K$ we have $\PP\{E_k\} \geq 1 - \nicefrac{(k+1)\beta}{(K+1)}$. In particular, for $k = K$ we have that with probability at least $1 - \beta$
\begin{equation*}
        \frac{1}{(K+1)}\sum_{t=0}^Kc_t(f(x^t) - f(x^\star)) {\leq}\frac{2R^2}{\gamma(K+1)}
\end{equation*}
    and $\{x^k\}_{k=0}^{K} \subseteq Q$, which follows from \eqref{eq:clipped_SGD_convex_induction_inequality_2}.
Now, we have to deal with $c_t$. To do so, we consider two possible cases for each $t = 0,1,\ldots, K$: either $c_t = 1$ or $c_t = \frac{\lambda}{2\|\nabla f(x^t)\|}$. We define the corresponding sets of indices: $\cT_1 \eqdef \{t \in \{0,1,\ldots, K\}\mid c_t = 1\}$ and $\cT_2 \eqdef \{t \in \{0,1,\ldots, K\}\mid c_t = \frac{\lambda}{2\|\nabla f(x^t)\|}\}$. Then, the above inequality can be rewritten as
\begin{equation}
    \frac{1}{(K+1)}\sum_{t \in \mathcal{T}_1}(f(x^t) - f(x^\star)) + \frac{1}{(K+1)}\sum_{t \in \mathcal{T}_2}\frac{\lambda(f(x^t) - f(x^\star))}{2\norm{\nabla f(x^t)}} \leq \frac{2R^2}{\gamma(K+1)}, \label{eq:pre_final_inequality}
\end{equation}
implying
\begin{equation}
    \frac{1}{(K+1)}\sum_{t\in\mathcal{T}_1}(f(x^t) - f(x^\star)) {\leq} \frac{2R^2}{\gamma(K+1)} \label{eq:pre_final_inequality_1}
\end{equation}
and
\begin{equation}
    \frac{1}{(K+1)}\sum_{t \in \mathcal{T}_2}\frac{\lambda(f(x^t) - f(x^\star))}{2\norm{\nabla f(x^t)}} {\leq}\frac{2R^2}{\gamma(K+1)}. \label{eq:pre_final_inequality_2}
\end{equation}
Using the corollary of smoothness assumption, i.e., $\norm{\nabla f(x^t)}\leq \sqrt{2L(f(x^t)-f(x^\star))}$, we get from \eqref{eq:pre_final_inequality_2} that
\begin{equation}
    \frac{1}{K+1}\sum_{t\in \mathcal{T}_2}\sqrt{f(x^t)-f(x^\star)}\leq \frac{4\sqrt{2L}R^2}{\lambda\gamma (K+1)}. \label{eq:pre_final_inequality_3}
\end{equation}
For inequality \eqref{eq:pre_final_inequality_1}, we follow the technique from \citep{koloskova2023revisiting} and apply inequality $x^2\geq 2\epsilon x-\epsilon^2$, which holding for any $\epsilon, x$. Setting $x^2 = f(x^t) - f(x^\star)$, we get
\begin{equation}
    \frac{1}{K+1}\sum_{t\in \mathcal{T}_1}\left(2\epsilon\sqrt{f(x^t)-f(x^\star)}-\epsilon^2\right) \leq \frac{2R^2}{\gamma(K+1)}, \notag
\end{equation}
implying
\begin{equation}
    \frac{1}{K+1}\sum_{t\in \mathcal{T}_1}\sqrt{f(x^t)-f(x^\star)}\leq \frac{R^2}{\gamma(K+1)\epsilon}+\frac{\epsilon}{2}. \notag
\end{equation}
Choosing $\epsilon = \frac{\sqrt{2}R}{\sqrt{\gamma (K+1)}}$, we obtain
\begin{equation}
    \frac{1}{K+1}\sum_{t\in \mathcal{T}_1}\sqrt{f(x^t)-f(x^\star)} \leq \sqrt{\frac{2R^2}{\gamma(K+1)}}. \label{eq:pre_final_inequality_4}
\end{equation}
Combining inequalities \eqref{eq:pre_final_inequality_3} and \eqref{eq:pre_final_inequality_4}, we get
\begin{equation}
    \frac{1}{K+1}\sum_{t=0}^K\sqrt{f(x^t)-f(x^\star)} \leq \sqrt{\frac{2R^2}{\gamma(K+1)}}+ \frac{4\sqrt{2L}R^2}{\lambda\gamma (K+1)},
\end{equation}
which implies
\begin{equation}
   \min_{t\in [0,K]} \left(f(x^t)-f(x^\star)\right) \leq \frac{4R^2}{\gamma(K+1)} + \frac{64LR^4}{\lambda^2\gamma^2(K+1)^2},
\end{equation}
where we have utilized the inequality $(a+b)^2\leq 2a^2+2b^2$. This concludes the proof.
\end{proof}
Theorem \ref{main_thm_convex} states 7 values for step-size, from which the smallest should be selected. To simplify matters, we demonstrate that if $\lambda$ is selected equal or smaller than the order of $\mathcal{O}\left(\left(\frac{K}{\ln K}\right)^{\nicefrac{1}{\alpha}}\right)$, then three step-sizes are redundant and can be omitted.
\\
\\
\begin{corollary}
    Let all conditions of Theorem \ref{main_thm_convex} hold. Furthermore, assume that $K$ is large and one selects $\lambda \leq \mathcal{O}\left(\left(\frac{K}{\ln K}\right)^{\nicefrac{1}{\alpha}}\right)$, then conclusions of Theorem \ref{main_thm_convex} are valid as long as $\gamma$ is selected to satisfy $\gamma\leq \min\left\{\nicefrac{1}{8L},\gamma_1,\gamma_2,\gamma_3\right\}$ where we have
    \begin{align*}
         \gamma_1 &\eqdef \frac{R}{42(2^{2\alpha-1}+1)^{1/2}\sigma^{\alpha/2}\lambda^{1-\alpha/2}\sqrt{6(K+1)\ln\frac{8(K+1)}{\beta}{\left(1+ {\color{black}{\frac{\zeta_\lambda^\alpha}{\sigma^\alpha}}}\right)}}}, \\ 
        \gamma_2 &\eqdef \frac{R\lambda^{\alpha-1}}{28(K+1)2^{2\alpha-1}\sigma^\alpha\left(1+\frac{\zeta_\lambda^\alpha}{\sigma^\alpha} \right)\left(\frac{\zeta_\lambda}{\lambda}+\frac{1}{2}+\frac{\lambda^{\alpha-1}\zeta_\lambda}{2^{2\alpha-1}\left(\sigma^\alpha+\zeta_\lambda^\alpha\right)}+\left(1+\frac{\zeta_\lambda^\alpha}{\sigma^\alpha}\right)^{-1/\alpha}\right)}, \\
        \gamma_3 &\eqdef \frac{R}{56\sigma_\omega\sqrt{d(K+1)}(\sqrt{2}+\sqrt{2}\phi)}.
    \end{align*}
\end{corollary}
\vspace{0.4cm}
\begin{proof} For large $K$, it is evident that $\gamma_3$ decreases at a rate of $\mathcal{O}\left(\sigma_\omega \sqrt{K\ln K}\right)$, while $\gamma_6$ in \eqref{convex_sixth_step_size_condition} decreases at a rate of $\mathcal{O}\left(\sigma_\omega \sqrt{K}\right)$. Subsequently, $\gamma_3$ dominates $\gamma_6$ and $\gamma_6$ can be omitted. Furthermore, $\gamma_5$ in \eqref{convex_fifth_step_size_condition} decreases with a rate of $\mathcal{O}\left(K^{\nicefrac{1}{\alpha}} (\ln K)^{1-\nicefrac{1}{\alpha}}\right)$ which is less than the rate of $\gamma_2$. It can be deduced that for large $\lambda$, $\gamma_2$ decreases at the rate $\mathcal{O}\left(K\right)$ which is faster than $\gamma_5$. If $\lambda$ is small, $\gamma_2$ dominates $\gamma_5$ again due to the $\lambda$ in the numerator of $\gamma_2$. Hence, $\gamma_5$ can be discarded. As for $\gamma_4$ in \eqref{convex_fourth_step_size_condition}, we know that $\sigma_\omega$ is on the order of $\mathcal{O}\left(\nicefrac{\lambda}{\epsilon}\sqrt{K\ln \left(\nicefrac{K}{\delta}\right) }\right)$. Hence, one can replace $\lambda$ with $\mathcal{O}\left(\nicefrac{\sigma_\omega \epsilon}{\sqrt{K \ln\left(\nicefrac{K}{\delta}\right)}}\right)$. Therefore, $\gamma_4$ decreases by the order $\mathcal{O}\left(\sigma_\omega\epsilon \sqrt{K \ln\left(\nicefrac{K}{\delta}\right)}\right)$, which is the same order as $\gamma_3$. Hence, $\gamma_4$ can be omitted, and the proof is complete.
 \end{proof}

\newpage
\section{Rate and Neighborhood for Clipped-SGD: Convex Case}
Now that we have established the convergence properties of \algname{DP-Clipped-SGD} for convex problems, we turn to evaluating its convergence rate. This rate depends critically on the choice of the step-size $\gamma$, and in general, the resulting expressions can be quite complex. To obtain more interpretable bounds, we consider simplified rate expressions by analyzing separate cases based on different ranges of $\lambda$. Since we focus on the asymptotic behavior, numerical constants are omitted for clarity.

In this section, we consider the cases without the DP noise ($\sigma_\omega=0$) and investigate all possible clipping levels. 
\paragraph{Case 1: $\lambda > 4LR$.} In this case, $\zeta_\lambda=0$, and the step-size conditions reduce to the following:
     \begin{eqnarray}
      \gamma \leq \mathcal{O}\left(\min \left\{\frac{1}{L},\frac{R}{\sigma^{\alpha/2}\lambda^{1-\alpha/2}\sqrt{K\ln\frac{K}{\beta}}},\frac{R\lambda^{\alpha-1}}{K\sigma^\alpha}\right\}\right). \label{convex_step_size_condtion_large_lambda}
     \end{eqnarray}
     In particular, when $\gamma$ equals the minimum from the above condition, the iterates produced by \algname{Clipped-SGD} after $K$ iterations with
probability at least $1-\beta$ satisfy 
\begin{eqnarray}
\min_{t\in [0,K]} f(x^t)-f(x^\star)= \mathcal{O} \left(\max \left\{ \eqref{convex_first_term_large_lambda_case_1}, \eqref{convex_second_term_large_lambda_case_1},\eqref{convex_third_term_large_lambda_case_1} \right\}\right), \label{large_lambda_objective_decrease}
\end{eqnarray}
where
\begin{eqnarray}
    &&R\lambda^{1-\alpha/2}\sigma^{\alpha/2}\sqrt{\frac{\ln \nicefrac{K}{\beta}}{K}}+ \frac{LR^2\sigma^\alpha \ln \nicefrac{K}{\beta}}{\lambda^\alpha  K},\label{convex_first_term_large_lambda_case_1} \\
    &&\frac{R\sigma^\alpha}{\lambda^{\alpha-1}}+\frac{LR^2\sigma^{2\alpha}}{\lambda^{2\alpha}}, \label{convex_second_term_large_lambda_case_1} \\
    && \frac{LR^2}{K} +\frac{L^3R^4}{\lambda^2K^2} .\label{convex_third_term_large_lambda_case_1}
\end{eqnarray}
We clearly see that the dominant term in \eqref{convex_first_term_large_lambda_case_1} is an increasing function of $\lambda$, and the dominant term in \eqref{convex_second_term_large_lambda_case_1} is a decreasing function. 
Solving for optimal $\lambda$ as the equilibrium of the dominant terms in \eqref{convex_first_term_large_lambda_case_1} and \eqref{convex_second_term_large_lambda_case_1}, we get $\lambda = \mathcal{O}\left(\sigma \left(\frac{K}{\ln \frac{K}{\beta}}\right)^\frac{1}{\alpha}\right)$.
Plugging in this $\lambda$, we get with probability at least $1-\beta$:
\begin{eqnarray}
\min_{t\in [0,K]} f(x^t)-f(x^\star)= \mathcal{O}\left( \max \left \{\eqref{convex_first_term_optimal_rate_neighborhood_large_lambda_case_1}, \eqref{convex_second_term_optimal_rate_neighborhood_large_lambda_case_1}\right \}\right),
\end{eqnarray}
where
\begin{eqnarray}
   && R\sigma \left(\frac{\ln \frac{K}{\beta}}{K}\right)^{\frac{\alpha-1}{\alpha}}+ \frac{LR^2\ln^2 \nicefrac{K}{\beta}}{  K^2}. \label{convex_first_term_optimal_rate_neighborhood_large_lambda_case_1} \\
   &&\frac{LR^2}{K} + \frac{L^3R^4 \left (\ln{\frac{K}{\beta}}\right)^{\frac{2}{\alpha}}}{\sigma^2 K^{\frac{2\alpha+2}{\alpha}}} .\label{convex_second_term_optimal_rate_neighborhood_large_lambda_case_1}
\end{eqnarray}
In this case, \algname{Clipped-SGD} converges to the exact optimum asymptotically with high probability, and the dominant term matches the one from \citet{sadiev2023high}. As it can be seen from $\eqref{convex_first_term_large_lambda_case_1}, \eqref{convex_second_term_large_lambda_case_1}$, when the clipping level is not that large, we converge to a neighborhood of the solution, but with a faster $\cO(\nicefrac{1}{\sqrt{K}})$ rate.

Next, when $\lambda \leq 4LR$, we have $\zeta_\lambda=\frac{4LR-\lambda}{2}$. As it can be seen from $\eqref{convex_first_step_size_condition}, \eqref{convex_second_step_size_condition}$, in these cases, we also have to consider the relation between $\lambda$ and $\sigma$. Thus, we split $\lambda \leq 4LR$ regime into $6$ different regimes to cover all possible cases.

\paragraph{Case 2: $\frac{4}{3}LR < \lambda \leq 4LR, \;\;\; \zeta_\lambda < \lambda < \sigma$.} In this case, the step-size conditions reduce to the following:
  \begin{eqnarray}
      \gamma \leq \mathcal{O}\left(\min \left\{\frac{1}{L},\frac{R}{\sigma^{\alpha/2}\lambda^{1-\alpha/2}\sqrt{K\ln\frac{K}{\beta}}},\frac{R\lambda^{\alpha-1}}{K\sigma^\alpha}\right\}\right) .\label{convex_step_size_condtion_small_lambda_case_2}
     \end{eqnarray}
As can be seen, the result is the same as in the previous case. The optimal $\lambda$ derived in the previous section violates the constraint that $\lambda \leq 4LR$; thus, the optimal $\lambda=4LR$. For this choice of $\lambda$, we have with probability at least $1-\beta$
\begin{eqnarray}
\min_{t\in [0,K]} f(x^t)-f(x^\star)= \mathcal{O} \left(\max \left\{ \eqref{convex_first_term_small_lambda_case_2_convex}, \eqref{convex_second_term_small_lambda_case_2_convex},\eqref{convex_third_term_small_lambda_case_2_convex} \right\}\right) ,\label{large_lambda_objective_decrease_convex}
\end{eqnarray}
where
\begin{eqnarray}
    &&\sqrt{R^{4-\alpha}L^{2-\alpha}\sigma^\alpha \frac{\ln{\nicefrac{K}{\beta}}}{K}} + \frac{R^{2-\alpha}\sigma^\alpha \ln{\nicefrac{K}{\beta}}}{L^{\alpha-1}K},\label{convex_first_term_small_lambda_case_2_convex} \\
    &&\frac{R^{2-\alpha}\sigma^{\alpha}}{L^{\alpha-1}} + \frac{\sigma^{2\alpha}}{L^{2\alpha-1}R^{2\alpha-2}},\label{convex_second_term_small_lambda_case_2_convex} \\
    &&\frac{LR^2}{K} +\frac{LR^2}{K^2}. \label{convex_third_term_small_lambda_case_2_convex}
\end{eqnarray}
\paragraph{Case 3: $\frac{4}{3}LR < \lambda \leq 4LR, \quad \zeta_\lambda < \sigma <\lambda$.} In this case, the step-size conditions reduce to the following:
\begin{eqnarray}
\gamma \leq \mathcal{O}\left(\min \left\{\frac{1}{L},\frac{R}{\sigma^{\alpha/2}\lambda^{1-\alpha/2}\sqrt{K\ln\frac{K}{\beta}{}}} ,\frac{R\lambda^{\alpha-1}}{K\max \{\sigma^{\alpha},\lambda^{\alpha-1}\zeta_\lambda\}}\right\}\right).
\end{eqnarray}
If $\max \{\sigma^{\alpha},\lambda^{\alpha-1}\zeta_\lambda\} = \sigma^\alpha$, then the bounds are similar to the previous case. If $\max \{\sigma^{\alpha},\lambda^{\alpha-1}\zeta_\lambda\} = \lambda^{\alpha-1}\zeta_\lambda$ is satisfied, $\min_{t\in [0,K]}f(x^t) - f(x^\star)$ is bounded with probability at least $1 - \beta$ by the maximum of the following terms:
\begin{eqnarray}
   && R\lambda^{1-\alpha/2}\sigma^{\alpha/2}\sqrt{\frac{\ln{\nicefrac{K}{\beta}}}{K}} + \frac{LR^2\sigma^\alpha \ln{\nicefrac{K}{\beta}}}{\lambda^\alpha K}, \\
    &&R\zeta_\lambda + \frac{LR^2\zeta_\lambda^2}{\lambda^2} ,\label{convex_second_term_small_lambda_case_3_convex} \\
    && \frac{LR^2}{K}+ \frac{L^3R^4}{\lambda^2 K^2}.
\end{eqnarray}
In the latter case (i.e., maximum occurring in the second argument), the optimal $\lambda$ is $4LR-\eta$, where $\eta$ is a sufficiently small number such that $\lambda^{\alpha-1}\zeta_\lambda \geq \sigma^\alpha$, i.e., $\lambda$ satisfies $\zeta_\lambda = \max\left\{\frac{\sigma^\alpha}{\lambda^{\alpha-1}}, \lambda^{1-\alpha/2}\sigma^{\alpha/2}\sqrt{\frac{\ln{\nicefrac{K}{\beta}}}{K}}\right\}$.
Note that the \eqref{convex_second_term_small_lambda_case_3_convex} is decreasing in $\lambda$, and $\lambda=4LR$ is not feasible. With this choice of $\lambda$, we get with probability at least $1-\beta$: 
\begin{eqnarray}
\min_{t\in [0,K]} f(x^t)-f(x^\star)= \mathcal{O} \left(\max \left\{ \eqref{convex_first_term_small_lambda_case_3_convex}, \eqref{convex_second_term_small_lambda_case_3_convex},\eqref{convex_third_term_small_lambda_case_3_convex} \right\}\right), \label{large_lambda_objective_decrease}
\end{eqnarray}
where
\begin{eqnarray}
    &&R\sqrt{(4LR-\eta)^{2-\alpha}\sigma^\alpha \frac{\ln{\nicefrac{K}{\beta}}}{K}} + \frac{LR^{2}\sigma^\alpha \ln{\nicefrac{K}{\beta}}}{(LR-\eta)^{\alpha}K},\label{convex_first_term_small_lambda_case_3_convex} \\
    &&\frac{R\eta}{2} + \frac{LR^2\eta^2}{(4LR-\eta)^2} ,\label{convex_second_term_small_lambda_case_3_convex} \\
    && \frac{L\Delta}{K}+ \frac{L^2\Delta^2}{(4\sqrt{L\Delta}-\eta)^2K^2}.\label{convex_third_term_small_lambda_case_3_convex}
\end{eqnarray}

\paragraph{Case 4: $\frac{4}{3}LR < \lambda \leq 4LR, \quad \sigma < \zeta_\lambda <\lambda$.} In this case, the step-size conditions reduce to the following:
\begin{eqnarray}
    \gamma \leq \cO \left(\min \left\{\frac{1}{L},\frac{R}{\zeta_\lambda^{\alpha/2}\lambda^{1-\alpha/2}\sqrt{K\ln\frac{K}{\beta}{}}},\frac{R\lambda^{\alpha-1}}{K(\lambda^{\alpha-1}\zeta_\lambda)}\right\}\right) ,
\end{eqnarray} 
and $\min_{t\in [0,K]}f(x^t) - f(x^\star)$ is bounded with probability at least $1 - \beta$ by the maximum of the following terms: 
\begin{eqnarray}
    &&R\lambda^{1-\alpha/2}\zeta_{\lambda}^{\alpha/2}\sqrt{\frac{\ln{\nicefrac{K}{\beta}}}{K}} + \frac{LR^2\zeta_\lambda^\alpha \ln{\nicefrac{K}{\beta}}}{\lambda^\alpha K}, \\
    &&R\zeta_\lambda + \frac{LR^2\zeta_\lambda^2}{\lambda^2}, \\
    &&\frac{LR^2}{K} + \frac{L^3R^4}{\lambda^2K^2}.
\end{eqnarray}
 The optimal in this case is $\lambda=4LR-2\sigma$, and the neighborhood of the convergence and the rate are presented below: with probability at least $1-\beta$
\begin{eqnarray}
\min_{t\in [0,K]} f(x^t)-f(x^\star)= \mathcal{O} \left(\max \left\{ \eqref{convex_first_term_small_lambda_case_4_convex}, \eqref{convex_second_term_small_lambda_case_4_convex},\eqref{convex_third_term_small_lambda_case_4_convex} \right\}\right) ,\label{large_lambda_objective_decrease}
\end{eqnarray}
where
\begin{eqnarray}
    &&R\sqrt{(4LR-2\sigma)^{2-\alpha}\sigma^\alpha \frac{\ln{\nicefrac{K}{\beta}}}{K}} + \frac{LR^{2}\sigma^\alpha \ln{\nicefrac{K}{\beta}}}{(4LR-2\sigma)^{\alpha}K},\label{convex_first_term_small_lambda_case_4_convex} \\
    &&{R\sigma} + \frac{LR^2\sigma^2}{(4LR-2\sigma)^2}, \label{convex_second_term_small_lambda_case_4_convex} \\
    && \frac{LR^2}{K}+ \frac{L^3R^4}{(4LR-2\sigma)^2K^2}.\label{convex_third_term_small_lambda_case_4_convex}
\end{eqnarray}

\paragraph{Case 5: $\lambda \leq \frac{4}{3}LR , \quad \lambda < \zeta_\lambda <\sigma$.} In this case, the step-size conditions reduce to the following:
\begin{eqnarray}
   \gamma \leq \cO\left(\min \left\{ \frac{1}{L},\frac{R}{\sigma^{\alpha/2}\lambda^{1-\alpha/2}\sqrt{K\ln\frac{K}{\beta}{}}}, \frac{R\lambda^{\alpha}}{K(\sigma^{\alpha}\zeta_\lambda)}\right\}\right).
\end{eqnarray}
Function sub-optimality $\min_{t\in [0,K]}f(x^t) - f(x^\star)$ is bounded with probability at least $1 - \beta$ by the maximum of the following terms:
\begin{eqnarray}
    && R\lambda^{1-\alpha/2}\sigma^{\alpha/2}\sqrt{\frac{\ln{\nicefrac{K}{\beta}}}{K}} + \frac{LR^2\sigma^\alpha \ln{\nicefrac{K}{\beta}}}{\lambda^\alpha K}, \\
    &&R\frac{\sigma^\alpha \zeta_\lambda}{\lambda^\alpha} + \frac{LR^2 \sigma^{2\alpha}\zeta_\lambda^2}{\lambda^{2\alpha+2}},\\
    && \frac{LR^2}{K} + \frac{L^3R^4}{\lambda^2K^2}.
\end{eqnarray}
In this regime, the optimal $\lambda=\frac{4}{3}LR$. With this choice of $\lambda$ we get: with probability at least $1 - \beta$
\begin{eqnarray}
\min_{t\in [0,K]} f(x^t)-f(x^\star)= \mathcal{O} \left(\max \left\{ \eqref{convex_first_term_small_lambda_case_5_convex}, \eqref{convex_second_term_small_lambda_case_5_convex}, \eqref{convex_third_term_small_lambda_case_5_convex}\right\}\right), \label{large_lambda_objective_decrease}
\end{eqnarray}
where
\begin{eqnarray}
  &&\sqrt{R^{4-\alpha}L^{2-\alpha}\sigma^\alpha \frac{\ln{\nicefrac{K}{\beta}}}{K}} + \frac{R^{2-\alpha}\sigma^\alpha \ln{\nicefrac{K}{\beta}}}{L^{\alpha-1}K}\label{convex_first_term_small_lambda_case_5_convex} ,\\
    &&\frac{R^{2-\alpha}\sigma^{\alpha}}{L^{\alpha-1}} + \frac{\sigma^{2\alpha}}{L^{2\alpha-1}R^{2\alpha-2}} \label{convex_second_term_small_lambda_case_5_convex} ,\\
    && \frac{LR^2}{K} + \frac{LR^2}{K^2}. \label{convex_third_term_small_lambda_case_5_convex}
\end{eqnarray}

\paragraph{Case 6: $\lambda \leq \frac{4}{3}LR , \quad \lambda < \sigma <\zeta_\lambda$.} In this case, the step-size conditions reduce to the following:
\begin{eqnarray}
   \gamma \leq \cO \left( \min \left\{\frac{1}{L},\frac{R}{\zeta_\lambda^{\alpha/2}\lambda^{1-\alpha/2}\sqrt{K\ln\frac{K}{\beta}{}}} 
    ,\frac{R\lambda^{\alpha}}{K(\zeta_\lambda^{\alpha+1})} \right\} \right).  
\end{eqnarray}
Function sub-optimality $\min_{t\in [0,K]}f(x^t) - f(x^\star)$ is bounded with probability at least $1 - \beta$ by the maximum of the following terms:
\begin{eqnarray}
    &&R\lambda^{1-\alpha/2}\zeta_\lambda^{\alpha/2}\sqrt{\frac{\ln{\nicefrac{K}{\beta}}}{K}} + \frac{LR^2\zeta_\lambda^\alpha \ln{\nicefrac{K}{\beta}}}{\lambda^\alpha K}, \label{eq:vdbjfbvjdfdfvdfvf}\\
    &&\frac{R\zeta_\lambda^{\alpha+1}}{\lambda^\alpha} + \frac{LR^2\zeta_\lambda^{2\alpha}}{\lambda^{2\alpha+2}}, \label{eq:dhjvfbvjdfbjdfjf}\\
    &&\frac{LR^2}{K}+\frac{L^3R^4}{\lambda^2 K^2}.
\end{eqnarray}
Next, we find the optimal $\lambda$ via equalizing the leading terms (the first ones) in \eqref{eq:vdbjfbvjdfdfvdfvf} and \eqref{eq:dhjvfbvjdfbjdfjf}. This results in $\lambda=\frac{4LR}{2C+1}$, where $C=\left(\frac{\ln{\frac{K}{\beta}}}{K}\right)^{\frac{1}{\alpha+2}}$, which is infeasible. Thus, in this regime, the optimal $\lambda$ is $\frac{4}{3}LR-\eta$, where $\eta \geq 0$ is such that $\lambda < \sigma < \zeta_\lambda$. Given this choice of $\lambda$, we obtain with probability at least $1 - \beta$
\begin{align}
    \min_{t\in [0,K]} f(x^t)-f(x^\star)= \mathcal{O} \left(\max \left\{ \eqref{convex_first_term_small_lambda_case_6_convex}, \eqref{convex_second_term_small_lambda_case_6_convex},\eqref{convex_third_term_small_lambda_case_6_convex} \right\}\right) ,\label{large_lambda_objective_decrease}
\end{align}
where
\begin{eqnarray}
    &&R (LR-\eta)^{1-\alpha/2}(LR+\eta)^{\alpha/2}\sqrt{ \frac{\ln{\nicefrac{K}{\beta}}}{K}} + \frac{LR^{2}(LR+\eta)^{2\alpha}\ln{\nicefrac{K}{\beta}}}{(LR-\eta)^{2\alpha+2} K}\label{convex_first_term_small_lambda_case_6_convex}, \\
    && \frac{R(LR+\eta)^{\alpha+1}}{(LR-\eta)^\alpha} +\frac{LR^2(LR+\eta)^{2\alpha}}{(LR-\eta)^{2\alpha+2}} \label{convex_second_term_small_lambda_case_6_convex}, \\
    && \frac{LR^2}{K}+ \frac{L^3R^4}{(LR-\eta)^2K^2}\label{convex_third_term_small_lambda_case_6_convex}.
\end{eqnarray}

\paragraph{Case 7: $\lambda \leq \frac{4}{3}LR , \quad \sigma < \lambda <\zeta_\lambda$.} In this case, the step-size conditions reduce to the following:
\begin{eqnarray}
   \gamma \leq \cO \left( \min \left\{\frac{1}{L},\frac{R}{\zeta_\lambda^{\alpha/2}\lambda^{1-\alpha/2}\sqrt{K\ln\frac{K}{\beta}{}}},\frac{R\lambda^{\alpha-1}}{K\max\left\{\frac{\zeta_\lambda^{\alpha+1}}{\lambda},\zeta_\lambda^{\alpha-1}\sigma \right\}} \right\}\right).  
\end{eqnarray}
We note that $\max\left\{\frac{\zeta_\lambda^{\alpha+1}}{\lambda},\zeta_\lambda^{\alpha-1}\sigma \right\} = \zeta^\alpha \max\left\{ \frac{\zeta_\lambda}{\lambda}, \frac{\sigma}{\lambda} \right\} = \frac{\zeta_\lambda^{\alpha+1}}{\lambda}$ since $\sigma < \lambda < \zeta_\lambda$. Therefore, similarly to the previous case, we have
\begin{eqnarray}
   \gamma \leq \cO \left( \min \left\{\frac{1}{L},\frac{R}{\zeta_\lambda^{\alpha/2}\lambda^{1-\alpha/2}\sqrt{K\ln\frac{K}{\beta}{}}} 
    ,\frac{R\lambda^{\alpha}}{K(\zeta_\lambda^{\alpha+1})} \right\} \right),  
\end{eqnarray}
and $\min_{t\in [0,K]}f(x^t) - f(x^\star)$ is bounded with probability at least $1 - \beta$ by the maximum of the following terms:
\begin{eqnarray}
    &&R\lambda^{1-\alpha/2}\zeta_\lambda^{\alpha/2}\sqrt{\frac{\ln{\nicefrac{K}{\beta}}}{K}} + \frac{LR^2\zeta_\lambda^\alpha \ln{\nicefrac{K}{\beta}}}{\lambda^\alpha K}, \label{eq:vdbjfbvjdfdfvdfvf_1}\\
    &&\frac{R\zeta_\lambda^{\alpha+1}}{\lambda^\alpha} + \frac{LR^2\zeta_\lambda^{2\alpha}}{\lambda^{2\alpha+2}}, \label{eq:dhjvfbvjdfbjdfjf_1}\\
    &&\frac{LR^2}{K}+\frac{L^3R^4}{\lambda^2 K^2}.
\end{eqnarray}
The optimal $\lambda$ is $\frac{4}{3}LR$, since the both leading terms in \eqref{eq:vdbjfbvjdfdfvdfvf_1} and \eqref{eq:dhjvfbvjdfbjdfjf_1} are decreasing in $\lambda$. With this choice, we get with probability at least $1 - \beta$
\begin{align}
    \min_{t\in [0,K]} f(x^t)-f(x^\star)= \mathcal{O} \left(\max \left\{ \eqref{convex_first_term_small_lambda_case_7_convex}, \eqref{convex_second_term_small_lambda_case_7_convex},\eqref{convex_third_term_small_lambda_case_7_convex} \right\}\right) ,\label{large_lambda_objective_decrease}
\end{align}
where
\begin{eqnarray}
    &&LR^2\sqrt{ \frac{\ln{\nicefrac{K}{\beta}}}{K}} + \frac{LR^{2}\ln{\nicefrac{K}{\beta}}}{K}\label{convex_first_term_small_lambda_case_7_convex}, \\
    &&R\sigma +\frac{\sigma^2}{L} \label{convex_second_term_small_lambda_case_7_convex}, \\
    && \frac{LR^2}{K}+ \frac{LR^2}{K^2}\label{convex_third_term_small_lambda_case_7_convex}.
\end{eqnarray}

Now that we have covered all regions, it's time to consider the DP noise as well.

\clearpage

\section{Rate and Neighborhood for \algname{DP-Clipped-SGD}: Convex Case}

To ensure the output of the algorithm is $(\varepsilon,\delta)$-differentially private in this setting, expectation minimization, it suffices to set the noise scale as $\sigma_\omega=\Theta \left( \frac{\lambda}{\varepsilon}\sqrt{K \ln \left( \frac{K}{\delta} \right) \ln \left( \frac{1}{\delta} \right)}\right)$ and apply the advanced composition theorem of \cite{dwork2014algorithmic}. In the finite sum case, one can reduce the amount of noise by a factor of $\sqrt{\ln \left( \frac{K}{\delta} \right)}$ as it was shown by \cite{abadi2016deep}. For the sake of brevity, in the DP case, we only consider two cases: large $\lambda$ and relatively small $\lambda$ regimes. The other cases can be derived with a similar analysis.

\paragraph{Case 1: $\lambda > 4LR$.} In this case, $\zeta_\lambda=0$, and the step-size conditions reduce to the following: 
     \begin{eqnarray}
      \gamma \leq \mathcal{O}\left(\min \left\{\frac{1}{L},\frac{R}{\sigma^{\alpha/2}\lambda^{1-\alpha/2}\sqrt{K\ln\frac{K}{\beta}}},\frac{R\lambda^{\alpha-1}}{K\sigma^\alpha},\frac{R}{\sigma_\omega\sqrt{dK\ln{\frac{K}{\beta}}}}\right\}\right) .\label{convex_step_size_condtion_large_lambda}
     \end{eqnarray}
     In particular, when $\gamma$ equals the minimum from step-size condition, then the iterates produced by \algname{DP-Clipped-SGD} after $K$ iterations with
probability at least $1-\beta$ satisfy 
\begin{eqnarray}
\min_{k\in [0,K]} f(x^t)-f(x^\star)= \mathcal{O} \left(\max \left\{ \eqref{convex_first_term_large_lambda_DP_convex}, \eqref{convex_second_term_large_lambda_DP_convex},\eqref{convex_third_term_large_lambda_DP_convex}, \eqref{convex_fourth_term_large_lambda_DP_convex} \right\}\right) ,\label{large_lambda_objective_decrease}
\end{eqnarray}
where
\begin{eqnarray}
    &&R\lambda^{1-\alpha/2}\sigma^{\alpha/2}\sqrt{\frac{\ln \nicefrac{K}{\beta}}{K}}+ \frac{LR^2\sigma^\alpha \ln \nicefrac{K}{\beta}}{\lambda^\alpha  K}\label{convex_first_term_large_lambda_DP_convex} ,\\
    &&\frac{R\sigma^\alpha}{\lambda^{\alpha-1}}+\frac{LR^2\sigma^{2\alpha}}{\lambda^{2\alpha}} \label{convex_second_term_large_lambda_DP_convex}, \\
    && \frac{LR^2}{K} +\frac{L^3R^4}{\lambda^2K^2} \label{convex_third_term_large_lambda_DP_convex} ,\\
    && R\sigma_\omega \sqrt{\frac{d\ln{\frac{K}{\beta}}}{K}} + \frac{LR^2\sigma_\omega^2d\ln{\frac{K}{\beta}}}{\lambda^2 K}\label{convex_fourth_term_large_lambda_DP_convex}.
\end{eqnarray}
Here, \eqref{convex_second_term_large_lambda_DP_convex} accounts for the bias caused by clipping, and \eqref{convex_fourth_term_large_lambda_DP_convex} accounts for the accumulation of DP noise. These terms are decreasing and increasing in $\lambda$ respectively, if we use $\sigma_\omega=\Theta \left( \frac{\lambda}{\varepsilon}\sqrt{K \ln \left( \frac{K}{\delta} \right) \ln \left( \frac{1}{\delta} \right)}\right)$. To find the optimal $\lambda$, we find the equilibrium of these two terms. Solving the equilibrium equation, we get $\lambda = \cO \left (\frac{\varepsilon \sigma^\alpha}{d\ln{\left(\frac{1}{\delta}\right) \ln \left(\frac{K}{\delta}\right)\ln \left({\frac{K}{\beta}}\right)}}\right)^{\frac{1}{\alpha}}$. 
Unless $\varepsilon\sigma^\alpha$ is large enough, this value violates the constraint that $\lambda > 4LR$, and it's not feasible. Thus, we have the following formula for the optimal $\lambda$:
\begin{align}
  \lambda=\max \left \{ 4LR, \left(\frac{\varepsilon \sigma^\alpha}{d\ln{\left(\frac{1}{\delta}\right) \ln \left(\frac{K}{\delta}\right)\ln \left({\frac{K}{\beta}}\right)}} \right)^{\frac{1}{\alpha}}\right \}.  
\end{align}
For this choice of $\lambda$, we get that with probability at least $1 - \beta$
\begin{eqnarray}
\min_{k\in [0,K]} f(x^t)-f(x^\star)= \mathcal{O} \left(\max \left \{\eqref{convex_first_term_large_lambda_DP_optimal_convex}, \eqref{convex_second_term_large_lambda_DP_optimal_convex}, \eqref{convex_third_term_large_lambda_DP_optimal_convex},\eqref{convex_fourth_term_large_lambda_DP_optimal_convex}\right \}\right),
\end{eqnarray}
with
\begin{eqnarray}
   && \max \left \{\sqrt{R^{4-\alpha}L^{2-\alpha}\sigma^\alpha \frac{\ln{\nicefrac{K}{\beta}}}{K}}  , R \left (\frac{\varepsilon \sigma^{\alpha}}{\sqrt{d\ln{\left(\frac{1}{\delta}\right) \ln \left(\frac{K}{\delta}\right)\ln \left({\frac{K}{\beta}}\right)}}}\right)^{\frac{1}{\alpha}}\sqrt{\frac{\ln^{\frac{3\alpha-2}{2\alpha}}{\frac{K}{\beta}}}{K}}\right \}\label{convex_first_term_large_lambda_DP_optimal_convex} ,\\
   &&\min \left \{ \frac{R^{2-\alpha}\sigma^{\alpha}}{L^{\alpha-1}},R\sigma \left( \frac{d\ln{\left(\frac{1}{\delta}\right) \ln \left(\frac{K}{\delta}\right)}}{\varepsilon}\right)^{\frac{\alpha-1}{\alpha}} \right \}
   \label{convex_second_term_large_lambda_DP_optimal_convex} ,\\
   &&\min \left \{\frac{LR^2}{K^2} , \frac{L^3 R^4 \left(d\ln{\left(\frac{1}{\delta}\right) \ln \left(\frac{K}{\delta}\right)}\right)^{\frac{1}{\alpha}}}{(\varepsilon)^{\frac{1}{\alpha}}\sigma} \frac{\ln^{\frac{1}{\alpha}}{\frac{K}{\beta}}}{K^2}\right \}+ \frac{LR^2}{K} \label{convex_third_term_large_lambda_DP_optimal_convex} ,\\
   &&\max \left \{\frac{LR^2}{\varepsilon}\sqrt{d\ln{\left(\frac{1}{\delta}\right) \ln \left(\frac{K}{\delta}\right)\ln \left({\frac{K}{\beta}}\right)}}, \frac{R\sigma \left (d\ln{\left(\frac{1}{\delta}\right) \ln \left(\frac{K}{\delta}\right)\ln \left({\frac{K}{\beta}}\right)}\right)^{\frac{\alpha+2}{2\alpha}}}{\varepsilon^{\frac{\alpha-1}{\alpha}}} \right \} \notag\\  && \;\;\;\;\;\;\;\;+ \frac{LR^2}{\varepsilon^2}d\ln{\left(\frac{1}{\delta}\right) \ln \left(\frac{K}{\delta}\right)\ln \left({\frac{K}{\beta}}\right)}\label{convex_fourth_term_large_lambda_DP_optimal_convex},
\end{eqnarray}
where, for the sake of brevity, we only report the dominant terms.

\paragraph{Case 2: $ \lambda \leq \frac{4}{3}LR
\quad \lambda < \sigma <\zeta_\lambda$.} In this case, the step-size conditions reduce to
\begin{eqnarray}
   \gamma \leq \cO \left( \min \left\{\frac{1}{L},\frac{R}{\zeta_\lambda^{\alpha/2}\lambda^{1-\alpha/2}\sqrt{K\ln\frac{K}{\beta}{}}} 
    ,\frac{R\lambda^{\alpha}}{K(\zeta_\lambda^{\alpha+1})},\frac{R}{\sigma_\omega\sqrt{dK\ln{\frac{K}{\beta}}}} \right\} \right),  
\end{eqnarray}
Taking $\gamma$ equal to the right-hand side, we get that with probability at least $1 - \beta$
\begin{align}
    \min_{t\in [0,K]} f(x^t)-f(x^\star) = \cO \left ( \left \{\eqref{convex_first_term_small_lambda_DP_convex} , \eqref{convex_second_term_small_lambda_DP_convex}, \eqref{convex_third_term_small_lambda_DP_convex}, \eqref{convex_fourth_term_small_lambda_DP_convex}\right \}\right),
\end{align}
with
\begin{eqnarray}
    &&R\lambda^{1-\alpha/2}\sigma^{\alpha/2}\sqrt{\frac{\ln{\nicefrac{K}{\beta}}}{K}} + \frac{LR^2\sigma^\alpha \ln{\nicefrac{K}{\beta}}}{\lambda^\alpha K} \label{convex_first_term_small_lambda_DP_convex},\\
    &&\frac{R\zeta_\lambda^{\alpha+1}}{\lambda^\alpha} + \frac{LR^2\zeta_\lambda^{2\alpha}}{\lambda^{2\alpha+2}}  \label{convex_second_term_small_lambda_DP_convex},\\
    &&\frac{LR^2}{K}+\frac{L^3R^4}{\lambda^2 K^2} \label{convex_third_term_small_lambda_DP_convex}, \\
    && R\sigma_\omega \sqrt{\frac{d \ln{\frac{K}{\beta}}}{K}} + \frac{LR^2\sigma_\omega^2d \ln{\frac{K}{\beta}}}{\lambda^2 K} \label{convex_fourth_term_small_lambda_DP_convex}.
\end{eqnarray}
Similarly to the previous case, we find the optimal $\lambda$ as the equilibrium of the leading terms in \eqref{convex_second_term_small_lambda_DP_convex} and \eqref{convex_fourth_term_small_lambda_DP_convex}. By doing so, we get the optimal $\lambda$:
\begin{align}
    \lambda= \min \left \{\frac{4}{3}LR, \frac{2\varepsilon LR}{\left(d\ln{\left(\frac{1}{\delta}\right) \ln \left(\frac{K}{\delta}\right)\ln \left({\frac{K}{\beta}}\right)}\right)^{\frac{1}{2\alpha+2}} +1}\right \}.
\end{align}
For this choice of $\lambda$, we get that with probability at least $1 - \beta$
\begin{align}
    \min_{k\in [0,K]} f(x^t)-f(x^\star)= \mathcal{O} \left(\max \left\{ \eqref{convex_first_term_small_lambda_DP_optimal}, \eqref{convex_second_term_small_lambda_DP_optimal},\eqref{convex_third_term_small_lambda_DP_optimal}, \eqref{convex_fourth_term_small_lambda_DP_optimal}\right\}\right), \label{large_lambda_objective_decrease}
\end{align}
with
\begin{eqnarray}
    &&\min \left \{\sqrt{R^{4-\alpha}L^{2-\alpha}\sigma^\alpha \frac{\ln{\nicefrac{K}{\beta}}}{K}}, \sqrt{\frac{R^{4-\alpha}(\varepsilon L)^{2-\alpha}\ln^{\frac{3\alpha}{4\alpha+4}}{\frac{K}{\beta}}}{\left(d\ln{\left(\frac{1}{\delta}\right) \ln \left(\frac{K}{\delta}\right)}\right)^{\frac{2-\alpha}{4\alpha+4}}K}} \right \}\label{convex_first_term_small_lambda_DP_optimal}, \\
    &&\max \left \{\frac{R^{2-\alpha}\sigma^{\alpha}}{L^{\alpha-1}} , \frac{R^{2-\alpha}\sigma^\alpha}{\varepsilon}\left(d\ln{\left(\frac{1}{\delta}\right) \ln \left(\frac{K}{\delta}\right)\ln \left({\frac{K}{\beta}}\right)}\right)^{\frac{\alpha-1}{2\alpha+2}}\right\}\label{convex_second_term_small_lambda_DP_optimal}, \\
    &&\max \left \{ \frac{LR^2}{K^2}, \frac{LR^2}{\varepsilon^2 K^2}\left (\left(d\ln{\left(\frac{1}{\delta}\right) \ln \left(\frac{K}{\delta}\right)\ln \left({\frac{K}{\beta}}\right)}\right)^{\frac{1}{2\alpha+2}} +1\right)^2 \right \}+\frac{LR^2}{K}
\label{convex_third_term_small_lambda_DP_optimal},\\
   &&\min \left \{\frac{LR^2}{\varepsilon}\sqrt{d\ln{\left(\frac{1}{\delta}\right) \ln \left(\frac{K}{\delta}\right)\ln \left({\frac{K}{\beta}}\right)}}, \frac{ LR^2\sqrt{\ln \frac{K}{\beta}}}{\left(d\ln{\left(\frac{1}{\delta}\right) \ln \left(\frac{K}{\delta}\right)\ln \left({\frac{K}{\beta}}\right)}\right)^{\frac{1}{2\alpha+2}} +1} \right \}\notag \\ &&\;\;\;\;\;\;\;\;+ \frac{LR^2d}{\varepsilon^2}\ln{\left(\frac{1}{\delta}\right) \ln \left(\frac{K}{\delta}\right)\ln \left({\frac{K}{\beta}}\right)}\label{convex_fourth_term_small_lambda_DP_optimal},
\end{eqnarray}
where, for the sake of brevity, we only report the dominant terms.
\clearpage

\section{Missing Proofs: Non-Convex Case}
Now, we focus on the case of non-convex functions. We start with the following lemma.
\begin{lemma}\label{lem: non_convex_descent_lemma}
    Let Assumptions \ref{ass:bounded_below}, \ref{ass:smoothness} hold on the set \newline $Q  =\left\{x\in \mathbb{R^d}|\exists y \in \mathbb{R}^d: f(y)\leq f^\ast + 2\Delta \text{  and  } \|x-y\|\leq \nicefrac{\sqrt{\Delta}}{20\sqrt{L}}\right\}$, where $\Delta \geq \Delta_0=f(x^0)-f^\ast$ and let $0 < \gamma \leq \nicefrac{1}{4L}$. If $x^k \in Q$  for all $k = 0,1,\ldots, K$ for some $K\geq 0$, then the iterates produced by \algname{DP-Clipped-SGD} satisfy
    \begin{eqnarray}
        \frac{\gamma}{2(T+1)} \sum_{t=0}^Tc_t \norm{\nabla f(x^t)}^2 &\leq&\frac{(f(x^0)-f^\ast)-(f(x^{T+1})-f^\ast)}{T+1} - \frac{\gamma}{T+1} \sum_{t=0}^T\langle \nabla f(x^t) , \theta_t \rangle \notag\\ 
          &-&\frac{\gamma}{T+1} \sum_{t=0}^T\langle \nabla f(x^t) , \omega_t \rangle + \frac{2L\gamma^2}{T+1} \sum_{t=0}^T\norm{\theta_t}^2 + \frac{L\gamma^2}{T+1}\sum_{t=0}^T\norm{\omega_t}^2, \notag
    \end{eqnarray}
    for all $T=0,1,\ldots,K$, and $\theta_t,  c_t$ are defined in $\eqref{theta_def}, \eqref{c_t_def}$ respectively.
\end{lemma}

\begin{proof} The smoothness of $f$ implies
\begin{eqnarray}
    f(x^{t+1}) &&\leq f(x^t)+\langle \nabla f(x^t), x^{t+1}-x^t \rangle + \frac{L}{2}\norm{x^{t+1}-x^t}^2 \notag \\
    && =f(x^t) -\gamma \langle \nabla f(x^t), \hat g_t + \omega_t + c_t \nabla f(x^t) -c_t \nabla f(x^t) \rangle \\
    &&\;\;\; + \frac{L\gamma^2}{2}\norm{\hat g_t + \omega_t + c_t \nabla f(x^t) -c_t \nabla f(x^t)}^2 \notag \\
    && \leq f(x^t) -\gamma c_t \norm{\nabla f(x^t)}^2 - \gamma \langle \nabla f(x^t), \theta_t \rangle - \gamma \langle \nabla f(x^t), \omega_t \rangle + L \gamma^2\norm{\omega_t}^2 \notag \\
    && \quad + 2L\gamma^2 \norm{\theta_t}^2 + 2L\gamma^2 c_t^2 \norm{\nabla f(x^t)}^2 \notag \\
    && = f(x^t)-(\gamma c_t-2  \gamma^2Lc_t^2)\norm{\nabla f(x^t)}^2 - \gamma \langle \nabla f(x^t),\theta_t \rangle  - \gamma \langle \nabla f(x^t),\omega_t \rangle  \notag \\
    && \quad + L \gamma^2\norm{\omega_t}^2 + + 2L\gamma^2 \norm{\theta_t}^2. \notag  
\end{eqnarray}
Rearranging the terms, utilizing $\gamma \leq \nicefrac{1}{4L}$, and $c_t^2\leq c_t$, we sum over $t$ to obtain
\begin{eqnarray}
        \frac{\gamma}{2(T+1)} \sum_{t=0}^Tc_t \norm{\nabla f(x^t)}^2 &\leq& \frac{(f(x^0)-f^\ast)-(f(x^{T+1})-f^\ast)}{T+1} - \frac{\gamma}{T+1} \sum_{t=0}^T\langle \nabla f(x^t) , \theta_t \rangle \notag\\ 
         \;\;\ &-&\frac{\gamma}{T+1} \sum_{t=0}^T\langle \nabla f(x^t) , \omega_t \rangle + \frac{2L\gamma^2}{T+1} \sum_{t=0}^T\norm{\theta_t}^2 + \frac{L\gamma^2}{T+1}\sum_{t=0}^T\norm{\omega_t}^2, \notag
\end{eqnarray}
which concludes the proof.
\end{proof}

The above lemma is utilized to prove the main convergence result for \algname{DP-Clipped-SGD}.
\begin{theorem}\label{main_thm_nonconvex}
     Let Assumptions \ref{ass:bounded_below}, \ref{ass:smoothness}, and \ref{ass:oracle} hold for the following set\newline $ Q =\left\{x\in \mathbb{R^d}|\exists y \in \mathbb{R}^d: f(y)\leq f^\ast + 2\Delta \text{  and  } \|x-y\|\leq \nicefrac{\sqrt{\Delta}}{20\sqrt{L}}\right\} $, where $\Delta \geq \Delta_0 = f(x^0)-f^\ast$, $\zeta_\lambda=\max\{0,2\sqrt{L\Delta}-\frac{\lambda}{2}\}$, and $\gamma=\min\{\nicefrac{1}{4L},\gamma_1,\gamma_2, \gamma_3, \gamma_4, \gamma_5, \gamma_6\}$,
     \begin{eqnarray}
          \gamma_1 &\eqdef&\frac{\sqrt{\Delta}}{21\sqrt{L}(2^{2\alpha-1}+1)^{1/2}\sigma^{\alpha/2}\lambda^{1-\alpha/2}\sqrt{6(K+1)\ln\frac{8(K+1)}{\beta}{\left(1+ {\color{black}{\frac{\zeta_\lambda^\alpha}{\sigma^\alpha}}}\right)}}}, \label{non_convex_first_step_size_condition} \\
        \gamma_2 &\eqdef& \frac{\sqrt{\Delta}\lambda^{\alpha-1}}{14\sqrt{L}(K+1)2^{2\alpha-1}\left(\sigma^\alpha+{\zeta_\lambda^\alpha} \right)\left(\frac{\zeta_\lambda}{\lambda}+\frac{1}{2}+\frac{\lambda^{\alpha-1}\zeta_\lambda}{2^{2\alpha-1}\left(\sigma^\alpha+\zeta_\lambda^\alpha\right)}+\left(1+\frac{\zeta_\lambda^\alpha}{\sigma^\alpha}\right)^{-1/\alpha}\right)} ,\label{non_convex_second_step_size_condition}\\
        \gamma_3 &\eqdef&  \frac{\sqrt{\Delta}}{14\sqrt{L}\sigma_\omega\sqrt{d(K+1)}(\sqrt{2}+\sqrt{2}\phi)} ,\label{non_convex_third_step_size_condition}\\
        \gamma_4&\eqdef& \frac{\sqrt{\Delta}} {20\sqrt{L}\left(\lambda + \sigma_\omega  \left(\sqrt{d}+ \sqrt{2\ln\left(\frac{K+1}{\beta}\right)}\right)\right) }, \label{non_convex_fourth_step_size_condition}\\
        \gamma_5&\eqdef&\frac{\sqrt{\Delta}}{28\lambda\sqrt{L}\ln\frac{8(K+1)}{\beta}},\label{non_convex_fifth_step_size_condition}\\
        \gamma_6&\eqdef&\frac{\sqrt{\Delta}}{\sqrt{L}\sigma_w\sqrt{7\left((K+1)d + 2\sqrt{(K+1)d\ln\frac{4(K+1)}{\beta}}+2\ln\frac{4(K+1)}{\beta}\right)}}.\label{non_convex_sixth_step_size_condition}
     \end{eqnarray}
     for some $K > 0$ and $\beta \in  (0, 1]$. Then, after $K$ iterations of \algname{DP-Clipped-SGD} the iterates with probability at least $1-\beta$ satisfy
     \begin{eqnarray}
        \min_{t \in [0,K]} \norm{\nabla f(x^t)}^2\leq \frac{8\Delta}{\gamma(K+1)}+ \frac{128\Delta^2}{\lambda^2 \gamma^2 (K+1)^2} .\label{obj_nonconvex}
     \end{eqnarray} 
\end{theorem}
\begin{proof}
    Let $\Delta_k = f(x^k) - f^\ast$ for all $k\geq 0$. 
    We aim to show by induction that $\Delta_{l} \leq 2\Delta$ with high probability. This fact will allow us to apply Lemma~\ref{lem: non_convex_descent_lemma} and then use Bernstein's inequality to evaluate the stochastic part of the upper-bound. More precisely, for each $k = 0,\ldots, K$ we define the probability event $E_k$ as follows. The inequalities
    \begin{eqnarray}
         &-\gamma \sum_{t=0}^T\langle \nabla f(x^t) , \omega_t + \theta_t \rangle + L\gamma^2 \sum_{t=0}^T \left (2\norm{\theta_t}^2 + \norm{\omega_t}^2 \right) \leq \Delta ,\label{eq:clipped_SGD_non_convex_induction_inequality_1}& \\
        &\Delta_t \leq 2\Delta ,\label{eq:clipped_SGD_non_convex_induction_inequality_2}& \\
        &\norm{\omega_t} \leq \sigma_\omega \left(\sqrt{d}+\sqrt{2\ln \left (\frac{K+1}{(t+1)\beta}\right)}\right),& \label{eq. Clipped_SGD_non_convex_third_induction_inequality}
    \end{eqnarray}
    hold for all $t = 0,1,\ldots, k$ simultaneously. We want to prove via induction that $\PP\{E_k\} \geq 1 - \nicefrac{(k+1)\beta}{(K+1)}$ for all $k = 0,1,\ldots, K$. For $k = 0$ the statement is trivial. Assume that the statement is true for some $k = T - 1 \leq K$ and $\PP\{E_{T-1}\} \geq 1 - \nicefrac{T\beta}{(K+1)}$. One needs to prove that $\PP\{E_{T}\} \geq 1 - \nicefrac{(T+1)\beta}{(K+1)}$. First, we notice that the probability event $E_{T-1}$ implies $\Delta_t \leq 2\Delta$ for all $t = 0,1,\ldots, T-1$, i.e., $x^t \in \{y \in \R^d\mid f(y) \leq f^\ast + 2\Delta\}$ for $t = 0,1,\ldots, T-1$. Moreover, due to the choice of clipping level $\lambda$, we have
    \begin{equation*}
        \|x^{T} - x^{T-1}\| = \gamma \|\hat{g}_{T-1}\|+\gamma\|\omega_{T-1}\| \leq \gamma \lambda + \gamma \sigma_\omega \left(\sqrt{d}+\sqrt{2\ln \left (\frac{K+1}{T\beta}\right)}\right) \overset{\eqref{non_convex_fourth_step_size_condition}}{\leq} \frac{\sqrt{\Delta}}{20\sqrt{L}}.
    \end{equation*}
     Therefore, $E_{T-1}$ implies $\{x^k\}_{k=0}^{T} \in Q$, meaning that the assumptions of Lemma~\ref{lem: non_convex_descent_lemma} are satisfied and we have
    \begin{eqnarray}
        \frac{\gamma}{2}\sum\limits_{l=0}^{t-1}\|\nabla f(x^l)\|^2 &\leq& \Delta_0 - \Delta_t - \gamma\sum\limits_{l=0}^{t-1}\langle \nabla f(x^l), \theta_l \rangle -\gamma\sum\limits_{l=0}^{t-1}\langle \nabla f(x^l), \omega_l \rangle + 2L\gamma^2 \sum\limits_{l=0}^{t-1} \|\theta_l\|^2 \notag \\ &+& L\gamma^2 \sum\limits_{l=0}^{t-1} \|\omega_l\|^2, \notag
\label{eq:clipped_SGD_non_convex_technical_1}
    \end{eqnarray}
    for all $t = 0,1,\ldots, T$ simultaneously. This event also implies
    \begin{eqnarray}
       \frac{\gamma}{2} \sum\limits_{l=0}^{t-1} c_l \|\nabla f(x^l)\|^2 &{\leq}& \Delta - \gamma\sum\limits_{k=0}^{t-1}\langle \nabla f(x^l), \theta_l \rangle -\gamma\sum\limits_{k=0}^{t-1}\langle \nabla f(x^l), \omega_l \rangle + 2L\gamma^2 \sum\limits_{l=0}^{t-1} \|\theta_l\|^2 \notag\\ &+& L\gamma^2 \sum\limits_{l=0}^{t-1} \|\omega_l\|^2 \notag \\ &{\leq}& {2\Delta}.\label{eq:clipped_SGD_non_convex_technical_1_1}
    \end{eqnarray}
    Taking into account that $\frac{\gamma}{2}\sum\limits_{l=0}^{T-1} c_l \|\nabla f(x^l)\|^2 \geq 0$, $E_{T-1}$ also implies
    \begin{eqnarray}
        \Delta_T \leq \Delta- \gamma\sum\limits_{l=0}^{T-1}\langle \nabla f(x^l), \theta_l \rangle -\gamma\sum\limits_{l=0}^{T-1}\langle \nabla f(x^l), \omega_l \rangle + 2L\gamma^2 \sum\limits_{l=0}^{T-1} \|\theta_l\|^2  + L\gamma^2 \sum\limits_{l=0}^{T-1} \|\omega_l\|^2 .\notag \label{eq:clipped_SGD_non_convex_technical_2}
    \end{eqnarray}
    Next, we define random vectors
    \begin{equation}
        \eta_t = \begin{cases} \nabla f(x^t),& \text{if } \|\nabla f(x^t)\| \leq 2\sqrt{L\Delta},\\ 0,&\text{otherwise}, \end{cases} \label{non_convex_eta_bound}
    \end{equation}
    for all $t = 0,1,\ldots, T-1$. By definition, these random vectors are bounded with probability 1
    \begin{equation}
        \|\eta_t\| \leq 2\sqrt{L\Delta}. \label{eq:clipped_SGD_non_convex_technical_6}
    \end{equation}
    Moreover, for $t = 1,\ldots, T-1$ event $E_{T-1}$, and corollary of smoothness imply
    \begin{eqnarray}
        \|\nabla f(x^{l})\| \overset{\eqref{non_convex_eta_bound}}{\leq} \sqrt{2L(f(x^l) - f^\ast)} = \sqrt{2L\Delta_l} \leq  2\sqrt{L\Delta} ,
    \end{eqnarray}
    meaning that $E_{T-1}$ implies that $\eta_t = \nabla f(x^t)$ for all $t = 0,1,\ldots, T-1$. 
    We notice that $\theta_{t} = \theta_{t}^u + \theta_{t}^b$, where $\theta_t^u$ and $\theta_t^b$ are defined in \eqref{Biased-Unbiased decomposition}. Using new notation, we get that $E_{T-1}$ implies
    \begin{eqnarray}
        \Delta_T 
        &\leq& \Delta \underbrace{-\gamma\sum\limits_{t=0}^{T-1}\langle \theta_t^u, \eta_t\rangle}_{\circledOne}  \underbrace{-\gamma\sum\limits_{t=0}^{T-1}\langle \theta_t^b, \eta_t\rangle}_{\circledTwo}  -\underbrace{\gamma \sum_{t=0}^{T-1}\langle \omega_t, \eta_t \rangle}_{\circledThree}+\underbrace{4L\gamma^2\sum\limits_{t=0}^{T-1}\EE\left[\left\|\theta_{t}^u\right\|^2\mid \mathcal{F}^{t-1}\right]}_{\circledFour} \notag\\
        &&\hspace{-0.5cm}+ \underbrace{4L\gamma^2\sum\limits_{t=0}^{T-1}\left(\left\|\theta_{t}^u\right\|^2 - \EE\left[\left\|\theta_{t}^u\right\|^2\mid \mathcal{F}^{t-1}\right]\right)}_{\circledFive}+ \underbrace{4L\gamma^2\sum\limits_{t=0}^{T-1}\left\|\theta_{t}^b\right\|^2}_{\circledSix} + \underbrace{L\gamma^2\sum\limits_{t=0}^{T-1}\left\|\omega_{t}\right\|^2}_{\circledSeven}. \label{eq:clipped_SGD_non_convex_technical_7}
    \end{eqnarray}
    It remains to derive good enough high-probability upper bounds for the terms $\circledOne, \circledTwo, \circledThree, \circledFour, \circledFive ,\circledSix, \circledSeven$. This amounts to proving $\circledOne + \circledTwo + \circledThree + \circledFour + \circledFive + \circledSix + \circledSeven \leq \Delta$ with high probability. In the subsequent parts of the proof, we will need to use the bounds for the norm and second moments of $\theta_{t}^u$ and $\theta_{t}^b$ many times. First, by definition of the clipping operator, we have with probability $1$ that
     \begin{equation}
        \|\theta_{t}^u\| \leq 2\lambda, \label{eq:non_convex_norm_theta_u_bound}
    \end{equation}
    and from Lemma~\ref{lem:Bias-Variance} we also have
    \begin{eqnarray}
        \|\theta_{t}^b\| &&\leq \frac{2^{2\alpha - 1}\sigma\left(\sigma^\alpha + (\max\{0, \norm{ \nabla f(x^t)} - \nicefrac{\lambda}{2}\})^\alpha\right)^{\frac{\alpha - 1}{\alpha}}}{\lambda^{\alpha - 1}} \notag \\ 
        &&+ \max\{\norm{ \nabla f(x^t)}, \nicefrac{\lambda}{2}\}\frac{2^{2\alpha - 1}\left(\sigma^\alpha + (\max\{0, \norm{ \nabla f(x^t)} - \nicefrac{\lambda}{2}\})^\alpha\right)}{\lambda^{\alpha}}\notag \\&&+ \max\{0, \norm{ \nabla f(x^t)} - \nicefrac{\lambda}{2}\}, \notag
    \end{eqnarray}
    \begin{eqnarray}
        \EE\left[\norm{\theta_t^u}^2\mid \mathcal{F}^{t-1}\right] \leq \frac{9(2^{2\alpha-1} + 1)\lambda^{2 - \alpha}\sigma^{\alpha}}{4} + \frac{9(2^{2\alpha-1}+1)\lambda^{2 - \alpha}(\max\{0, \norm{ \nabla f(x^t)} - \nicefrac{\lambda}{2}\})^{\alpha}}{4}. \notag
    \end{eqnarray}
As can be seen, these bounds are iteration-dependent. To overcome this, we bound $\norm{ \nabla f(x^t)}$ by $2\sqrt{L\Delta}$, which follows from $E_{T-1}$, i.e., $E_{T-1}$ implies
    \begin{eqnarray}
        \|\theta_{t}^b\| &&\leq \frac{2^{2\alpha - 1}\sigma\left(\sigma^\alpha + \zeta_\lambda^\alpha\right)^{\frac{\alpha - 1}{\alpha}}}{\lambda^{\alpha - 1}} + \left (\zeta_\lambda + \frac{\lambda}{2}\right)\frac{2^{2\alpha - 1}\left(\sigma^\alpha + \zeta_\lambda^\alpha\right)}{\lambda^{\alpha}} + \zeta_\lambda, \label{eq:non_convex_norm_theta_b_bound}
    \end{eqnarray}
        \begin{eqnarray}
        \EE\left[\norm{\theta_t^u}^2\mid \mathcal{F}^{t-1}\right] \leq \frac{9(2^{2\alpha-1} + 1)\lambda^{2 - \alpha}\sigma^{\alpha}}{4} + \frac{9(2^{2\alpha-1}+1)\lambda^{2 - \alpha}\zeta_\lambda^{\alpha}}{4}. \label{eq:non_convex_theta_u_norm_variance}
    \end{eqnarray}
\paragraph{Upper bound for $\circledOne$.} By definition of $\theta_{t}^u$, we have $\EE\left[\theta_{t}^u\mid \mathcal{F}^{t-1}\right] = 0$ and
    \begin{equation}
        \EE\left[-\gamma\langle\theta_t^u, \eta_t\rangle\mid \mathcal{F}^{t-1}\right] = 0. \notag
    \end{equation}
    Next, sum $\circledOne$ has bounded with probability $1$ terms:
    \begin{equation}
        |\gamma\left\la \theta_{t}^u, \eta_t\right\ra| {\leq} \gamma \|\theta_{t}^u\| \cdot \|\eta_t\| \overset{\eqref{non_convex_eta_bound}}{\leq} 4\gamma \lambda \sqrt{L \Delta}\overset{\eqref{non_convex_fifth_step_size_condition}}{\leq} \frac{\Delta}{7\ln\frac{8(K+1)}{\beta}} \eqdef c. \label{non_convex_first_martingale_bound} 
    \end{equation}
    The summands also have bounded conditional variances $\sigma_t^2 \eqdef \EE\left[\gamma^2\langle\theta_t^u, \eta_t\rangle^2\mid \mathcal{F}^{t-1}\right]$:
    \begin{equation}
        \sigma_t^2 \leq \EE\left[\gamma^2\|\theta_{t}^u\|^2\cdot \|\eta_t\|^2\mid \mathcal{F}^{t-1}\right]{\leq} 4\gamma^2L\Delta \EE\left[\|\theta_{t}^u\|^2\mid \mathcal{F}^{t-1}\right].
    \end{equation}
    In other words, we showed that $\{-\gamma\left\la \theta_{t}^u, \eta_t\right\ra\}_{t=0}^{T-1}$ is a bounded martingale difference sequence with bounded conditional variances $\{\sigma_t^2\}_{t=0}^{T-1}$. Next, we apply Bernstein's inequality (Lemma~\ref{lem: Bernstein_inequality}) with $X_t = -\gamma\left\la \theta_{t}^u, \eta_t\right\ra$, parameter $c$ as in \eqref{non_convex_first_martingale_bound}, $b = \frac{\Delta}{7}$, $G = \frac{\Delta^2}{294\ln\frac{8(K+1)}{\beta}}$:
    \begin{equation*}
        \PP\left\{|\circledOne| > \frac{\Delta}{7}\quad \text{and}\quad \sum\limits_{t=0}^{T-1} \sigma_{t}^2 \leq \frac{\Delta^2}{294\ln\frac{8(K+1)}{\beta}}\right\} \leq 2\exp\left(- \frac{b^2}{2G + \nicefrac{2cb}{3}}\right) = \frac{\beta}{4(K+1)}.
    \end{equation*}
    Equivalently, we have
    \begin{equation}
        \PP\left\{ E_{\circledOne} \right\} \geq 1 - \frac{\beta}{4(K+1)},\quad \text{for}\quad E_{\circledOne} = \left\{ \text{either} \quad  \sum\limits_{t=0}^{T-1} \sigma_{t}^2 > \frac{\Delta^2}{294\ln\frac{8(K+1)}{\beta}} \quad \text{or}\quad |\circledOne| \leq \frac{\Delta}{7}\right\}. \label{eq:clipped_SGD_non_convex_sum_1_upper_bound}
    \end{equation}
    In addition, $E_{T-1}$ implies that
    \begin{eqnarray}
        \sum\limits_{t=0}^{T-1} \sigma_{t}^2 &&{\leq} 4\gamma^2L\Delta \sum\limits_{t=0}^{T-1}  \EE\left[\|\theta_{t}^u\|^2 \mid \mathcal{F}^{t-1}\right] \notag\\
        &&\overset{\eqref{eq:non_convex_theta_u_norm_variance}}{\leq} 9\gamma^2L\Delta T\left(\left(2^{2\alpha-1}+1 \right)\lambda^{2-\alpha}\sigma^\alpha + (2^{2\alpha-1}+1)\lambda^{2-\alpha}\zeta_\lambda\right) \notag \\
        &&\overset{\eqref{non_convex_first_step_size_condition}}{\leq} \frac{\Delta^2}{294 \ln\frac{8(K+1)}{\beta}}. \label{eq:clipped_SGD_non_convex_sum_1_variance_bound}
    \end{eqnarray}
\paragraph{Upper bound for $\circledTwo$.} From $E_{T-1}$ it follows that
\begin{eqnarray}
        \circledTwo &&= -\gamma\sum\limits_{t=0}^{T-1}\langle \theta_t^b, \eta_t \rangle {\leq} \gamma\sum\limits_{t=0}^{T-1}\|\theta_{t}^b\|\cdot \|\eta_t\| \notag \\
        && \overset{\eqref{eq:non_convex_norm_theta_b_bound}}{\leq} 2\gamma \sqrt{L\Delta}T \left(\frac{2^{2\alpha - 1}\sigma\left(\sigma^\alpha + \zeta_\lambda^\alpha\right)^{\frac{\alpha - 1}{\alpha}}}{\lambda^{\alpha - 1}} \notag 
        + (\zeta_\lambda+\nicefrac{\lambda}{2})\frac{2^{2\alpha - 1}\left(\sigma^\alpha + \zeta_\lambda^\alpha\right)}{\lambda^{\alpha}}\notag + \zeta_\lambda\right) \notag \\
        && \overset{\eqref{non_convex_second_step_size_condition}}{\leq} \frac{\Delta}{7}.
\end{eqnarray}
\paragraph{Upper bound for $\circledThree$.}
We have
\begin{eqnarray}
     |\circledThree|  =  \left|-\gamma \sum_{t=0}^{T-1}\langle \omega_t , \eta_t\rangle \right| = \left|\sum_{t=0}^{T-1}\sum_{i=1}^d \gamma\omega_{t,i},\eta_{t,i}\right|,
\end{eqnarray}
where $\eta_{t,i} \eqdef [\eta_t]_i$ and $\omega_{t,i} \eqdef [\omega_t]_i$ denote the $i$-th components of $\eta_t$ and $\omega_t$ respectively.

Each summand is the product of a zero-mean Gaussian random variable and a bounded random variable, resulting in the product being a zero-mean light-tailed random variable with parameter  $\sigma_{t,i}^2=16\gamma^2 L\Delta\sigma_\omega^2$. To prove this, consider
\begin{eqnarray}
   \EE \left[\exp \left(\frac{\gamma^2}{\sigma_{t,i}^2}\left | \eta_{t,i}^2 \omega_{t,i}^2\right |\right) \mid \cF^{t-1}\right] &\overset{\eqref{eq:clipped_SGD_non_convex_technical_6}}{\leq}& \EE\left[\exp \left(\frac{4 L\Delta \gamma^2}{16\gamma^2L\Delta\sigma_\omega^2} \left | \omega_{t,i}\right |^2\right)\right] \notag \\ 
   &\leq&   \exp \left (\frac{|\omega_{t,i}|^2}{4\sigma^2_{\omega}}\right) \overset{(ii)}{\leq} \exp(1),
\end{eqnarray}
where $(ii)$ uses the fact that $\omega_{t,i}^2$ is a sub-Gaussian random variable with parameter $\sigma_\omega^2$.
Now that we have established the light-tailedness of summands, we can use the Lemma \ref{lem: subGaussian_norm_concentration} to obtain
\begin{eqnarray}
    \PP\left\{\left |{\sum_{t=0}^{T-1}\sum_{i=1}^d \gamma\eta_{t,i}\omega_{t,i}}\right | > \left(\sqrt{2}+\sqrt{2}\phi \right)\sqrt{\sum_{t=0}^{K}\sum_{i=1}^d 4\gamma^2L\Delta\sigma_\omega^2}\right\} &\leq& \exp\left(\frac{-\phi^2}{3}\right)\\
    &=&\frac{\beta}{4(K+1)}.
\end{eqnarray}
The choice of $\gamma \leq \gamma_3$ for $\gamma_3$ defined in \eqref{non_convex_third_step_size_condition} implies 
\begin{equation*}
    \left(\sqrt{2}+\sqrt{2}\phi \right)\sqrt{\sum_{t=0}^{T-1}\sum_{i=1}^d 4\gamma^2L\Delta\sigma_\omega^2} \leq \left(\sqrt{2}+\sqrt{2}\phi \right)\sqrt{4\gamma^2L\Delta(K+1)d\sigma_\omega^2} \overset{\eqref{non_convex_third_step_size_condition}}{\leq} \frac{\Delta}{7},
\end{equation*}
and
\begin{align}
    \mathbb{P}\{E_{\circledThree}\} \geq 1-\frac{\beta}{4(K+1)} \;\;\; \text{for} \;\;\; E_{\circledThree}=\left\{\left | \circledThree \right| > \frac{\Delta}{7}\right\}.
\end{align}\\
\paragraph{Upper bound for $\circledFour$.} From $E_{T-1}$ and the conditions on the step-size, it follows that
\begin{eqnarray}
        \circledFour &=& 2L\gamma^2\sum\limits_{t=0}^{T-1}\EE\left[\left\|\theta_{t}^u\right\|^2 \mid \mathcal{F}^{t-1}\right] \notag \\
        &\overset{\eqref{eq:non_convex_theta_u_norm_variance}}{\leq}& 2LT\gamma^2 \left(\frac{9(2^{2\alpha-1} + 1)\lambda^{2 - \alpha}\sigma^{\alpha}}{4} + \frac{9(2^{2\alpha-1}+1)\lambda^{2 - \alpha}\zeta_\lambda^{\alpha}}{4}\right) \notag\\
        &\overset{\eqref{non_convex_first_step_size_condition}}{\leq}& \frac{\Delta}{7}.\label{eq:clipped_SGD_non_convex_sum_4_upper_bound}
\end{eqnarray}

\paragraph{Upper bound for $\circledFive$.} First, we have
    \begin{equation}
        \EE\left[2L\gamma^2\left(\left\|\theta_{t}^u\right\|^2 - \EE\left[\left\|\theta_{t}^u\right\|^2\mid \mathcal{F}^{t-1}\right]\right)\mid \mathcal{F}^{t-1}\right] = 0. \notag
    \end{equation}
    Next, sum $\circledFive$ has bounded with probability $1$ terms:
    \begin{eqnarray}
        \left|2L\gamma^2\left(\left\|\theta_{t}^u\right\|^2 - \EE\left[\left\|\theta_{t}^u\right\|^2\mid \mathcal{F}^{t-1}\right]\mid \mathcal{F}^{t-1}\right)\right| &\leq& 2L\gamma^2\left( \|\theta_{t}^u\|^2 +   \EE\left[\left\|\theta_{t}^u\right\|^2\mid \mathcal{F}^{t-1}\right]\right)\notag\\
        &{\leq}& 16L\gamma^2\lambda^2\overset{\eqref{non_convex_fifth_step_size_condition}}{\leq} \frac{\Delta}{7 \ln\frac{8(K+1)}{\beta}}\eqdef c. \label{eq:clipped_SGD_non_convex_technical_10}
    \end{eqnarray}
    The summands also have bounded conditional variances as shown below:
    \begin{eqnarray}
            \widetilde\sigma_t^2 &\eqdef& \EE\left[4L^2\gamma^4\left(\left\|\theta_{t}^u\right\|^2 - \EE\left[\left\|\theta_{t}^u \right\|^2 \mid \mathcal{F}^{t-1}\right]\right)^2 \mid \mathcal{F}^{t-1}\right]\\
        \widetilde\sigma_t^2 &\overset{\eqref{eq:clipped_SGD_non_convex_technical_10}}{\leq}& \frac{\Delta}{7 \ln\frac{8(K+1)}{\beta}} \EE\left[2L\gamma^2\left|\left\|\theta_{t}^u\right\|^2 - \EE\left[\left\|\theta_{t}^u\right\|^2 \mid \mathcal{F}^{t-1}\right]\right| \mid \mathcal{F}^{t-1}\right] \notag \\&\leq& \frac{4L\gamma^2 \Delta}{7\ln\frac{8(K+1)}{\beta}} \EE\left[\|\theta_{t}^u\|^2 \mid \mathcal{F}^{t-1}\right], \label{eq:clipped_SGD_non_convex_technical_11}
    \end{eqnarray}
    since $\ln\frac{8K}{\beta} \geq 1$. In other words, we showed that $\left\{2L\gamma^2\left(\left\|\theta_{t}^u\right\|^2 - \EE\left[\left\|\theta_{t}^u\right\|^2\mid \mathcal{F}^{t-1}\right]\right)\right\}_{t=0}^{T-1}$ is a bounded martingale difference sequence with bounded conditional variances $\{\widetilde\sigma_t^2\}_{t=0}^{T-1}$. Next, we apply Bernstein's inequality (Lemma~\ref{lem: Bernstein_inequality}) with $X_t = 2L\gamma^2\left(\left\|\theta_{t}^u\right\|^2 - \EE\left[\left\|\theta_{t}^u\right\|^2\mid \mathcal{F}^{t-1}\right]\right)$, parameter $c$ as in \eqref{eq:clipped_SGD_non_convex_technical_10}, $b = \frac{\Delta}{7}$, $G = \frac{\Delta^2}{294\ln\frac{8(K+1)}{\beta}}$:
    \begin{equation*}
        \PP\left\{|\circledFive| > \frac{\Delta}{7}\quad \text{and}\quad \sum\limits_{t=0}^{T-1} \widetilde\sigma_{t}^2 \leq \frac{\Delta^2}{294\ln\frac{8(K+1)}{\beta}}\right\} \leq 2\exp\left(- \frac{b^2}{2G + \nicefrac{2cb}{3}}\right) = \frac{\beta}{4(K+1)}.
    \end{equation*}
    Equivalently, we have
    \begin{equation}
        \PP\left\{ E_{\circledFive} \right\} \geq 1 - \frac{\beta}{4(K+1)},\quad \text{for}\quad E_{\circledFour} = \left\{ \text{either} \quad  \sum\limits_{t=0}^{T-1} \widetilde\sigma_{t}^2 > \frac{\Delta^2}{294\ln\frac{8(K+1)}{\beta}} \quad \text{or}\quad |\circledFive| \leq \frac{\Delta}{7}\right\}. \label{eq:clipped_SGD_non_convex_sum_3_upper_bound}
    \end{equation}
    In addition, $E_{T-1}$ implies that
    \begin{eqnarray}
        \sum\limits_{t=0}^{T-1} \widetilde\sigma_{t}^2 &{\leq}& \frac{4L\gamma^2\Delta}{7\ln\frac{8(K+1)}{\beta}} \sum\limits_{t=0}^{T-1}  \EE\left[\|\theta_{t}^u\|^2\mid \mathcal{F}^{t-1}\right] \overset{\eqref{eq:non_convex_theta_u_norm_variance},\eqref{non_convex_first_step_size_condition}}{\leq} \frac{\Delta^2}{294 \ln\frac{8(K+1)}{\beta}}. \label{eq:clipped_SGD_non_convex_sum_3_variance_bound}
    \end{eqnarray}

\paragraph{Upper bound for $\circledSix$.} From $E_{T-1}$, and the conditions on the step-size it follows that
\begin{eqnarray}
        \circledSix &=& L\gamma^2\sum\limits_{t=0}^{T-1}\left\|\theta_{t}^b\right\|^2 \\
        &\leq& L\gamma^2\left(\frac{2^{2\alpha - 1}\sigma\left(\sigma^\alpha + \zeta_\lambda^\alpha\right)^{\frac{\alpha - 1}{\alpha}}}{\lambda^{\alpha - 1}}  
        + (\zeta_\lambda+\nicefrac{\lambda}{2})\frac{2^{2\alpha - 1}\left(\sigma^\alpha + \zeta_\lambda^\alpha\right)}{\lambda^{\alpha}}\notag + \zeta_\lambda\right)^2\\
        &\overset{\eqref{non_convex_second_step_size_condition}}{\leq}& \frac{\Delta}{7}.\label{eq:clipped_SGD_non_convex_sum_5_upper_bound}
\end{eqnarray}
\paragraph{Upper bound for $\circledSeven$.} We have 
\begin{eqnarray}
    \circledSeven=L\gamma^2\sum_{t=0}^{T-1} \norm{\omega_t}^2 = L\gamma^2\sigma_\omega^2 \sum_{t=0}^{T-1}\sum_{i=1}^dz_{t,i}^2,
\end{eqnarray}
where $z_{t,i}\eqdef \nicefrac{\omega_{t,i}}{\sigma_\omega}$ . Using Lemma \ref{lem: Chi-square_concentration}, we get
\begin{eqnarray}
    \PP\left\{\sum_{t=0}^{T-1}\sum_{i=1}^dz_{t,i}^2 >Td + 2\sqrt{Td\ln\frac{4(K+1)}{\beta}}+2\ln\frac{4(K+1)}{\beta}\right\}\leq \frac{\beta}{4(K+1)}.
\end{eqnarray}
Since $\gamma \leq \gamma_6$, for $\gamma_6$ defined in \eqref{non_convex_sixth_step_size_condition}
\begin{eqnarray} 
    \mathbb{P}\left\{\circledSeven > \frac{\Delta}{7}\right\} \leq \frac{\beta}{4(K+1)}.
\end{eqnarray}
Equivalently, we have
\begin{align}
    \mathbb{P}\{E_\circledSeven\} \geq 1-\frac{\beta}{4(K+1)} \;\; \text{for} \;\; E_\circledSeven = \left\{|\circledSeven | \leq \frac{\Delta}{7}\right\}.
\end{align}
Now, we have the upper bounds for  $\circledOne, \circledTwo, \circledThree, \circledFour, \circledFive, \circledSix, \circledSeven$  . 
    Thus, probability event $E_{T-1} \cap E_{\circledOne} \cap E_{\circledThree} \cap E_{\circledFour} \cap E_{\circledSeven}$ implies
\begin{eqnarray}
        \Delta_T &\leq& \Delta + \frac{\Delta}{7} + \frac{\Delta}{7} + \frac{\Delta}{7} + \frac{\Delta}{7} + 
        \frac{\Delta}{7} +
        \frac{\Delta}{7}+
        \frac{\Delta}{7} = 2\Delta, \notag
\end{eqnarray}
    which is equivalent to \eqref{eq:clipped_SGD_non_convex_induction_inequality_1} and \eqref{eq:clipped_SGD_non_convex_induction_inequality_2} for $t = T$, and 
\begin{equation}
        \PP\{E_T\}\geq \PP\left\{E_{T-1} \cap E_{\circledOne} \cap E_{\circledThree} \cap E_{\circledFour} \cap E_{\circledSeven}\right\} = 1 - \PP\left\{\overline{E}_{T-1} \cup \overline{E}_{\circledOne} \cup \overline{E}  _{\circledThree}\cup \overline{E}_{\circledFour}\cup \overline{E}_{\circledSeven}\right\} \notag
\end{equation}
\begin{equation}
        \geq 1 - \PP\{\overline{E}_{T-1}\} - \PP\{\overline{E}_{\circledOne}\} - \PP\{\overline{E}_{\circledThree}\}-\PP\{ \overline{E}_{\circledFour}\}- \PP\{\overline{E}_{\circledSeven}\} \geq 1 - \frac{(T+1)\beta}{K+1}.
\end{equation}
    This finishes the inductive part of our proof, i.e., for all $k = 0,1,\ldots, K$ we have $\PP\{E_k\} \geq 1 - \nicefrac{(k+1)\beta}{(K+1)}$. In particular, for $k = K$ and with probability at least $1 - \beta$, we have
\begin{equation*}
        \frac{1}{K+1}\sum\limits_{t=0}^{K} c_t\|\nabla f(x^t)\|^2 \overset{\eqref{eq:clipped_SGD_non_convex_technical_1_1}}{\leq}\frac{4\Delta}{\gamma(K+1)},
\end{equation*}
    and $\{x^t\}_{t=0}^{K} \in Q$, which follows from \eqref{eq:clipped_SGD_non_convex_induction_inequality_2}.
Now we have to deal with $c_t$. To do so, we consider two possible cases for each $t = 0,1,\ldots,K$. We either have $c_t = 1$ or $c_t = \frac{\lambda}{2\|\nabla f(x^t)\|}$. We define the corresponding sets of indices: $\cT_1 \eqdef \{t \in \{0,1,\ldots, K\}\mid c_t = 1\}$ and $\cT_2 \eqdef \{t \in \{0,1,\ldots, K\}\mid c_t = \frac{\lambda}{2\|\nabla f(x^t)\|} \}$. Then, the above inequality can be written as
\begin{equation}
    \frac{1}{(K+1)}\sum_{t\in\mathcal{T}_1}\norm{\nabla f(x^t)}^2 + \frac{1}{(K+1)}\sum_{t\in \mathcal{T}_2}\frac{\lambda\norm{\nabla f(x^t)}^2}{2\norm{\nabla f(x^t)}} \leq \frac{4\Delta}{\gamma(K+1)}, \notag
\end{equation}
implying
\begin{equation}
    \frac{1}{(K+1)}\sum_{t\in\mathcal{T}_1}\norm{\nabla f(x^t)} ^2{\leq}\frac{4\Delta}{\gamma(K+1)}, \label{eq:non_convex_pre_final_1}
\end{equation}
and
\begin{equation}
    \frac{1}{K+1}\sum_{t\in \mathcal{T}_2} \norm{\nabla f(x^t)}\leq \frac{8\Delta}{\lambda\gamma (K+1)} ,\label{eq:non_convex_pre_final_2}
\end{equation}
For inequality \eqref{eq:non_convex_pre_final_1}, we follow the technique from \citep{koloskova2023revisiting} and apply inequality $x^2\geq 2\epsilon x-\epsilon^2$, holding for any $\epsilon, x >0$. Taking $x = \norm{\nabla f(x^t)}^2$, we get
\begin{equation}
    \frac{1}{K+1}\sum_{t\in \mathcal{T}_1}\left(2\epsilon \norm{\nabla f(x^t)}-\epsilon^2\right) \leq \frac{4\Delta}{\gamma (K+1)}, \notag
\end{equation}
implying
\begin{equation}
    \frac{1}{K+1}\sum_{t\in \mathcal{T}_1}\norm{\nabla f (x^t)}\leq \frac{2\Delta}{\gamma (K+1)\epsilon}+\frac{\epsilon}{2}.\notag
\end{equation}
Upon selecting $\epsilon=\frac{2\sqrt{\Delta}}{\sqrt{\gamma (K+1)}}$, we obtain
\begin{equation}
    \frac{1}{K+1}\sum_{t\in \mathcal{T}_1}\norm{\nabla f (x^t)} \leq  \sqrt{\frac{4\Delta}{\gamma(K+1)}}.
\end{equation}
Combining inequalities \eqref{eq:non_convex_pre_final_1} and \eqref{eq:non_convex_pre_final_2} we get:
\begin{equation}
    \frac{1}{K+1}\sum_{t=0}^{K} \norm{\nabla f(x^t)} \leq \sqrt{\frac{4\Delta}{\gamma (K+1)}} + \frac{8\Delta}{\lambda\gamma(K+1)}.
\end{equation}
Upon considering the best iterate, we have the following bound
\begin{equation}
    \min_{t\in [0,K]}\norm{\nabla f(x^t)}^2 \leq \frac{8\Delta}{\gamma(K+1)}+ \frac{128\Delta^2}{\lambda^2\gamma^2 (K+1)^2}.
\end{equation}
\end{proof}

Theorem \ref{main_thm_nonconvex} states $7$ values for the step-size, from which the smallest should be selected. To simplify matters, we demonstrate that if $\lambda$ is selected equal or smaller than the order of $\mathcal{O}\left(\left(\frac{K}{\ln K}\right)^{\nicefrac{1}{\alpha}}\right)$, then three step-sizes are redundant and can be omitted.

\begin{corollary}
    Let all conditions of Theorem \ref{main_thm_nonconvex} hold. Furthermore, assume that $K$ is large and one selects $\lambda \leq \mathcal{O}\left(\left(\frac{K}{\ln K}\right)^{\nicefrac{1}{\alpha}}\right)$, then conclusions of Theorem \ref{main_thm_nonconvex} are valid as long as $\gamma$ is selected to satisfy $\gamma\leq \min\left\{\nicefrac{1}{4L},\gamma_1,\gamma_2,\gamma_3\right\}$ where we have
    \begin{align*}
         \gamma_1 &\eqdef\frac{\sqrt{\Delta}}{21\sqrt{L}(2^{2\alpha-1}+1)^{1/2}\sigma^{\alpha/2}\lambda^{1-\alpha/2}\sqrt{6(K+1)\ln\frac{8(K+1)}{\beta}{\left(1+ {\color{black}{\frac{\zeta_\lambda^\alpha}{\sigma^\alpha}}}\right)}}}, \\
        \gamma_2 &\eqdef \frac{\sqrt{\Delta}\lambda^{\alpha-1}}{14\sqrt{L}(K+1)2^{2\alpha-1}\left(\sigma^\alpha+{\zeta_\lambda^\alpha} \right)\left(\frac{\zeta_\lambda}{\lambda}+\frac{1}{2}+\frac{\lambda^{\alpha-1}\zeta_\lambda}{2^{2\alpha-1}\left(\sigma^\alpha+\zeta_\lambda^\alpha\right)}+\left(1+\frac{\zeta_\lambda^\alpha}{\sigma^\alpha}\right)^{-1/\alpha}\right)}, \\
        \gamma_3 &\eqdef  \frac{\sqrt{\Delta}}{14\sqrt{L}\sigma_\omega\sqrt{d(K+1)}(\sqrt{2}+\sqrt{2}\phi)}. 
    \end{align*}
\end{corollary}
\vspace{0.4cm}
\begin{proof} For large $K$, it is evident that $\gamma_3$ decreases at a rate of $\mathcal{O}\left(\sigma_\omega \sqrt{K\ln K}\right)$, while $\gamma_6$ in \eqref{non_convex_sixth_step_size_condition} decreases at a rate of $\mathcal{O}\left(\sigma_\omega \sqrt{K}\right)$. Subsequently, $\gamma_3$ dominates $\gamma_6$ and $\gamma_6$ can be omitted. Furthermore, $\gamma_5$ in \eqref{non_convex_fifth_step_size_condition} decreases with a rate of $\mathcal{O}\left(K^{\nicefrac{1}{\alpha}} (\ln K)^{1-\nicefrac{1}{\alpha}}\right)$ which is less than the rate of $\gamma_2$. It can be deduced that for large $\lambda$, $\gamma_2$ decreases at the rate $\mathcal{O}\left(K\right)$ which is faster than $\gamma_5$. If $\lambda$ is small, $\gamma_2$ dominates $\gamma_5$ again due to the $\lambda$ in the numerator of $\gamma_2$. Hence, $\gamma_5$ can be discarded. As for $\gamma_4$ in \eqref{non_convex_fourth_step_size_condition}, we know that $\sigma_\omega$ is on the order of $\mathcal{O}\left(\nicefrac{\lambda}{\epsilon}\sqrt{K\ln \left(\nicefrac{K}{\delta}\right) }\right)$. Hence, one can replace $\lambda$ with $\mathcal{O}\left(\nicefrac{\sigma_\omega \epsilon}{\sqrt{K \ln\left(\nicefrac{K}{\delta}\right)}}\right)$. Therefore, $\gamma_4$ decreases by the order $\mathcal{O}\left(\sigma_\omega\epsilon \sqrt{K \ln\left(\nicefrac{K}{\delta}\right)}\right)$, which is the same order as $\gamma_3$. Hence, $\gamma_4$ can be omitted, and the proof is complete.
 \end{proof}
 \newpage

\section{Rate and Neighborhood for Clipped-SGD: Non-Convex Case}
Now that we have established the convergence properties of \algname{DP-Clipped-SGD} for non-convex problems, we turn to evaluating its convergence rate. This rate depends critically on the choice of the step-size $\gamma$, and in general, the resulting expressions can be quite complex. To obtain more interpretable bounds, we consider simplified rate expressions by analyzing separate cases based on different ranges of $\lambda$. Since we focus on the asymptotic behavior, numerical constants are omitted for clarity.

In this section, we consider the cases without the DP noise ($\sigma_\omega=0$) and investigate all possible clipping levels. 

\paragraph{Case 1: $\lambda > 4\sqrt{L\Delta}$.} In this case, $\zeta_\lambda=0$, and the step-size conditions reduce to the following: 
     \begin{eqnarray}
      \gamma \leq \mathcal{O}\left(\min \left\{\frac{1}{L},\frac{\sqrt{\frac{\Delta}{L}}}{\sigma^{\alpha/2}\lambda^{1-\alpha/2}\sqrt{K\ln\frac{K}{\beta}}},\frac{\sqrt{\frac{\Delta}{L}}\lambda^{\alpha-1}}{K\sigma^\alpha}\right\}\right) .\label{non-convex_step_size_condtion_large_lambda}
     \end{eqnarray}
     In particular, when $\gamma$ equals the minimum from the above condition, the iterates produced by \algname{Clipped-SGD} after $K$ iterations with probability at least $1-\beta$ satisfy 
\begin{eqnarray}
\min_{t\in [0,K]}\norm{\nabla f(x^t)}^2= \mathcal{O} \left(\max \left\{ \eqref{non-convex_first_term_large_lambda}, \eqref{non-convex_second_term_large_lambda},\eqref{non-convex_third_term_large_lambda} \right\}\right), \label{large_lambda_objective_decrease}
\end{eqnarray}
where
\begin{eqnarray}
    &&\sqrt{L\Delta}\lambda^{1-\alpha/2}\sigma^{\alpha/2}\sqrt{\frac{\ln \nicefrac{K}{\beta}}{K}}+ \frac{L\Delta\sigma^\alpha \ln \nicefrac{K}{\beta}}{\lambda^\alpha  K},\label{non-convex_first_term_large_lambda} \\
    &&\frac{\sqrt{L\Delta}\sigma^\alpha}{\lambda^{\alpha-1}}+\frac{L\Delta\sigma^{2\alpha}}{\lambda^{2\alpha}} ,\label{non-convex_second_term_large_lambda} \\
    && \frac{L\Delta}{K} +\frac{L^2\Delta^2}{\lambda^2K^2} .\label{non-convex_third_term_large_lambda}
\end{eqnarray}
We clearly see that the dominant term \eqref{non-convex_first_term_large_lambda} is an increasing function of $\lambda$, and the dominant term in \eqref{non-convex_second_term_large_lambda} is a decreasing function. 
Solving for the optimal $\lambda$ where the leading terms in \eqref{non-convex_first_term_large_lambda} and \eqref{non-convex_second_term_large_lambda} become equal, we obtain $\lambda = \mathcal{O}\left(\sigma \left(\frac{K}{\ln \frac{K}{\beta}}\right)^\frac{1}{\alpha}\right)$.
Substituting back this $\lambda$, we get that with probability at least $1 - \beta$
\begin{eqnarray}
\min_{t\in [0,K]} \norm{\nabla f(x^t)}^2= \mathcal{O} \left(\max \left \{\eqref{non-convex_first_term_optimal_rate_neighborhood_large_lambda}, \eqref{non-convex_second_term_optimal_rate_neighborhood_large_lambda}\right \} \right),
\end{eqnarray}
where
\begin{eqnarray}
   && \sqrt{L\Delta}\sigma \left(\frac{\ln \frac{K}{\beta}}{K}\right)^{\frac{\alpha-1}{\alpha}}+ \frac{L\Delta\ln^2 \nicefrac{K}{\beta}}{  K^2} ,\label{non-convex_first_term_optimal_rate_neighborhood_large_lambda} \\
   &&\frac{L\Delta}{K} + \frac{L^2\Delta^2 \left (\ln{\frac{K}{\beta}}\right)^{\frac{2}{\alpha}}}{\sigma^2 K^{\frac{2\alpha+2}{\alpha}}} .\label{non-convex_second_term_optimal_rate_neighborhood_large_lambda}
\end{eqnarray}
Note in this case, we converge to the exact optimum, and the dominant term matches \citep{sadiev2023high}. As it can be seen from $\eqref{non-convex_first_term_large_lambda}, \eqref{non-convex_second_term_large_lambda}$, when the clipping level is not that large, we converge to a neighborhood of the solution, but with a faster rate.

When $\lambda \leq 4\sqrt{L\Delta}$, we have $\zeta_\lambda=\frac{4\sqrt{L\Delta}-\lambda}{2}$. As observed from $\eqref{non_convex_first_step_size_condition}, \eqref{non_convex_second_step_size_condition}$, we also have to consider the relation between $\lambda$ and $\sigma$ in these cases. Thus, we split the $\lambda \leq 4\sqrt{L\Delta}$ case into $6$ different regimes to cover all possible cases.

\paragraph{Case 2: $\frac{4}{3}\sqrt{L\Delta} < \lambda \leq 4\sqrt{L\Delta} \;\;\; \zeta_\lambda < \lambda < \sigma$.} In this case, the step-size conditions reduce to the following:
  \begin{eqnarray}
      \gamma \leq \mathcal{O}\left(\min \left\{\frac{1}{L},\frac{\sqrt{\frac{\Delta}{L}}}{\sigma^{\alpha/2}\lambda^{1-\alpha/2}\sqrt{K\ln\frac{K}{\beta}}},\frac{\sqrt{\frac{\Delta}{L}}\lambda^{\alpha-1}}{K\sigma^\alpha}\right\}\right) .\label{non-convex_step_size_condtion_small_lambda_case_2}
     \end{eqnarray}
As it can be seen, the bounds on step-size are similar to Case 1. However, the optimal $\lambda$ derived in the previous section violates the constraint that $\lambda \leq 4\sqrt{L\Delta}$. Subsequently, the optimal $\lambda$ becomes $\lambda=4\sqrt{L\Delta}$. For this choice of $\lambda$, we have that with probability at least $1 - \beta$
\begin{eqnarray}
\min_{t\in [0,K]} \norm{\nabla f(x^t)}^2= \mathcal{O} \left(\max \left\{ \eqref{non-convex_first_term_small_lambda_case_2}, \eqref{non-convex_second_term_small_lambda_case_2},\eqref{non-convex_third_term_small_lambda_case_2} \right\}\right), \label{large_lambda_objective_decrease}
\end{eqnarray}
where
\begin{eqnarray}
    &&\sqrt{\left({L\Delta}\right)^{\frac{4-\alpha}{2}}\sigma^\alpha \frac{\ln{\nicefrac{K}{\beta}}}{K}} + \frac{\left({L\Delta}\right)^{\frac{2-\alpha}{2}}\sigma^\alpha \ln{\nicefrac{K}{\beta}}}{K},\label{non-convex_first_term_small_lambda_case_2} \\
    &&\frac{\sigma^{\alpha}}{(\sqrt{L\Delta})^{\alpha-2}} + \frac{\sigma^{2\alpha}}{(L\Delta)^{\alpha-1}},\label{non-convex_second_term_small_lambda_case_2} \\
    &&\frac{L\Delta}{K} +\frac{L\Delta}{K^2} .\label{non-convex_third_term_small_lambda_case_2}
\end{eqnarray}

\paragraph{Case 3: $\frac{4}{3}\sqrt{L\Delta} < \lambda \leq 4\sqrt{L\Delta}, \quad \zeta_\lambda < \sigma <\lambda$.} In this case, the step-size conditions reduce to
\begin{eqnarray}
\gamma \leq \mathcal{O}\left(\min \left\{\frac{1}{L},\frac{\sqrt{\frac{\Delta}{L}}}{\sigma^{\alpha/2}\lambda^{1-\alpha/2}\sqrt{K\ln\frac{K}{\beta}{}}} ,\frac{\sqrt{\frac{\Delta}{L}}\lambda^{\alpha-1}}{K\max \{\sigma^{\alpha},\lambda^{\alpha-1}\zeta_\lambda\}}\right\}\right).
\end{eqnarray}
If $\max \{\sigma^{\alpha},\lambda^{\alpha-1}\zeta_\lambda\} = \sigma^{\alpha}$, then the resulting bounds are similar to the previous case. 
If $\max \{\sigma^{\alpha},\lambda^{\alpha-1}\zeta_\lambda\} = \lambda^{\alpha-1}\zeta_\lambda$ is satisfied, $\min_{t\in [0,K]} \norm{\nabla f(x^t)}^2$ is bounded with probability at least $1 - \beta$ by the maximum of the following terms:
\begin{eqnarray}
   && \sqrt{L\Delta}\lambda^{1-\alpha/2}\sigma^{\alpha/2}\sqrt{\frac{\ln{\nicefrac{K}{\beta}}}{K}} + \frac{L\Delta\sigma^\alpha \ln{\nicefrac{K}{\beta}}}{\lambda^\alpha K}, \\
    &&\sqrt{L\Delta}\zeta_\lambda + \frac{L\Delta\zeta_\lambda^2}{\lambda^2}, \label{non-convex_second_term_small_lambda_case_3} \\
    && \frac{L\Delta}{K}+ \frac{L^2\Delta^2}{\lambda^2 K^2}.
\end{eqnarray}
In the latter case (i.e., maximum occurring in the second argument), the optimal $\lambda$ is $4\sqrt{L\Delta}-\eta $, where $\eta$ is a sufficiently small number such that $\lambda^{\alpha-1}\zeta_\lambda \geq \sigma^\alpha$, i.e., $\lambda$ satisfies $\zeta_\lambda = \max\left\{\frac{\sigma^\alpha}{\lambda^{\alpha-1}}, \lambda^{1-\alpha/2}\sigma^{\alpha/2}\sqrt{\frac{\ln{\nicefrac{K}{\beta}}}{K}}\right\}$. 
Note that the \eqref{non-convex_second_term_small_lambda_case_3} is decreasing in $\lambda$, and $\lambda=4\sqrt{L\Delta}$ is not feasible. With this choice of $\lambda$, we get: 
\begin{eqnarray}
\min_{t\in [0,K]} \norm{\nabla f(x^t)}^2= \mathcal{O} \left(\max \left\{ \eqref{non-convex_first_term_small_lambda_case_3}, \eqref{non-convex_second_term_small_lambda_case_3_1},\eqref{non-convex_third_term_small_lambda_case_3} \right\}\right), \label{large_lambda_objective_decrease}
\end{eqnarray}
where
\begin{eqnarray}
    &&\sqrt{L\Delta(4\sqrt{L\Delta}-\eta)^{2-\alpha}\sigma^\alpha \frac{\ln{\nicefrac{K}{\beta}}}{K}} + \frac{L\Delta\sigma^\alpha \ln{\nicefrac{K}{\beta}}}{(\sqrt{L\Delta}-\eta)^{\alpha}K},\label{non-convex_first_term_small_lambda_case_3} \\
    &&\frac{\sqrt{L\Delta}\eta}{2} + \frac{L\Delta\eta^2}{(4\sqrt{L\Delta}-\eta)^2}, \label{non-convex_second_term_small_lambda_case_3_1} \\
    && \frac{L\Delta}{K}+ \frac{L^2\Delta^2}{(4\sqrt{L\Delta}-\eta)^2K^2}.\label{non-convex_third_term_small_lambda_case_3}
\end{eqnarray}

\paragraph{Case 4: $\frac{4}{3}\sqrt{L\Delta}< \lambda \leq 4\sqrt{L\Delta}, \quad \sigma < \zeta_\lambda <\lambda$.} For this case, step-size conditions reduce to
\begin{eqnarray}
    \gamma \leq \cO \left(\min \left\{\frac{1}{L},\frac{\sqrt{\frac{\Delta}{L}}}{\zeta_\lambda^{\alpha/2}\lambda^{1-\alpha/2}\sqrt{K\ln\frac{K}{\beta}{}}},\frac{\sqrt{\frac{\Delta}{L}}\lambda^{\alpha-1}}{K(\lambda^{\alpha-1}\zeta_\lambda)}\right\}\right) ,
\end{eqnarray} 
and $\min_{t\in [0,K]} \norm{\nabla f(x^t)}^2$ is bounded with probability at least $1 - \beta$ by the maximum of the following terms
\begin{eqnarray}
    &&\sqrt{L\Delta}\lambda^{1-\alpha/2}\zeta_{\lambda}^{\alpha/2}\sqrt{\frac{\ln{\nicefrac{K}{\beta}}}{K}} + \frac{L\Delta\zeta_\lambda^\alpha \ln{\nicefrac{K}{\beta}}}{\lambda^\alpha K}, \\
    &&\sqrt{L\Delta}\zeta_\lambda + \frac{L\Delta\zeta_\lambda^2}{\lambda^2}, \\
    &&\frac{L\Delta}{K} + \frac{L^2\Delta^2}{\lambda^2K^2}.
\end{eqnarray}
 The optimal $\lambda$ in this case is $\lambda=4\sqrt{L\Delta}-2\sigma$, and we have that with probability at least $1 - \beta$
\begin{eqnarray}
\min_{t\in [0,K]} \norm{\nabla f(x^t)}^2= \mathcal{O} \left(\max \left\{ \eqref{non-convex_first_term_small_lambda_case_4}, \eqref{non-convex_second_term_small_lambda_case_4},\eqref{non-convex_third_term_small_lambda_case_4} \right\}\right), \label{large_lambda_objective_decrease}
\end{eqnarray}
where
\begin{eqnarray}
    &&\sqrt{L\Delta(4\sqrt{L\Delta}-2\sigma)^{2-\alpha}\sigma^\alpha \frac{\ln{\nicefrac{K}{\beta}}}{K}} + \frac{L\Delta\sigma^\alpha \ln{\nicefrac{K}{\beta}}}{(4\sqrt{L\Delta}-2\sigma)^{\alpha}K},\label{non-convex_first_term_small_lambda_case_4} \\
    &&{\sqrt{L\Delta}\sigma} + \frac{L\Delta\sigma^2}{(4\sqrt{L\Delta}-2\sigma)^2}, \label{non-convex_second_term_small_lambda_case_4} \\
    && \frac{L\Delta}{K}+ \frac{L^2\Delta^2}{(4\sqrt{L\Delta}-2\sigma)^2K^2}\label{non-convex_third_term_small_lambda_case_4}.
\end{eqnarray}

\paragraph{Case 5: $\lambda \leq \frac{4}{3}\sqrt{L\Delta}, \quad \lambda < \zeta_\lambda <\sigma$.} In this case, the step-size conditions reduce to
\begin{eqnarray}
   \gamma \leq \cO\left(\min \left\{ \frac{1}{L},\frac{\sqrt{\frac{\Delta}{L}}}{\sigma^{\alpha/2}\lambda^{1-\alpha/2}\sqrt{K\ln\frac{K}{\beta}{}}}, \frac{\sqrt{\frac{\Delta}{L}}\lambda^{\alpha}}{K(\sigma^{\alpha}\zeta_\lambda)}\right\}\right),
\end{eqnarray}
and $\min_{t\in [0,K]} \norm{\nabla f(x^t)}^2$ is bounded with probability at least $1 - \beta$ by the maximum of the following terms
\begin{eqnarray}
    && \sqrt{L\Delta}\lambda^{1-\alpha/2}\sigma^{\alpha/2}\sqrt{\frac{\ln{\nicefrac{K}{\beta}}}{K}} + \frac{L\Delta\sigma^\alpha \ln{\nicefrac{K}{\beta}}}{\lambda^\alpha K}, \\
    &&\sqrt{L\Delta}\frac{\sigma^\alpha \zeta_\lambda}{\lambda^\alpha} + \frac{L\Delta \sigma^{2\alpha}\zeta_\lambda^2}{\lambda^{2\alpha+2}},\\
    && \frac{L\Delta}{K} + \frac{L^2\Delta^2}{\lambda^2K^2}.
\end{eqnarray}
In this regime, the optimal $\lambda=\frac{4}{3}\sqrt{L\Delta}$. With this choice of $\lambda$, we get with probability at least $1 - \beta$
\begin{eqnarray}
\min_{t\in [0,K]} \norm{\nabla f(x^t)}^2= \mathcal{O} \left(\max \left\{ \eqref{non-convex_first_term_small_lambda_case_5}, \eqref{non-convex_second_term_small_lambda_case_5}, \eqref{non-convex_third_term_small_lambda_case_5}\right\}\right) ,\label{large_lambda_objective_decrease}
\end{eqnarray}
where
\begin{eqnarray}
  &&\sqrt{(L\Delta)^{\frac{4-\alpha}{2}}\sigma^\alpha \frac{\ln{\nicefrac{K}{\beta}}}{K}} + \frac{(L\Delta)^{\frac{2-\alpha}{2}}\sigma^\alpha \ln{\nicefrac{K}{\beta}}}{K},\label{non-convex_first_term_small_lambda_case_5} \\
    &&\frac{\sigma^{\alpha}}{(\sqrt{L\Delta})^{\alpha-2}} + \frac{\sigma^{2\alpha}}{(L\Delta)^{\alpha-1}} ,\label{non-convex_second_term_small_lambda_case_5} \\
    && \frac{L\Delta}{K} + \frac{L\Delta}{K^2} .\label{non-convex_third_term_small_lambda_case_5}
\end{eqnarray}

\paragraph{Case 6: $\lambda \leq \frac{4}{3}\sqrt{L\Delta} , \quad \lambda < \sigma <\zeta_\lambda$.} In this case, the step-size conditions reduce to
\begin{eqnarray}
   \gamma \leq \cO \left( \min \left\{\frac{1}{L},\frac{\sqrt{\frac{\Delta}{L}}}{\zeta_\lambda^{\alpha/2}\lambda^{1-\alpha/2}\sqrt{K\ln\frac{K}{\beta}{}}} 
    ,\frac{\sqrt{\frac{\Delta}{L}}\lambda^{\alpha}}{K(\zeta_\lambda^{\alpha+1})} \right\} \right),
\end{eqnarray}
and $\min_{t\in [0,K]} \norm{\nabla f(x^t)}^2$ is bounded with probability at least $1 - \beta$ by the maximum of the following terms
\begin{eqnarray}
    &&\sqrt{L\Delta}\lambda^{1-\alpha/2}\zeta_\lambda^{\alpha/2}\sqrt{\frac{\ln{\nicefrac{K}{\beta}}}{K}} + \frac{L\Delta\zeta_\lambda^\alpha \ln{\nicefrac{K}{\beta}}}{\lambda^\alpha K}, \label{eq:vjdfjdvfjbvdfjbvdjf}\\
    &&\frac{\sqrt{L\Delta}\zeta_\lambda^{\alpha+1}}{\lambda^\alpha} + \frac{L\Delta \zeta_\lambda^{2\alpha}}{\lambda^{2\alpha+2}}, \label{eq:nvdkfkdvjdfbf}\\
    &&\frac{L\Delta}{K}+\frac{L^2\Delta^2}{\lambda^2 K^2}.
\end{eqnarray}
Next, we find the optimal $\lambda$ via equalizing the leading terms (the first ones) in \eqref{eq:vjdfjdvfjbvdfjbvdjf} and \eqref{eq:nvdkfkdvjdfbf}. This yields $\lambda=\frac{4\sqrt{L\Delta}}{2C+1}$, where $C=\left(\frac{\ln{\frac{K}{\beta}}}{K}\right)^{\frac{1}{\alpha+2}}$, which is infeasible. Thus, the optimal $\lambda$ in this regime is $\lambda=\frac{4}{3}\sqrt{L\Delta}-\eta$, where $\eta \geq 0$ is such that $\lambda < \sigma < \zeta_\lambda$. Given this choice of $\lambda$, we obtain with probability at least $1 - \beta$
\begin{align}
    \min_{t\in [0,K]} \norm{\nabla f(x^t)}^2= \mathcal{O} \left(\max \left\{ \eqref{non-convex_first_term_small_lambda_case_6}, \eqref{non-convex_second_term_small_lambda_case_6},\eqref{non-convex_third_term_small_lambda_case_6} \right\}\right) ,\label{large_lambda_objective_decrease}
\end{align}
where
\begin{eqnarray}
    &&(\sqrt{L\Delta}-\eta)^{1-\alpha/2}(\sqrt{L\Delta}+\eta)^{\alpha/2}\sqrt{L\Delta \frac{\ln{\nicefrac{K}{\beta}}}{K}} + \frac{L\Delta(\sqrt{L\Delta}+\eta)^\alpha \ln{\nicefrac{K}{\beta}}}{(\sqrt{L\Delta}-\eta)^\alpha K},\label{non-convex_first_term_small_lambda_case_6} \\
    &&\frac{\sqrt{L\Delta}(\sqrt{L\Delta}+\eta)^{\alpha+1}}{(\sqrt{L\Delta}-\eta)^\alpha} +\frac{L\Delta(\sqrt{L\Delta}+\eta)^{2\alpha}}{(\sqrt{L\Delta}-\eta)^{2\alpha+2}} ,\label{non-convex_second_term_small_lambda_case_6} \\
    && \frac{L\Delta}{K}+ \frac{L^2\Delta^2}{(\sqrt{L\Delta}-\eta)^2K^2}.\label{non-convex_third_term_small_lambda_case_6}
\end{eqnarray}

\paragraph{Case 7: $\lambda \leq \frac{4}{3}\sqrt{L\Delta} , \quad \sigma < \lambda <\zeta_\lambda$.} In this case, the step-size conditions reduce to 
\begin{eqnarray}
   \gamma \leq \cO \left( \min \left\{\frac{1}{L},\frac{\sqrt{\frac{\Delta}{L}}}{\zeta_\lambda^{\alpha/2}\lambda^{1-\alpha/2}\sqrt{K\ln\frac{K}{\beta}{}}},\frac{\sqrt{\frac{\Delta}{L}}\lambda^{\alpha-1}}{K\max\left\{\frac{\zeta_\lambda^{\alpha+1}}{\lambda},\zeta_\lambda^{\alpha-1}\sigma \right\}} \right\}\right)  .
\end{eqnarray}
We note that $\max\left\{\frac{\zeta_\lambda^{\alpha+1}}{\lambda},\zeta_\lambda^{\alpha-1}\sigma \right\} = \zeta^\alpha \max\left\{ \frac{\zeta_\lambda}{\lambda}, \frac{\sigma}{\lambda} \right\} = \frac{\zeta_\lambda^{\alpha+1}}{\lambda}$ since $\sigma < \lambda < \zeta_\lambda$. Therefore, similarly to the previous case, we have
\begin{eqnarray}
   \gamma \leq \cO \left( \min \left\{\frac{1}{L},\frac{\sqrt{\frac{\Delta}{L}}}{\zeta_\lambda^{\alpha/2}\lambda^{1-\alpha/2}\sqrt{K\ln\frac{K}{\beta}{}}},\frac{\sqrt{\frac{\Delta}{L}}\lambda^{\alpha}}{K\zeta_\lambda^{\alpha+1}} \right\}\right),
\end{eqnarray}
and $\min_{t\in [0,K]} \norm{\nabla f(x^t)}^2$ is bounded with probability at least $1 - \beta$ by the maximum of the following terms
\begin{eqnarray}
    &&\sqrt{L\Delta}\lambda^{1-\alpha/2}\zeta_\lambda^{\alpha/2}\sqrt{\frac{\ln{\nicefrac{K}{\beta}}}{K}} + \frac{L\Delta\zeta_\lambda^\alpha \ln{\nicefrac{K}{\beta}}}{\lambda^\alpha K}, \label{eq:vjdfjdvfjbvdfjbvdjf_1}\\
    &&\frac{\sqrt{L\Delta}\zeta_\lambda^{\alpha+1}}{\lambda^\alpha} + \frac{L\Delta \zeta_\lambda^{2\alpha}}{\lambda^{2\alpha+2}}, \label{eq:nvdkfkdvjdfbf_1}\\
    &&\frac{L\Delta}{K}+\frac{L^2\Delta^2}{\lambda^2 K^2}.
\end{eqnarray}
The optimal $\lambda$ equals $\frac{4}{3}\sqrt{L\Delta}$. This happens because both leading terms in \eqref{eq:vjdfjdvfjbvdfjbvdjf_1} and \eqref{eq:nvdkfkdvjdfbf_1} are decreasing in $\lambda$. With this choice, we get with probability at least $1 - \beta$
\begin{align}
    \min_{t\in [0,K]} \norm{\nabla f(x^t)}^2= \mathcal{O} \left(\max \left\{ \eqref{non-convex_first_term_small_lambda_case_7}, \eqref{non-convex_second_term_small_lambda_case_7},\eqref{non-convex_third_term_small_lambda_case_7} \right\}\right) ,\label{large_lambda_objective_decrease}
\end{align}
where
\begin{eqnarray}
    &&\sqrt{L\Delta \frac{\ln{\nicefrac{K}{\beta}}}{K}} + \frac{L\Delta \ln{\nicefrac{K}{\beta}}}{K},\label{non-convex_first_term_small_lambda_case_7} \\
    &&\sqrt{L\Delta}\sigma +\frac{\sigma^2}{L\Delta} ,\label{non-convex_second_term_small_lambda_case_7} \\
    && \frac{L\Delta}{K}+ \frac{L\Delta}{K^2}.\label{non-convex_third_term_small_lambda_case_7}
\end{eqnarray}

Now that we have covered all possible regions, it's time to consider the DP noise as well. 

\clearpage

\section{Rate and Neighborhood for \algname{DP-Clipped-SGD}: Non-Convex Case}

To ensure the output of the algorithm is $(\varepsilon,\delta)$-differentially private in this setting, expectation minimization, it suffices to set the noise scale as $\sigma_\omega=\Theta \left( \frac{\lambda}{\varepsilon}\sqrt{K \ln \left( \frac{K}{\delta} \right) \ln \left( \frac{1}{\delta} \right)}\right)$ and apply the advanced composition theorem of \cite{dwork2014algorithmic}. In the finite sum case, one can reduce the amount of noise by a factor of $\sqrt{\ln \left( \frac{K}{\delta} \right)}$ as it was shown by \cite{abadi2016deep}. For the sake of brevity, in the DP case, we only consider two cases: large $\lambda$ and relatively small $\lambda$ regimes. The other cases can be derived with a similar analysis.


\paragraph{Case 1: $\lambda > 4\sqrt{L\Delta}$.} In this case, $\zeta_\lambda=0$, and the step-size conditions reduce to the following:
     \begin{eqnarray}
      \gamma \leq \mathcal{O}\left(\min \left\{\frac{1}{L},\frac{\sqrt{\frac{\Delta}{L}}}{\sigma^{\alpha/2}\lambda^{1-\alpha/2}\sqrt{K\ln\frac{K}{\beta}}},\frac{\sqrt{\frac{\Delta}{L}}\lambda^{\alpha-1}}{K\sigma^\alpha},\frac{\sqrt{\frac{\Delta}{L}}}{\sigma_\omega\sqrt{dK\ln{\frac{K}{\beta}}}}\right\}\right) \label{non-convex_step_size_condtion_large_lambda}
     \end{eqnarray}
     In particular, when $\gamma$ equals the minimum from the step-size condition, then the iterates produced by \algname{DP-Clipped-SGD} after $K$ iterations with
probability at least $1-\beta$ satisfy 
\begin{eqnarray}
\min_{k\in [0,K]} \norm{\nabla f(x^t)}^2= \mathcal{O} \left(\max \left\{ \eqref{non-convex_first_term_large_lambda_DP}, \eqref{non-convex_second_term_large_lambda_DP},\eqref{non-convex_third_term_large_lambda_DP}, \eqref{non-convex_fourth_term_large_lambda_DP} \right\}\right) \label{large_lambda_objective_decrease}
\end{eqnarray}
where
\begin{eqnarray}
    &&\sqrt{L\Delta}\lambda^{1-\alpha/2}\sigma^{\alpha/2}\sqrt{\frac{\ln \nicefrac{K}{\beta}}{K}}+ \frac{L\Delta\sigma^\alpha \ln \nicefrac{K}{\beta}}{\lambda^\alpha  K}\label{non-convex_first_term_large_lambda_DP} \\
    &&\frac{\sqrt{L\Delta}\sigma^\alpha}{\lambda^{\alpha-1}}+\frac{L\Delta\sigma^{2\alpha}}{\lambda^{2\alpha}} \label{non-convex_second_term_large_lambda_DP} \\
    && \frac{L\Delta}{K} +\frac{L^2\Delta^2}{\lambda^2K^2} \label{non-convex_third_term_large_lambda_DP} \\
    && \sqrt{L\Delta}\sigma_\omega \sqrt{\frac{d\ln{\frac{K}{\beta}}}{K}} + \frac{L\Delta\sigma_\omega^2d\ln{\frac{K}{\beta}}}{\lambda^2 K}.\label{non-convex_fourth_term_large_lambda_DP}
\end{eqnarray}
Here, \eqref{non-convex_second_term_large_lambda_DP} accounts for the bias caused by clipping, and \eqref{non-convex_fourth_term_large_lambda_DP} accounts for the accumulation of DP noise. These terms are decreasing and increasing in $\lambda$ respectively, if we use $\sigma_\omega=\Theta \left( \frac{\lambda}{\varepsilon}\sqrt{K \ln \left( \frac{K}{\delta} \right) \ln \left( \frac{1}{\delta} \right)}\right)$. To find the optimal $\lambda$, we find the equilibrium of these two terms. Solving the equilibrium equation, we get $\lambda = \cO \left (\frac{\varepsilon \sigma^\alpha}{d\ln{\left(\frac{1}{\delta}\right) \ln \left(\frac{K}{\delta}\right)\ln \left({\frac{K}{\beta}}\right)}}\right)^{\frac{1}{\alpha}}$. 
Unless $\varepsilon\sigma^\alpha$ is large enough, this value violates the constraint that $\lambda > 4\sqrt{L\Delta}$, and it is not feasible. Thus, we have the following formula for the optimal $\lambda$:
\begin{align}
  \lambda=\max \left \{ 4\sqrt{L\Delta}, \left(\frac{\varepsilon \sigma^\alpha}{d\ln{\left(\frac{1}{\delta}\right) \ln \left(\frac{K}{\delta}\right)\ln \left({\frac{K}{\beta}}\right)}} \right)^{\frac{1}{\alpha}}\right \} . 
\end{align}
For this choice of $\lambda$, we get that with probability at least $1 - \beta$
\begin{eqnarray}
\min_{k\in [0,K]} \norm{\nabla f(x^t)}^2= \mathcal{O} \left(\max \left \{\eqref{non-convex_first_term_large_lambda_DP_optimal}, \eqref{non-convex_second_term_large_lambda_DP_optimal}, \eqref{non-convex_third_term_large_lambda_DP_optimal},\eqref{non-convex_fourth_term_large_lambda_DP_optimal}\right \}\right)
\end{eqnarray}
with
\begin{eqnarray}
   && \max \left \{\sqrt{(L\Delta)^{\frac{4-\alpha}{2}}\sigma^\alpha \frac{\ln{\nicefrac{K}{\beta}}}{K}}  , \sqrt{L\Delta} \left (\frac{\varepsilon \sigma^{\alpha}}{\sqrt{d\ln{\left(\frac{1}{\delta}\right) \ln \left(\frac{K}{\delta}\right)}}}\right)^{\frac{1}{\alpha}}\sqrt{\frac{\ln^{\frac{3\alpha-2}{2\alpha}}{\frac{K}{\beta}}}{K}}\right \}\label{non-convex_first_term_large_lambda_DP_optimal} \\
   &&\min \left \{ \frac{\sigma^{\alpha}}{\left(\sqrt{L\Delta}\right)^{\alpha-2}},\sqrt{L\Delta}\sigma \left( \frac{\sqrt{d\ln{\left(\frac{1}{\delta}\right) \ln \left(\frac{K}{\delta}\right)\ln \left({\frac{K}{\beta}}\right)}}}{\varepsilon}\right)^{\frac{\alpha-1}{\alpha}} \right \}
   \label{non-convex_second_term_large_lambda_DP_optimal} \\
   &&\min \left \{\frac{L\Delta}{K^2} , \frac{L^2 \Delta^2 \left(d\ln{\left(\frac{1}{\delta}\right) \ln \left(\frac{K}{\delta}\right)}\right)^{\frac{1}{\alpha}}}{(\varepsilon)^{\frac{1}{\alpha}}\sigma} \frac{\ln^{\frac{1}{\alpha}}{\frac{K}{\beta}}}{K^2}\right \}+ \frac{L\Delta}{K} \label{non-convex_third_term_large_lambda_DP_optimal} \\
   &&\max \left \{\frac{L\Delta}{\varepsilon}\sqrt{d\ln{\left(\frac{1}{\delta}\right) \ln \left(\frac{K}{\delta}\right)\ln \left({\frac{K}{\beta}}\right)}}, \frac{R\sigma \left (d\ln{\left(\frac{1}{\delta}\right) \ln \left(\frac{K}{\delta}\right)\ln \left({\frac{K}{\beta}}\right)}\right)^{\frac{\alpha+2}{2\alpha}}}{\varepsilon^{\frac{\alpha-1}{\alpha}}} \right \} \notag \\&& \;\;\;\;\;\;\;\;\;\;+ \frac{L\Delta}{\varepsilon^2}d\ln{\left(\frac{1}{\delta}\right) \ln \left(\frac{K}{\delta}\right)\ln \left({\frac{K}{\beta}}\right)}, \label{non-convex_fourth_term_large_lambda_DP_optimal}
\end{eqnarray}
where, for the sake of brevity, we only report the dominant terms.

\paragraph{Case 2: $ \lambda \leq \frac{4}{3}\sqrt{L\Delta}
\quad \lambda < \sigma <\zeta_\lambda$.} In this case, the step-size conditions reduce to the following:
\begin{eqnarray}
   \gamma \leq \cO \left( \min \left\{\frac{1}{L},\frac{\sqrt{\frac{\Delta}{L}}}{\zeta_\lambda^{\alpha/2}\lambda^{1-\alpha/2}\sqrt{K\ln\frac{K}{\beta}{}}} 
    ,\frac{\sqrt{\frac{\Delta}{L}}\lambda^{\alpha}}{K(\zeta_\lambda^{\alpha+1})},\frac{\sqrt{\frac{\Delta}{L}}}{\sigma_\omega\sqrt{dK\ln{\frac{K}{\beta}}}} \right\} \right).
\end{eqnarray}
Taking $\gamma$ equal to the right-hand side, we get that with probability at least $1 - \beta$
\begin{align}
    \min_{t\in [0,K]} \norm{\nabla f(x^t)}^2 = \cO \left ( \left \{\eqref{non-convex_first_term_small_lambda_DP} , \eqref{non-convex_second_term_small_lambda_DP}, \eqref{non-convex_third_term_small_lambda_DP}, \eqref{non-convex_fourth_term_small_lambda_DP}\right \}\right)
\end{align}
with
\begin{eqnarray}
    &&\sqrt{L\Delta}\lambda^{1-\alpha/2}\sigma^{\alpha/2}\sqrt{\frac{\ln{\nicefrac{K}{\beta}}}{K}} + \frac{L\Delta\sigma^\alpha \ln{\nicefrac{K}{\beta}}}{\lambda^\alpha K} \label{non-convex_first_term_small_lambda_DP}\\
    &&\frac{\sqrt{L\Delta}\zeta_\lambda^{\alpha+1}}{\lambda^\alpha} + \frac{L\Delta\zeta_\lambda^{2\alpha}}{\lambda^{2\alpha+2}}  \label{non-convex_second_term_small_lambda_DP}\\
    &&\frac{L\Delta}{K}+\frac{L^2\Delta^2}{\lambda^2 K^2} \label{non-convex_third_term_small_lambda_DP} \\
    && \sqrt{L\Delta}\sigma_\omega \sqrt{\frac{d \ln{\frac{K}{\beta}}}{K}} + \frac{L\Delta\sigma_\omega^2d \ln{\frac{K}{\beta}}}{\lambda^2 K}. \label{non-convex_fourth_term_small_lambda_DP}
\end{eqnarray}
Similarly to the previous case, we find the optimal $\lambda$ as the equilibrium of the leading terms in \eqref{non-convex_second_term_small_lambda_DP} and \eqref{non-convex_fourth_term_small_lambda_DP}. By doing so, we get the following optimal $\lambda$:
\begin{align}
    \lambda= \min \left \{\frac{4}{3}\sqrt{L\Delta}, \frac{2\varepsilon \sqrt{L\Delta}}{\left(d\ln{\left(\frac{1}{\delta}\right) \ln \left(\frac{K}{\delta}\right)\ln \left({\frac{K}{\beta}}\right)}\right)^{\frac{1}{2\alpha+2}} +1}\right \}
\end{align}
For this choice of $\lambda$, we get that with probability at least $1 - \beta$
\begin{align}
    \min_{k\in [0,K]} \norm{\nabla f(x^t)}^2= \mathcal{O} \left(\max \left\{ \eqref{non-convex_first_term_small_lambda_DP_optimal}, \eqref{non-convex_second_term_small_lambda_DP_optimal},\eqref{non-convex_third_term_small_lambda_DP_optimal}, \eqref{non-convex_fourth_term_small_lambda_DP_optimal}\right\}\right) 
\end{align}
with
\begin{eqnarray}
    &&\min \left \{\sqrt{(L\Delta)^{\frac{4-\alpha}{2}}\sigma^\alpha \frac{\ln{\nicefrac{K}{\beta}}}{K}}, \sqrt{\frac{(L\Delta)^{\frac{4-\alpha}{2}}\varepsilon ^{2-\alpha}\ln^{\frac{3\alpha}{4\alpha+4}}{\frac{K}{\beta}}}{\left(d\ln{\left(\frac{1}{\delta}\right) \ln \left(\frac{K}{\delta}\right)}\right)^{\frac{2-\alpha}{4\alpha+4}}K}} \right \}\label{non-convex_first_term_small_lambda_DP_optimal} \\
    &&\max \left \{\frac{\sigma^{\alpha}}{\sqrt{L\Delta}^{\alpha-2}} , \frac{(\sqrt{L\Delta})^{2-\alpha}\sigma^\alpha}{\varepsilon}\left(d\ln{\left(\frac{1}{\delta}\right) \ln \left(\frac{K}{\delta}\right)\ln \left({\frac{K}{\beta}}\right)}\right)^{\frac{\alpha-1}{2\alpha+2}}\right\}\label{non-convex_second_term_small_lambda_DP_optimal} \\
    &&\max \left \{ \frac{L\Delta}{K^2}, \frac{L\Delta}{\varepsilon^2 K^2}\left (\left(d\ln{\left(\frac{1}{\delta}\right) \ln \left(\frac{K}{\delta}\right)\ln \left({\frac{K}{\beta}}\right)}\right)^{\frac{1}{2\alpha+2}} +1\right)^2 \right \}+\frac{L\Delta}{K}
\label{non-convex_third_term_small_lambda_DP_optimal}\\
   &&\min \left \{\frac{L\Delta}{\varepsilon}\sqrt{d\ln{\left(\frac{1}{\delta}\right) \ln \left(\frac{K}{\delta}\right)\ln \left({\frac{K}{\beta}}\right)}}, \frac{ L\Delta\sqrt{\ln \frac{K}{\beta}}}{\left(d\ln{\left(\frac{1}{\delta}\right) \ln \left(\frac{K}{\delta}\right)\ln \left({\frac{K}{\beta}}\right)}\right)^{\frac{1}{2\alpha+2}} +1} \right \}\notag \\ && \;\;\;\;\;\;\;\;\;+ \frac{L\Delta d}{\varepsilon^2}d\ln{\left(\frac{1}{\delta}\right) \ln \left(\frac{K}{\delta}\right)\ln \left({\frac{K}{\beta}}\right)},\label{non-convex_fourth_term_small_lambda_DP_optimal}
\end{eqnarray}
where, for the sake of brevity, we only report the dominant terms.

\end{document}